\documentclass[twoside]{article}

\usepackage[accepted]{aistats2021}

\usepackage{amsmath}
\usepackage{amsthm}
\usepackage{amsfonts}
\usepackage{mathtools}
\usepackage{amssymb}
\usepackage{caption}

\usepackage[shortlabels]{enumitem}

\usepackage{url}

\newcommand{\abs}[1]{\left\lvert#1\right\rvert}% absolute value
\newcommand{\argmin}[1]{\underset{#1}{ {\text{arg\,min}}}}%argmin
\newcommand{\betahat}{\hat{\beta}}% empirical explanations
\newcommand{\betainf}{\beta_{\infty}}% large bandwidth beta
\newcommand{\bigo}[1]{\mathcal{O}\left(#1\right)}% big O notation
\newcommand{\card}[1]{\# #1}% card of a set without braces
\newcommand{\cardset}[1]{\# \{#1\}}% cardinality of a set
\newcommand{\condexpec}[2]{\mathbb{E}\left[#1\middle|#2\right]}% conditional expectation
\newcommand{\condexpecunder}[3]{\mathbb{E}_{#3}\left[#1\middle|#2\right]}% conditional expectation
\newcommand{\condproba}[2]{\mathbb{P}\left(#1\middle|#2\right)}% conditional probability
\newcommand{\corp}{\mathcal{C}}% corpus
\newcommand{\dencst}{c}% constant in the denominator of \Sigma^{-1}
\newcommand{\defeq}{\vcentcolon =}% define equals
\newcommand{\distcos}[2]{d_{\cos} (#1,#2)}% cosinus distance
\newcommand{\Exp}{\mathrm{exp}}% normal exponential
\newcommand{\Exps}{\mathrm{e}}% small exponential
\renewcommand{\exp}[1]{\Exp\left(#1\right)}% exponential of something
\newcommand{\exps}[1]{\Exps^{#1}}% e to the power something
\newcommand{\Expec}{\mathbb{E}}% symbol for expectation
\newcommand{\expec}[1]{\Expec\left[#1\right]}% expected value
\newcommand{\expecunder}[2]{\Expec_{#2}\left[#1\right]}% expectation under a certain law
\newcommand{\frobnorm}[1]{\norm{#1}_{\mathrm{F}}}% Frobenius norm
\newcommand{\Gammahat}{\hat{\Gamma}}% empirical Gamma
\newcommand{\Gammainf}{\Gamma_{\infty}}% large bandwidth Gamma
\newcommand{\Identity}{I}% identity matrix
\newcommand{\Indic}{\mathbf{1}}% indicator
\newcommand{\indic}[1]{\Indic_{#1}}% indicator of something

\newcommand{\dg}{{\mathcal{D}}}% global dictionary
\newcommand{\dl}{D_{\ell}}% local dictionary

\newcommand{\norm}[1]{\left\lVert#1\right\rVert}% norm of something

\newcommand{\Normtfidf}{\phi}% normalized tfidf
\newcommand{\opnorm}[1]{\norm{#1}_{\mathrm{op}}}% operator norm of a matrix
\newcommand{\smallopnorm}[1]{\smallnorm{#1}_{\mathrm{op}}}% operator norm of a matrix
\newcommand{\Tfidf}{\varphi}% tfidf
\newcommand{\normtfidf}[1]{\Normtfidf(#1)}% normalized tfidf of a word
\newcommand{\tfidf}[1]{\Tfidf(#1)}% tfidf

\newcommand{\Proba}{\mathbb{P}}% symbol for probability
\newcommand{\proba}[1]{\Proba\left (#1\right )}% probability of an event
\newcommand{\probaunder}[2]{\Proba_{#2}\left(#1\right)}% probability under
\newcommand{\Reals}{\mathbb{R}}% real numbers
\newcommand{\Sigmahat}{\hat{\Sigma}}% empirical covariance matrix
\newcommand{\smallexpec}[1]{\Expec[#1]}% expected value, small version
\newcommand{\smallnorm}[1]{\lVert#1\rVert}% norm of a vector
\newcommand{\sphere}[1]{S^{#1}}% sphere of dimension something
\newcommand{\Var}{\mathrm{Var}}% symbole for the variance
\newcommand{\var}[1]{\Var\left(#1\right)}% variance of a random variable

\newcommand{\word}{w}% any word
\newcommand{\cvas}{\overset{\text{a.s.}}{\longrightarrow}}% almost sure convergence
\theoremstyle{plain}
\newtheorem{theorem}{Theorem}%[section]
\newtheorem{proposition}{Proposition}%[section]
\newtheorem{lemma}{Lemma}%[section]
\newtheorem{corollary}{Corollary}%[section]
\theoremstyle{definition}
\newtheorem{definition}{Definition}%[section]
\newtheorem{remark}{Remark}%[section]
\makeatletter
\def\th@plain{%
  \thm@notefont{}% same as heading font
  \itshape % body font
}
\def\th@definition{%
  \thm@notefont{}% same as heading font
  \normalfont % body font
}
\makeatother

\usepackage[round]{natbib}

\begin{document}
\runningauthor{Dina Mardaoui and Damien Garreau}

\twocolumn[

\aistatstitle{An Analysis of LIME for Text Data}

\aistatsauthor{Dina Mardaoui \And Damien Garreau} %\texttt{dina.mardaoui@etu.univ-cotedazur.fr} \And Damien Garreau$^2$ \\ \texttt{damien.garreau@unice.fr}}
\aistatsaddress{Polytech Nice, France \And Universit\'e C\^ote d'Azur, Inria, CNRS, LJAD, France}
%\vspace{0.2cm}
%\aistatsaddress{$^1$Polytech Nice \\ $^2$Universit\'e C\^ote d'Azur, Inria, CNRS, LJAD, France} 
]

\begin{abstract}
Text data are increasingly handled in an automated fashion by machine learning algorithms. 
But the models handling these data are not always well-understood due to their complexity and are more and more often referred to as ``black-boxes.''
Interpretability methods aim to explain how these models operate. 
Among them, LIME has become one of the most popular in recent years. 
However, it comes without theoretical guarantees: even for simple models, we are not sure that LIME behaves accurately. 
In this paper, we provide a first theoretical analysis of LIME for text data. 
As a consequence of our theoretical findings, we show that LIME indeed provides meaningful explanations for simple models, namely decision trees and linear models. 
\end{abstract}

\section{Introduction}

Natural language processing has progressed at an accelerated pace in the last decade. 
This time period saw the second coming of artificial neural networks, embodied by the apparition of recurrent neural networks (RNNs) and more particularly long short-term memory networks (LSTMs). 
These new architectures, in conjunction with large, publicly available datasets and efficient optimization techniques, have allowed computers to compete with and sometime even beat humans on specific tasks. 

More recently, the paradigm has shifted from recurrent neural networks to \emph{transformers networks} \citep{vaswani_et_al_2017}.  
Instead of training models specifically for a task, large \emph{language models} are trained on supersized datasets. 
For instance, \texttt{Webtext2} contains the text data associated to $45$ millions links \citep{radford_et_al_2019}. 
The growth in complexity of these models seems to know no limit, especially with regards to their number of parameters. 
For instance, BERT \citep{devlin_et_al_2018} has roughly $340$ millions of parameters, a meager number compared to more recent models such as GTP-2 \citep[$1.5$ billions]{radford_et_al_2019} and GPT-3 \citep[$175$ billions]{brown_et_al_2020}.

Faced with such giants, it is becoming more and more challenging to understand how particular predictions are made. 
Yet, \emph{interpretability} of these algorithms is an urgent need. 
This is especially true in some applications such as healthcare, where natural language processing is used for instance to obtain summaries of patients records \citep{spyns_1996}. 
In such cases, we do not want to deploy in the wild an algorithm making near perfect predictions on the test set but for the wrong reasons: the consequences could be tragic. 

\begin{figure}
    \centering
\includegraphics[scale=0.21]{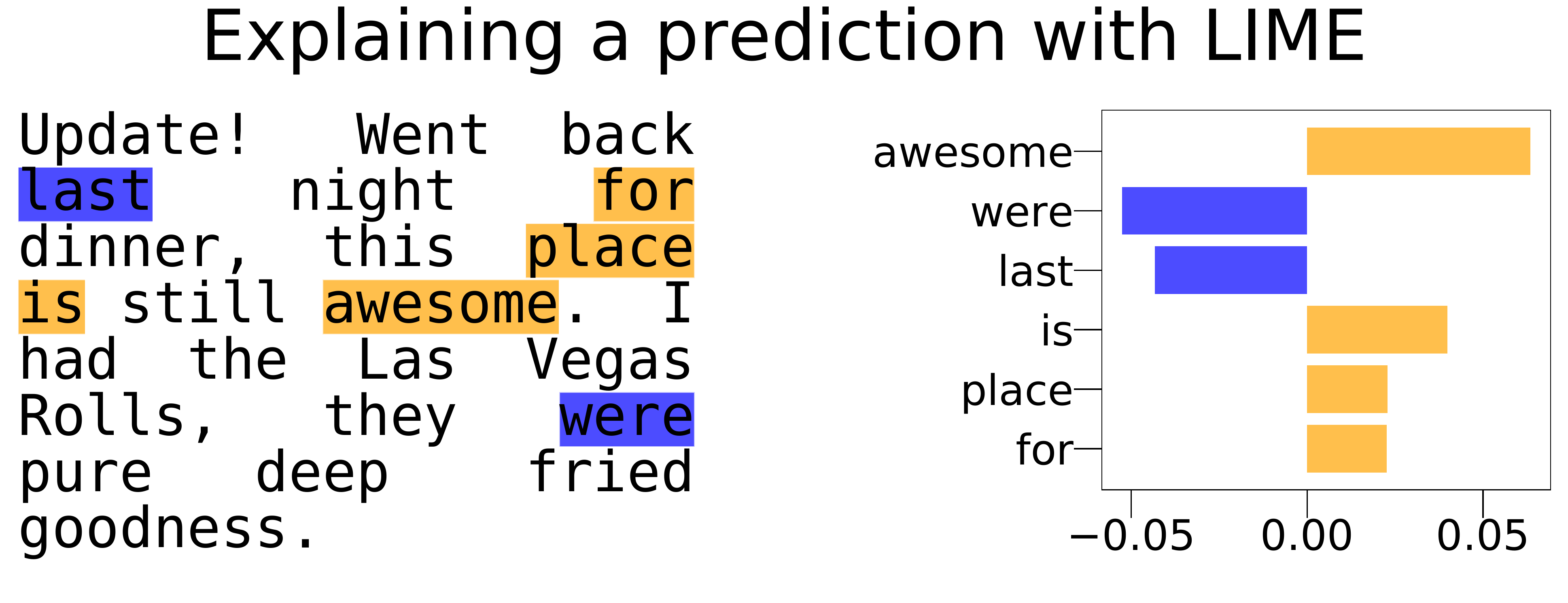}
    \vspace{-0.3in}
    \caption{\label{fig:general-explanation}Explaining the prediction of a random forest classifier on a Yelp review. \emph{Left panel:} the document to explain. The words deemed important for the prediction are highlighted, in orange (positive influence) and blue (negative influence). \emph{Right panel:}  values of the largest $6$ interpretable coefficients, ranked by absolute value.  }
\end{figure}

In this context, a flourishing literature proposing \emph{interpretability methods} emerged. 
We refer to the survey papers of \citet{Guidotti_et_al_2018} and \citet{Adadi_Berrada_2018} for an overview, and to \citet{danilevsky_et_al_2020} for a focus on natural language processing. 
With the notable exception of SHAP \citep{Lundberg_Lee_2017}, these methods do not come with any guarantees. 
Namely, given a simple model already interpretable to some extent, we cannot be sure that these methods provide meaningful explanations. 
For instance, explaining a model that is based on the presence of a given word should return an explanation that gives high weight to this word. 
Without such guarantees, using these methods on the tremendously more complex models aforementioned seems like a risky bet.  

In this paper, we focus on one of the most popular interpretability method: \emph{Local Interpretable Model-agnostic Explanations} \citet[LIME]{ribeiro_et_al_2016}, and more precisely its implementation for text data. 
LIME's process to explain the prediction of a model~$f$ for an example~$\xi$ can be summarized as follows: 
\begin{enumerate}[(i).,itemsep=1pt,topsep=0pt]
    \item from a corpus of documents $\corp$, create a TF-IDF transformer $\Normtfidf$ embedding documents into $\Reals^D$;
    \item create $n$ perturbed documents $x_1,\ldots,x_n$ by deleting words at random in $\xi$; 
    \item for each new example, get the prediction of the model $y_i\defeq f(\normtfidf{x_i})$;
    \item train a (weighted) linear surrogate model with inputs the absence / presence of words and responses the $y_i$s.
\end{enumerate}
The user is then given the coefficients of the surrogate model (or rather a subset of the coefficients, corresponding to the largest ones) as depicted in Figure~\ref{fig:general-explanation}. 
We call these coefficients the \emph{interpretable coefficients}. 

The model-agnostic approach of LIME has contributed greatly to its popularity: one does not need to know the precise architecture of~$f$ in order to get explanations, it is sufficient to be able to query~$f$ a large number of times. 
The explanations provided by the user are also very intuitive, making it easy to check that a model is behaving in the appropriate way (or not!) on a particular example. 

\paragraph{Contributions. }
In this paper, we present the first theoretical analysis of LIME for text data. 
%More precisely,
In detail,

\begin{itemize}[itemsep=1pt,topsep=0pt]
    \item we show that, when the number of perturbed samples is large, \textbf{the interpretable coefficients concentrate with high probability around a fixed vector $\beta$} that depends only on the model, the example to explain, and hyperparameters of the method;
    \item we provide an \textbf{explicit expression of $\beta$}, from which we gain interesting insights on LIME. In particular, \textbf{the explanations provided are linear in $f$};
    \item for simple decision trees, we go further into the computations. We show that \textbf{LIME provably provides meaningful explanations}, giving large coefficients to words that are pivotal for the prediction;
    \item for linear models, we come to the same conclusion by showing that the interpretable coefficient associate to a given word is approximately equal to \textbf{the product of the coefficient in the linear model and the TF-IDF transform of the word} in the example. 
\end{itemize}

We want to emphasize that all our results apply to the default implementation of LIME for text data\footnote{\url{https://github.com/marcotcr/lime}} (as of October 12, 2020), with the only caveat that we do not consider any feature selection procedure in our analysis. 
All our theoretical claims are supported by numerical experiments, the code thereof can be found  at \url{https://github.com/dmardaoui/lime_text_theory}.

\paragraph{Related work. }
The closest related work to the present paper is \citet{garreau_luxburg_2020_aistats}, in which the authors provided a theoretical analysis of a variant of LIME in the case of tabular data (that is, unstructured data belonging to $\Reals^N$) when $f$ is linear. 
This line of work was later extended by the same authors  \citep{garreau_luxburg_2020_arxiv}, this time in a setting very close to the default implementation and for other classes of models (in particular partition-based classifiers such as CART trees and kernel regressors built on the Gaussian kernel). 
While uncovering a number of good properties of LIME, these analyses also exposed some weaknesses of LIME, notably cancellation of interpretable features for some choices of hyperparameters. 

The present work is quite similar in spirit, however we are concerned with \emph{text data}. 
The LIME algorithm operates quite differently in this case. 
In particular, the input data goes first through a TF-IDF transform (a non-linear transformation) and there is no discretization step since interpretable features are readily available (the words of the document). 
Therefore both the analysis and our conclusions are quite different, as it will become clear in the rest of the paper. 

\section{LIME for text data}
\label{sec:lime}

In this section, we lay out the general operation of LIME for text data and introduce our notation in the process. 
From now on, we consider a model~$f$ and look at its prediction for a fixed example~$\xi$ belonging to a corpus~$\corp$ of size~$N$, which is built on a dictionary~$\dg$ of size~$D$. 
We let $\norm{\cdot}$ denote the Euclidean norm, and $\sphere{D-1}$ the unit sphere of $\Reals^D$. 

Before getting started, let us note that LIME is usually used in the \emph{classification} setting: $f$ takes values in $\{0,1\}$ (say), and $f(\normtfidf{\xi})$ represents the class attributed to~$\xi$ by~$f$. 
However, behind the scenes, LIME requires~$f$ to be a real-valued function. 
In the case of classification, this function is the probability of belonging to a certain class according to the model. 
In other words, the \emph{regression} version of LIME is used, and this is the setting that we consider in this paper. 
We now detail each step of the algorithm.

\subsection{TF-IDF transform}
\label{sec:tfidf}

LIME works with a vector representation of the documents.  
The TF-IDF transform \citep{luhn_1957,jones_1972} is a popular way to obtain such a representation. 
The idea underlying the TF-IDF is quite simple: to any document, associate a vector of size~$D$. 
If we set $\word_1,\ldots,\word_D$ to be our dictionary, the $j$th component of this vector represents the importance of word~$\word_j$. 
It is given by the product of two terms: the term frequency (TF, how frequent the word is in the document), and the inverse term frequency (IDF, how rare the word is in our corpus). 
Intuitively, the TF-IDF of a document has a high value for a given word if this word is frequent in the document and, at the same time, not so frequent in the corpus. 
In this way, common words such as ``the'' do not receive high weight. 

Formally, let us fix $\delta \in\corp$. 
For each word $\word_j\in\dg$, we set $m_j$ the number of times $\word_j$ appears in $\delta$. 
We also set $v_j\defeq \log \frac{N+1}{N_j+1}+1$, where~$N_j$ is the number of documents in~$\corp$ containing~$\word_j$. 
%This is where the pre-computing phase takes place: w
When presented with~$\corp$, we can pre-compute all the $v_j$s and at run time we only need to count the number of occurrences of~$\word_j$ in~$\delta$. 
We can now define the normalized TF-IDF:

\begin{definition}[Normalized TF-IDF]
\label{def:tf-idf}
We define the \emph{normalized TF-IDF} of $\delta$ as the vector $\normtfidf{\delta}\in\Reals^D$ defined coordinate-wise by 
\begin{equation}
\label{eq:def-norm-tf-idf}
\forall 1\leq j\leq D,\quad \normtfidf{\delta}_j \defeq \frac{m_j v_j}{\sqrt{\sum_{j=1}^D m_j^2v_j^2}}
\, .
\end{equation}
In particular, $\norm{\phi(\delta)}=1$, where $\norm{\cdot}$ is the Euclidean norm. 
\end{definition}

Note that there are many different ways to define the TF and IDF terms, as well as normalization choices.  
We restrict ourselves to the version used in the default implementation of LIME, with the understanding that different implementation choices would not change drastically our analysis. 
For instance, normalizing by the $\ell_1$ norm instead of the $\ell_2$ norm would lead to slightly different computations in Proposition~\ref{prop:beta-computation-linear-main}. 

Finally, note that this transformation step does not take place for tabular data, since the data already belong to $\Reals^D$ in this case. 

\subsection{Sampling}
\label{sec:sampling}

Let us now fix a given document $\xi$ and describe the sampling procedure of LIME. 
Essentially, the idea is to sample new documents similar to $\xi$ in order to see how~$f$ varies in a neighborhood of $\xi$. 

\begin{figure}
    \centering
\includegraphics[scale=0.18]{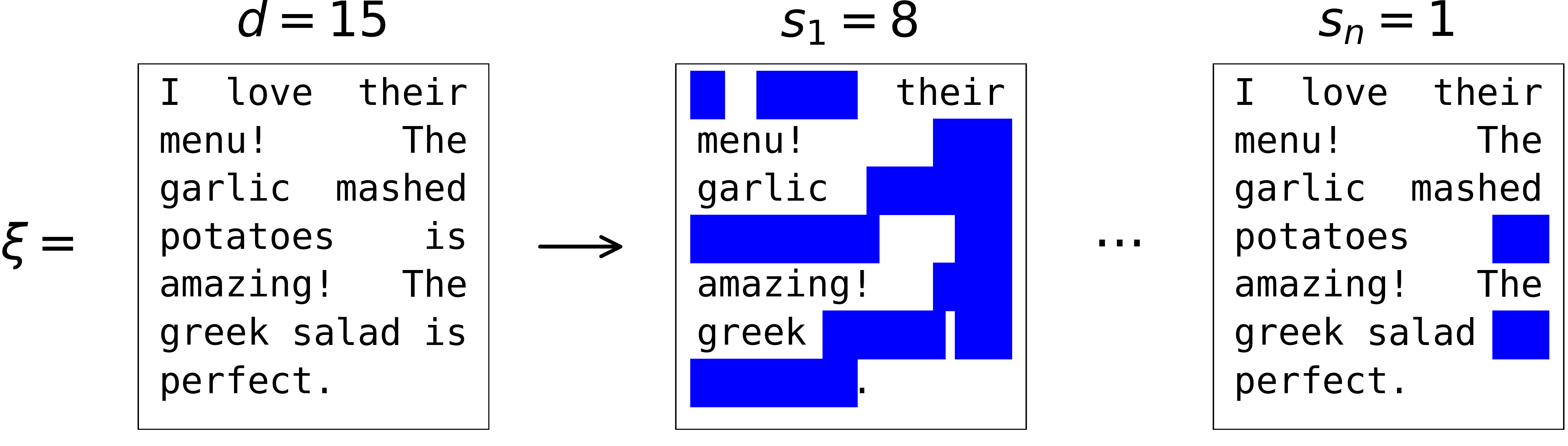}
    \caption{\label{fig:sampling}The sampling scheme of LIME for text data. To the left, the document to explain $\xi$, which contains $d=15$ distinct words. The new samples $x_1,\ldots,x_n$ are obtained by removing $s_i$ random words from $\xi$ (in blue). In the $n$th sample, one word is removed, yielding two deletions in the original document.}
\end{figure}

More precisely, let us denote by $d$ the number of distinct words in $\xi$ and set $\dl\defeq \{\word_1,\ldots,\word_d\}$ the \emph{local dictionary}. 
For each new sample, LIME first draws uniformly at random in $\{1,\ldots,d\}$ a number~$s_i$ of words to remove from~$\xi$. 
Subsequently, a subset $S_i\subseteq \{1,\ldots,d\}$ of size~$s_i$ is drawn uniformly at random: all the words with indices contained in $S_i$ are \emph{removed} from~$\xi$. 
Note that the multiplicity of removals is independent from $s_i$: if the word ``good'' appears $10$ times in $\xi$ and its index belongs to $S$, then all the instances of ``good'' are removed from $\xi$ (see Figure~\ref{fig:sampling}). 
This process is repeated~$n$ times, yielding~$n$ new samples $x_1,\ldots,x_n$. 
With these new documents come~$n$ new binary vectors $z_1,\ldots,z_n\in\{0,1\}^d$, marking the absence or presence of a word in $x_i$. 
Namely, $z_{i,j}=1$ if $\word_j$ belongs to $x_i$ and $0$ otherwise. 
We call the $z_i$s the \emph{interpretable features}. 
Note that we will write $\Indic\defeq (1,\ldots,1)^\top$ for the binary feature associated to~$\xi$: all the words are present. 

Already we see a difficulty appearing in our analysis: when removing words from $\xi$ at random, $\normtfidf{\xi}$ is modified in a non-trivial manner. 
In particular, the denominator of Eq.~\eqref{eq:def-norm-tf-idf} can change drastically if many words are removed.  

In the case of tabular data, the interpretable features are obtained in a completely different fashion, by discretizing the dataset. 

\subsection{Weights}

Let us start by defining the \emph{cosine distance}: 

\begin{definition}[Cosine distance]
For any $u,v\in\Reals^d$, we define
\begin{equation}
\label{eq:def-cos-distance}
\distcos{u}{v} \defeq 1 - \frac{u\cdot v}{\norm{u}\cdot \norm{v}}
\, .
\end{equation}
\end{definition}

Intuitively, the cosine distance between~$u$ and~$v$ is small if the \emph{angle} between~$u$ and~$v$ is small. 
Each new sample~$x_i$ receives a positive weight~$\pi_i$, defined~by
\begin{equation}
\label{eq:def-weights}
\pi_i \defeq \exp{\frac{-\distcos{\Indic}{z_i}^2}{2\nu^2}}
\, ,
\end{equation}
where $\nu$ is a positive \emph{bandwidth parameter}. 
The intuition behind these weights is that $x_i$ can be far away from $\xi$ if many words are removed (in the most extreme case, $s=d$, all the words from $\xi$ are removed). 
In that case, $z_i$ has mostly $0$ components, and is far away from $\Indic$.

Note that the cosine distance in Eq.~\eqref{eq:def-weights} is actually multiplied by $100$ in the current implementation of LIME. 
Thus there is the following correspondence between our notation and the code convention: $\nu_{\text{LIME}}=100\nu$. 
For instance, the default choice of bandwidth, $\nu_{\text{LIME}}=25$, corresponds to $\nu=0.25$. 

We now make the following important remark: \textbf{the weights only depends on the number of deletions.} 
Indeed, conditionally to $S_i$ having exactly $s$ elements, we have $z_i\cdot \Indic = d-s$ and $\norm{z_i}=\sqrt{d-s}$. 
Since $\norm{\Indic}=\sqrt{d}$, using Eq.~\eqref{eq:def-weights}, we deduce that $\pi_i=\psi(s/d)$, where we defined the mapping 
\begin{align}
\psi\colon [0,1]& \longrightarrow \Reals \label{eq:def-psi-main} \\
t &\longmapsto \exp{\frac{-(1-\sqrt{1-t})^2}{2\nu^2}} \notag 
\, .
\end{align}
%Then, by Eq.~\eqref{eq:def-weights}, $\pi_i=\psi(s/d)$. 
We can see in Figure~\ref{fig:psi} how the weights are given to observations: when $s$ is small, then $\psi(s/d)\approx 1$ and when $s\approx d$, $\psi(s/d)$ which is a small quantity depending on $\nu$.  
Note that the complicated dependency of the weights in $s$ brings additional difficulty in our analysis, and that we will sometimes restrict ourselves to the large bandwidth regime (that is, $\nu\to +\infty$).  
In that case, $\pi_i \approx 1$ for any $1\leq i\leq n$. 

Euclidean distance between the interpretable features is used instead of the cosine distance in the tabular data version of the algorithm.

\begin{figure}
    \centering
\includegraphics[scale=0.18]{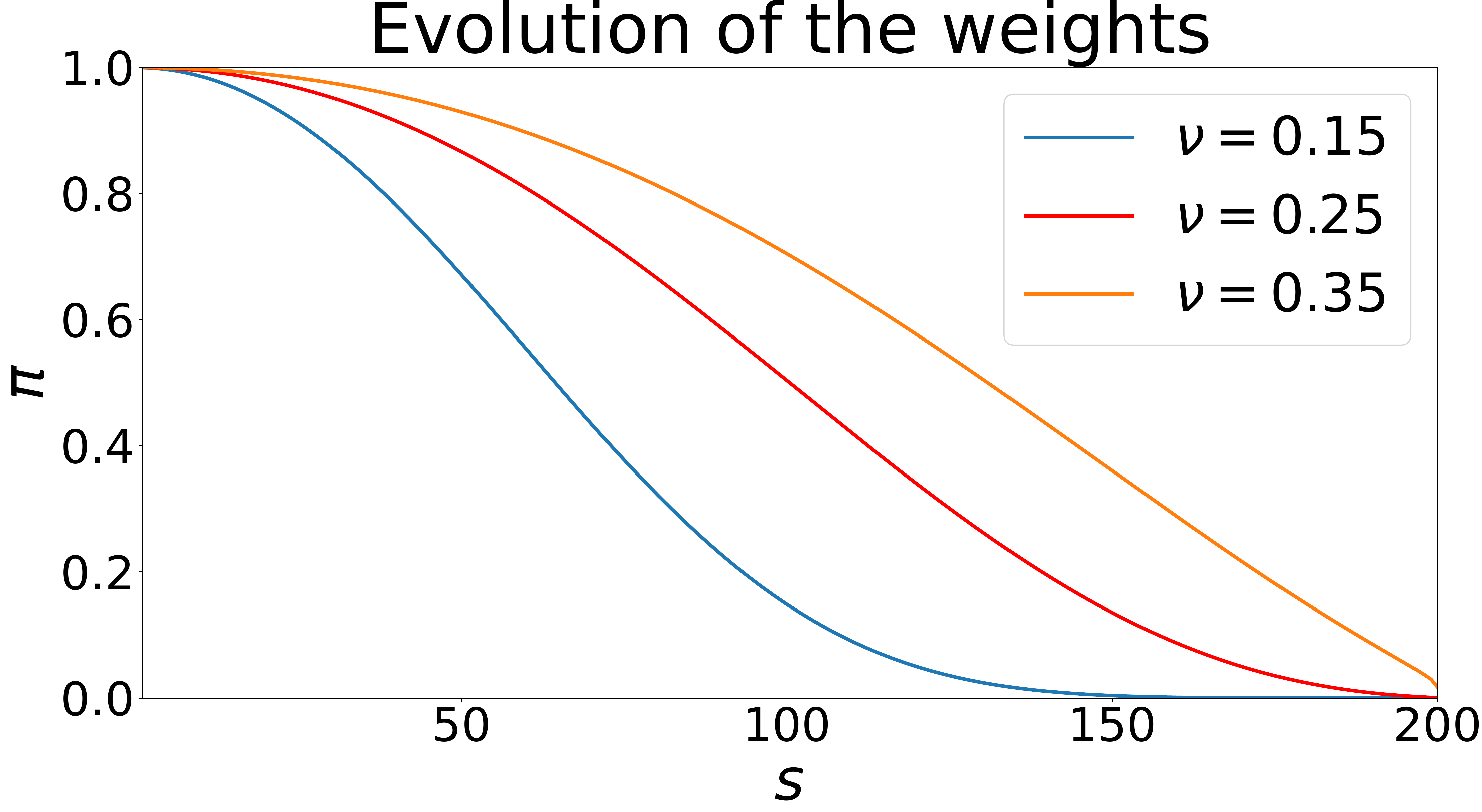}
    \caption{\label{fig:psi}Weights as a function of the number of deletions for different bandwidth parameters ($\nu=0.25$ is default). LIME gives more weights to documents with few deletions ($s/d\approx 0$ means that $\psi(s/d)\approx 1$ regardless of the bandwidth).}
\end{figure}

\subsection{Surrogate model}

The next step is to train a surrogate model on the interpretable features $z_1,\ldots,z_n$, trying to approximate the responses $y_i\defeq f(\normtfidf{x_i})$. 
In the default implementation of LIME, this model is linear and is obtained by weighted ridge regression \citep{hoerl_1970}. 
Formally, LIME outputs 
\begin{equation}
\label{eq:main-problem}
\betahat_n^{\lambda} \in\argmin{\beta\in\Reals^{d+1}} \biggl\{ \sum_{i=1}^n \pi_i(y_i - \beta^\top z_i)^2 + \lambda \norm{\beta}^2\biggr\}
\, ,
\end{equation}
where $\lambda>0$ is a regularization parameter. 
We call the components of $\betahat_n^\lambda$ the \emph{interpretable coefficients}, the $0$th coordinate in our notation is by convention the intercept. 
Note that some feature selection mechanism is often used in practice, limiting the number of interpretable features in output from LIME. 
We do not consider such mechanism in our analysis. 

We now make a fundamental observation. 
In its default implementation, LIME uses the default setting of \texttt{sklearn} for the regularization parameter, that is, $\lambda=1$. 
Hence the first term in Eq.~\eqref{eq:main-problem} is roughly of order $n$ and the second term of order $d$. 
Since we experiment in the large~$n$ regime ($n=5000$ is default) and with documents that have a few dozen distinct words, $n\gg d$. 
To put it plainly, we can consider that $\lambda=0$ in our analysis and still recover meaningful results. 
We will denote by $\betahat_n$ the solution of Eq.~\eqref{eq:main-problem} with $\lambda=0$, that is, ordinary least-squares. 

We conclude this presentation of LIME by noting that the main free parameter of the method is the bandwidth $\nu$. 
As far as we know, there is no principled way of choosing $\nu$. 
The default choice, $\nu=0.25$, does not seem satisfactory in many respects. 
In particular, other choices of bandwidth can lead to different values for interpretable coefficients. 
In the most extreme cases, they can even change sign, see Figure~\ref{fig:cancellation}. 
This phenomenon was also noted for tabular data in \citet{garreau_luxburg_2020_arxiv}. 

\begin{figure}
    \centering
\includegraphics[scale=0.2]{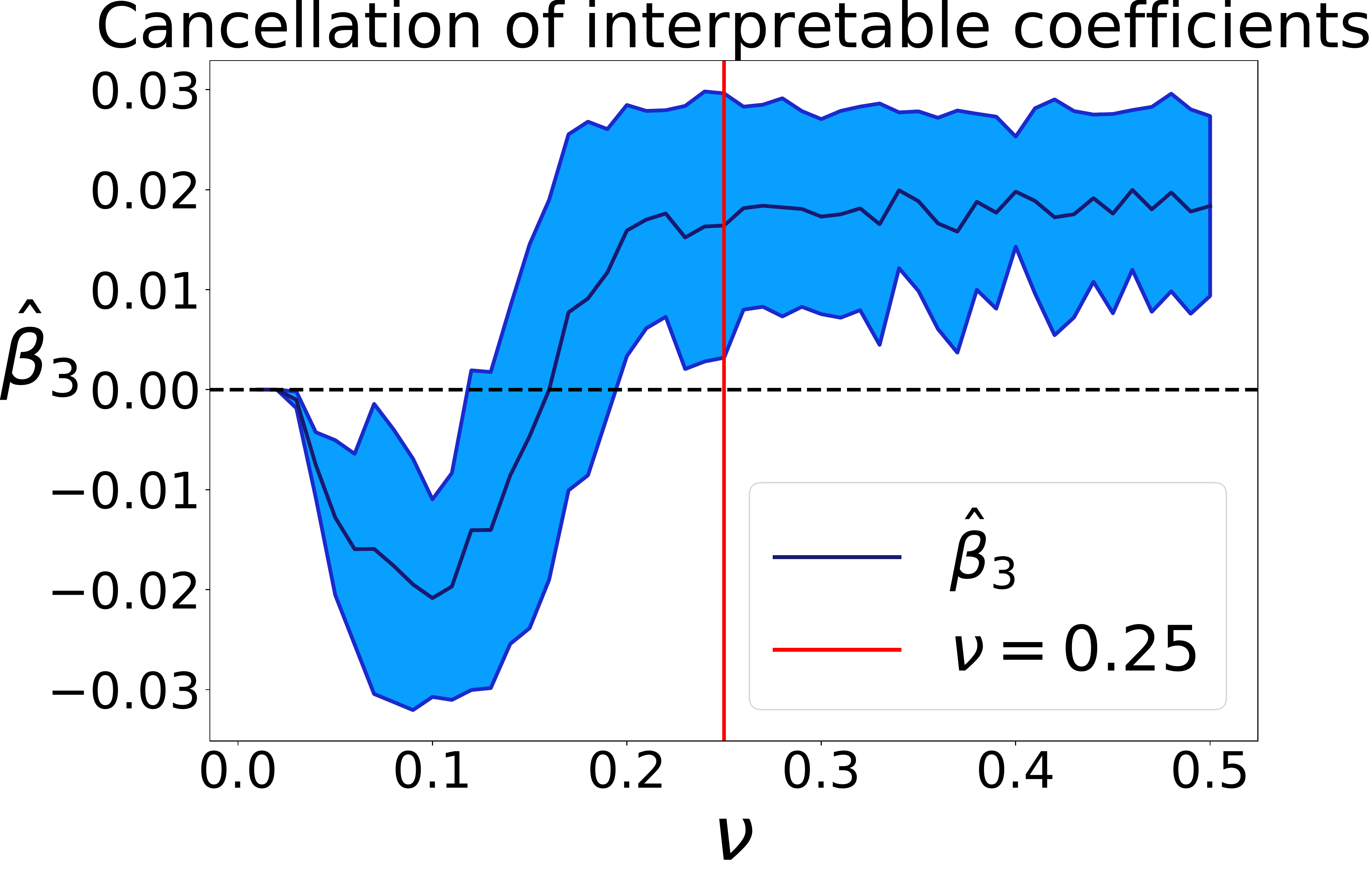}
    \vspace{-0.1in}
    \caption{\label{fig:cancellation}In this experiment, we plot the interpretable coefficient associated to the word ``came" as a function of the bandwidth parameter. The red vertical line marks the default bandwidth choice ($\nu=25$). We can see that LIME gives a negative influence for $\nu \approx 0.1$ and a positive one for $\nu > 0.2$. }
\end{figure}

\section{Main results}

Without further ado, let us present our main result. 
For clarity's sake, we split it in two parts: Section~\ref{sec:main:concentration} contains the concentration of $\betahat_n$ around $\beta^f$ whereas Section~\ref{sec:main:computation} presents the exact expression of $\beta^f$. 

\subsection{Concentration of $\betahat_n$}
\label{sec:main:concentration}

When the number of new samples~$n$ is large, we expect LIME to stabilize and the explanations not to vary too much. 
The next result supports this intuition.  

\begin{theorem}[Concentration of $\betahat_n$]
\label{th:concentration-of-betahat}
Suppose that the model $f$ is bounded by a positive constant $M$ on $\sphere{D-1}$. 
Recall that we let $d$ denote the number of distinct words of $\xi$, the example to explain. 
Let $0<\epsilon < M$ and $\eta\in (0,1)$.  
Then, there exist a vector $\beta^f\in\Reals^d$ such that, for every 
\[
n\gtrsim \max \left\{M^2d^{9} \exps{\frac{10}{\nu^2}}, Md^5\exps{\frac{5}{\nu^2}}\right\} \frac{\log \frac{8d}{\eta}}{\epsilon^2}
\, ,
\]
we have $\proba{\smallnorm{\betahat_n - \beta^f} \geq \epsilon} \leq \eta$. 
\end{theorem}

We refer to the supplementary material for a complete statement (we omitted numerical constants here for clarity) and a detailed proof. 
In essence, Theorem~\ref{th:concentration-of-betahat} tells us that we can focus on $\beta^f$ in order to understand how LIME operates, provided that~$n$ is large enough. 
The main limitation of Theorem~\ref{th:concentration-of-betahat} is the dependency of~$n$ in~$d$ and~$\nu$. 
The control that we achieve on $\smallnorm{\betahat_n-\beta}$ becomes quite poor for large~$d$ or small~$\nu$: we would then need~$n$ to be unreasonably large in order to witness concentration. 

We notice that Theorem~\ref{th:concentration-of-betahat} is very similar in its form to Theorem~1 in \citet{garreau_luxburg_2020_arxiv} except that (i) the dimension is replaced by the number of distinct words in the document to explain, and (ii) there is no discretization parameter in our case. 
The differences with the analysis in the tabular data framework will be more visible in the next section. 

\subsection{Expression of $\beta^f$}
\label{sec:main:computation}

Our next result shows that we can derive an explicit expression for $\beta^f$. 
Before stating our result, we need to introduce more notation. 
From now on, we set $x$ a random variable such that $x_1,\ldots,x_n$ are i.i.d. copies of $x$. 
Similarly,~$\pi$ corresponds to the draw of the $\pi_i$s and $z$ to that of the~$z_i$s. 

\begin{definition}[$\alpha$ coefficients]
\label{def:alphas}
Define $\alpha_0\defeq \expec{\pi}$ and, for any $1\leq p\leq d$, 
\begin{equation}
\label{eq:def-alphas-main}
\alpha_p  \defeq \expec{\pi \cdot z_1 \cdots z_p }
\, .
\end{equation}
\end{definition}

Intuitively, when $\nu$ is large, $\alpha_p$ corresponds to the probability that $p$ distinct words are present in $x$. 
The sampling process of LIME is such that $\alpha_p$ does not depend on the exact set of indices considered. 
In fact,~$\alpha_p$ only depends on~$d$ and~$\nu$. 
We show in the supplementary material that it is possible to compute the $\alpha$ coefficients in closed-form as a function of~$d$ and~$\nu$:

\begin{proposition}[Computation of the $\alpha$ coefficients]
\label{prop:alphas-computation-main}
Let $0\leq p\leq d$. 
For any $d\geq 1$ and $\nu >0$, it holds that 
\[
\alpha_p = \frac{1}{d} \sum_{s=1}^d \prod_{k=0}^{p-1} \frac{d-s-k}{d-k} \psi\left(\frac{s}{d}\right)
\, .
\]
\end{proposition}

From these coefficients, we form the normalization constant
\begin{equation}
\label{eq:def-densct}
\dencst_d \defeq (d-1)\alpha_0\alpha_2 -d\alpha_1^2 + \alpha_0\alpha_1
\, .
\end{equation}

We will also need the following. 

\begin{definition}[$\sigma$ coefficients]
\label{def:sigmas}
For any $d\geq 1$ and $\nu >0$, define
\begin{equation}
\label{eq:def-sigmas}
\begin{cases}
\sigma_1 &\defeq -\alpha_1
\, , \\
\sigma_2 &\defeq \frac{(d-2)\alpha_0 \alpha_2 - (d-1)\alpha_1^2 + \alpha_0\alpha_1}{\alpha_1-\alpha_2}\, , \\
\sigma_3 &\defeq \frac{\alpha_1^2-\alpha_0\alpha_2}{\alpha_1-\alpha_2 }
\, .
\end{cases}
\end{equation}
\end{definition}

With these notation in hand, we have:

\begin{proposition}[Expression of $\beta^f$]
\label{prop:expression-of-beta}
Under the assumptions of Theorem~\ref{th:concentration-of-betahat}, we have $\dencst_d >0$ and, for any $1\leq j\leq d$,
\begin{align}
\label{eq:def-beta}
\beta_j^f = \dencst^{-1}_d\biggl\{\sigma_1 \expec{\pi f(\normtfidf{x})} & + \sigma_2 \expec{\pi z_j f(\normtfidf{x})} \\
&+ \sigma_3 \sum_{\substack{k=1 \\ k\neq j}}^d \expec{\pi z_k f(\normtfidf{x})}\biggr\} \notag 
\, .
\end{align}
\end{proposition}

We also have an expression for the intercept which can be found in the supplementary material, as well as the proof of Proposition~\ref{prop:expression-of-beta}. 
At first glance, Eq.~\eqref{eq:def-beta} is quite similar to Eq.~(6) in \citet{garreau_luxburg_2020_arxiv}, which gives the expression of $\beta_j^f$ in the tabular data case. 
The main difference is the TF-IDF transform in the expectation, personified by $\Normtfidf$, and the additional terms (there is no $\sigma_3$ factor in the tabular data case). 
In addition, the expression of the $\sigma$ coefficients is much more complicated than in the tabular data case. 
We now present some immediate consequences of Proposition~\ref{prop:expression-of-beta}. 

\paragraph{Linearity of explanations. }
Perhaps the most striking feature of Eq.~\eqref{eq:def-beta} is that it is \textbf{linear in $f$}.
More precisely, the mapping $f\mapsto \beta^f$ is linear in $f$: for any given two functions $f$ and $g$, we have
\[
\beta^{f+g} = \beta^f + \beta^g
\, .
\]
Therefore, because of Theorem~\ref{th:concentration-of-betahat}, the explanations $\betahat_n$ obtained for a finite sample of new examples are also approximately linear in the model to explain. 
We illustrate this phenomenon in Figure~\ref{fig:linearity}. 
This is remarkable: many models used in machine learning can be written as a linear combination of smaller models (\emph{e.g.}, generalized linear models, kernel regressors, decision trees and random forests). 
In order to understand the explanations provided by these complicated models, one can try and understand the explanations for the elementary elements of the models first. 

\paragraph{Large bandwidth. }
It can be difficult to get a good sense of the values taken by the $\sigma$ coefficients, and therefore of $\beta$. 
Let us see how Proposition~\ref{prop:expression-of-beta} simplifies in the large bandwidth regime and what insights we can gain. 
We denote by $\betainf$ the limit of $\beta$ when $\nu\to +\infty$. 
When $\nu\to +\infty$, we prove in the supplementary material that, for any $1\leq j\leq d$, up to $\bigo{1/d}$ terms and a numerical constant, the $j$-th coordinate of $\betainf$ is then approximately equal to 
\begin{equation*}
\left(\betainf^f\right)_j\! \approx\! \condexpec{f(\normtfidf{x})}{\word_j\in x} - \frac{1}{d}\sum_{k\neq j} \condexpec{f(\normtfidf{x})}{\word_k\in x}
.
\end{equation*}
Intuitively, the interpretable coefficient associated to the word $\word_j$ is high if \textbf{the expected value of the model when word $\word_j$ is present is significantly higher than the typical expected value when other words are present}. 
We think that this is reasonable: if the model predicts much higher values when $\word_j$ belongs to the example, it surely means that~$\word_j$ being present is important for the prediction. 
Of course, this is far from the full picture, since (i) this reasoning is only valid for large bandwidth, and (ii) in practice, we are concerned with $\betahat_n$ which may be not so close to $\beta^f$ for small $n$. 

\begin{figure}
    \centering
\includegraphics[scale=0.25]{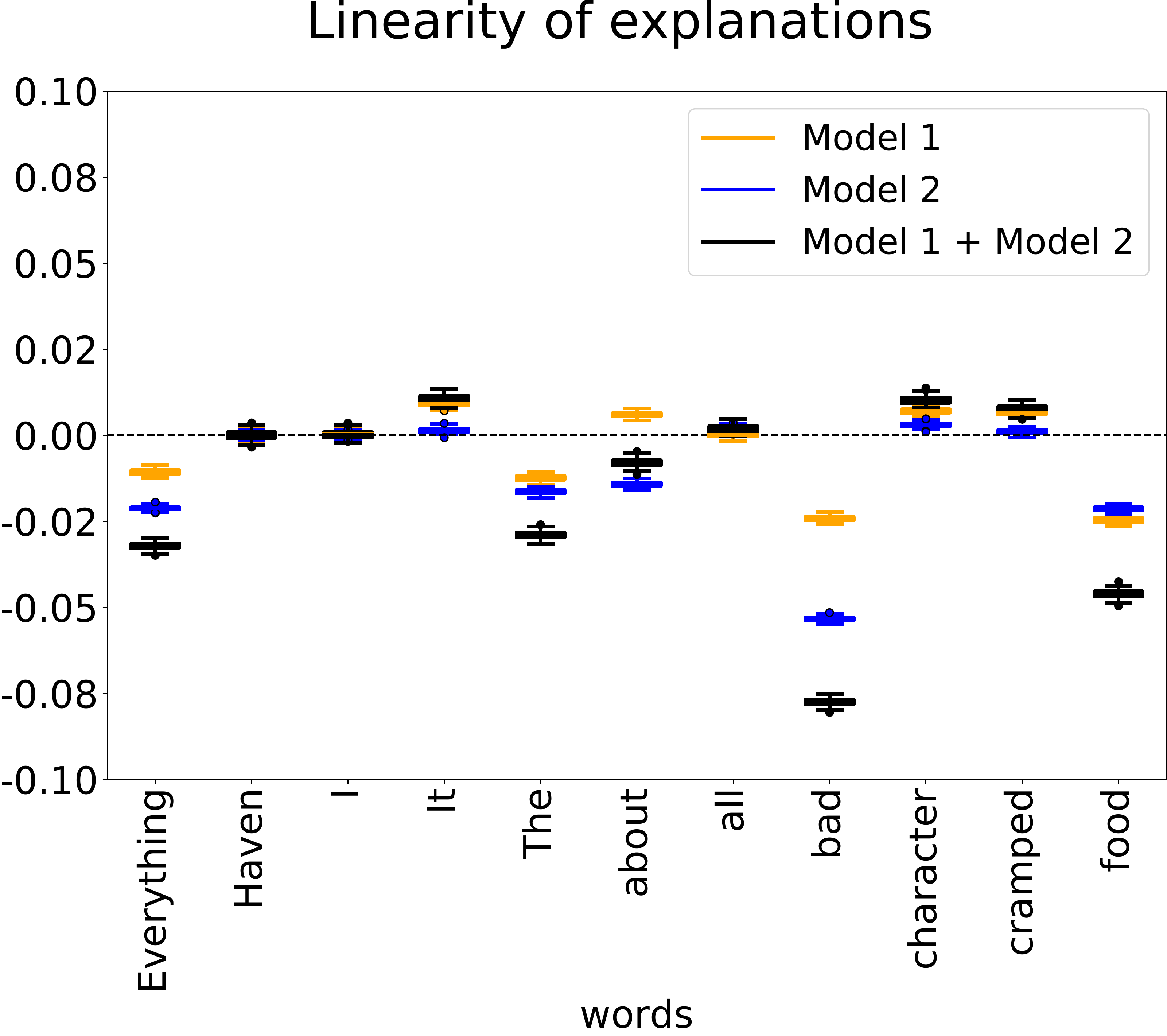}
    \vspace{-0.1in}
    \caption{\label{fig:linearity}The explanations given by LIME for the sum of two models (here two random forests regressors) are the sum of the explanations for each model, up to noise coming from the sampling procedure.}
\end{figure}

%%%%%%%%%%%%%%%%%%%%%%%%%%%%%%%%%%%%%%%%%%%%%%%%%%%%%%%%%%%

\subsection{Sketch of the proof}

We conclude this section with a brief sketch of the proof of Theorem~\ref{th:concentration-of-betahat}, the full proof can be found in the supplementary material. 

Since we set $\lambda=0$ in Eq.~\eqref{eq:main-problem}, $\betahat_n$ is the solution of a weighted least-squares problem. 
Denote by $W\in\Reals^{n\times n}$ the diagonal matrix such that $W_{i,i}=\pi_i$, and set $Z\in\{0,1\}^{n\times (d+1)}$ the matrix such that its $i$th line is $(1,z_i^\top)$. 
Then the solution of Eq.~\eqref{eq:main-problem} is given by 
\[
\betahat_n = \left(Z^\top WZ\right)^{-1}Z^\top Wy
\, ,
\]
where we defined $y\in\Reals^n$ such that $y_i=f(\normtfidf{x_i})$ for all $1\leq i\leq n$. 
Let us set $\Sigmahat_n\defeq \frac{1}{n}Z^\top WZ$ and $\Gammahat_n^f\defeq \frac{1}{n}Z^\top Wy$. 
By the law of large numbers, we know that both $\Sigmahat_n$ and $\Gammahat_n^f$ converge in probability towards their population counterparts $\Sigma\defeq \smallexpec{\Sigmahat_n}$ and $\Gamma^f\defeq \smallexpec{\Gammahat_n}$. 
Therefore, provided that $\Sigma$ is invertible, $\betahat_n$ is close to $\beta^f\defeq \Sigma^{-1}\Gamma^f$ with high probability. 

% differences with tabular
As we have seen in Section~\ref{sec:lime}, the main differences with respect to the tabular data implementation are (i) the interpretable features, and (ii) the TF-IDF transform. 
The first point lead to a completely different $\Sigma$ than the one obtained in \citet{garreau_luxburg_2020_arxiv}. 
In particular, it has no zero coefficients, leading to more complicated expression for $\beta^f$ and additional challenges when controlling $\opnorm{\Sigma^{-1}}$. 
The second point is quite challenging since, as noted in Section~\ref{sec:tfidf}, \textbf{the TF-IDF transform of a document changes radically when deleting words at random in the document.} 
This is the main reason why we have to resort to approximations when dealing with linear models.

%%%%%%%%%%%%%%%%%%%%%%%%%%%%%%%%%%%%%%%%%%%%%%%%%%%%%%%%%%%%%%%%%%%%%%%%%%%%%%%%

\section{Expression of $\beta^f$ for simple models}
\label{sec:discussion}

In this section, we see how to specialize Proposition~\ref{prop:expression-of-beta} to simple models $f$. 
Recall that our main goal in doing so is to investigate whether it makes sense or not to use LIME in these cases. 
We will focus on two classes of models: decision trees (Section~\ref{sec:decision-trees}) and linear models (Section~\ref{sec:linear-models}). 

\subsection{Decision trees}
\label{sec:decision-trees}

In this section we focus on simple decision trees built on the presence or absence of given words. 
For instance, let us look at the model returning $1$ if the word ``food'' is present, or if ``about'' and ``everything'' are present in the document. 
Ideally, LIME would give high positive weights to ``food,'' ``about,'' and ``everything,'' if they are present in the document to explain, and small weight to all other words. 

We first notice that such simple decision trees can be written as sums of products of the binary features. 
Indeed, recall that we defined $z_j=\indic{\word_j\in x}$. 
For instance, suppose that the first three words of our dictionary are ``food,'' ``about,'' and ``everything.''
Then the model from the previous paragraph can be written 
\begin{equation}
\label{eq:def-g}
g(x) = z_1 + (1-z_1)\cdot z_2 \cdot z_3
\, .
\end{equation}

Now it is clear that the $z_j$s can be written as function of the TF-IDF transform of a word, since $\word_j\in x$ if, and only if, $\normtfidf{x}_j > 0$. 
Therefore this class of models falls into our framework and we can use Theorem~\ref{th:concentration-of-betahat} and Proposition~\ref{prop:expression-of-beta} in order to gain insight on the explanations provided by LIME. 
For instance, Eq.~\eqref{eq:def-g} can be written as $f(\normtfidf{x})$ with, for any $\zeta\in\Reals^D$, 
\[
f(\zeta) \defeq \indic{\zeta_1 > 0} + (1-\indic{\zeta_1>0}) \cdot \indic{\zeta_2 > 0} \cdot \indic{\zeta_3 > 0}
\, .
\]
By linearity, it is sufficient to know how to compute~$\beta^f$ when~$f$ is a product of indicator functions. 

We now make an important remark: since the new example $x_1,\ldots,x_n$ are created by deleting words at random from the text $\xi$, \textbf{$x$ only contains words that are already present in $\xi$}. 
Therefore, without loss of generality, we can restrict ourselves to the local dictionary (the distinct words of $\xi$). 
Indeed, for any word $\word$ not already in $\xi$, $\indic{\word \in x}=0$ almost surely. 
As before, we denote by $\dl$ the local dictionary associated to $\xi$, and we denote its elements by $\word_1,\ldots,\word_d$. 
We can compute in closed-form the interpretable coefficients for a product of indicator functions:

\begin{proposition}[Computation of $\beta^f$, product of indicator functions]
\label{prop:beta-computation-indicator-product-general-main}
Let $J\subseteq \{1,\ldots,d\}$ be a set of~$p$ distinct indices and set $f(x) = \prod_{j\in J}\indic{x_j>0}$. 
Then, for any $j\in J$, 
\begin{align*}
\beta_j^f \!&=\! \dencst_d^{-1}\!\bigl[\sigma_1\alpha_p + \sigma_2\alpha_p + (d\!-\!p)\sigma_3\alpha_{p+1} + (p\!-\!1)\sigma_3\alpha_p\bigr]
\end{align*}
and, for any $j\in\{1,\ldots,d\}\setminus J$, 
\begin{align*}
\beta_j^f \!&=\! \dencst_d^{-1}\!\bigl[\sigma_1\alpha_p+\sigma_2\alpha_{p+1}+(d\!-\!p\!-\!1)\sigma_3\alpha_{p+1} + p\sigma_3\alpha_p \bigr]
.
\end{align*}
\end{proposition}

In particular, when $p=0$, Proposition~\ref{prop:beta-computation-indicator-product-general-main} simplifies greatly and we find that $1\leq k\leq d$, $\beta_k^f=\indic{k=j}$. 
It is already a reassuring result: when the model is just indicating if a given word is present, \textbf{the explanation given by LIME is one for this word and zero for all the other words}. 

It is slightly more complicated to see what happens when $p\geq 1$. 
To this extent, let us set $j\in J$ and $k\notin J$. 
Then it follows readily from Proposition~\ref{prop:beta-computation-indicator-product-general} that
\[
\beta^f_j - \beta_k^f = \dencst_d^{-1}(\sigma_2+\sigma_3)(\alpha_p-\alpha_{p+1})
\, .
\]
Since $\alpha_p\approx 1/(p+1)$ and $\sigma_2+\sigma_3\approx 6$, we deduce that $\beta_j^f \gg \beta_k^f$. 
Moreover, from Definition~\ref{def:alphas} and~\ref{def:sigmas} one can show that $\beta_k^f = \bigo{1/d}$ when $\nu$ is large. 
Thus Proposition~\ref{prop:beta-computation-indicator-product-general} tells us that  \textbf{LIME gives large positive coefficients to words that are in the support of~$f$ and small coefficients to all the other words}. 
This is a satisfying property. 

Together with the linearity property, Proposition~\ref{prop:beta-computation-indicator-product-general} allows us to compute $\beta^f$ for any decision tree that can be written as in Eq.~\eqref{eq:def-g}. 
We give an example of our theoretical predictions in Figure~\ref{fig:decision-tree-result}. 
As predicted, \textbf{the words that are pivotal in the prediction have high interpretable coefficients, whereas the other words receive near-zero coefficients}. 
It is interesting to notice that words that are near the root of the tree receive a greater weight. 
We present additional experiments in the supplementary material.

\begin{figure}
    \centering
\includegraphics[scale=0.25]{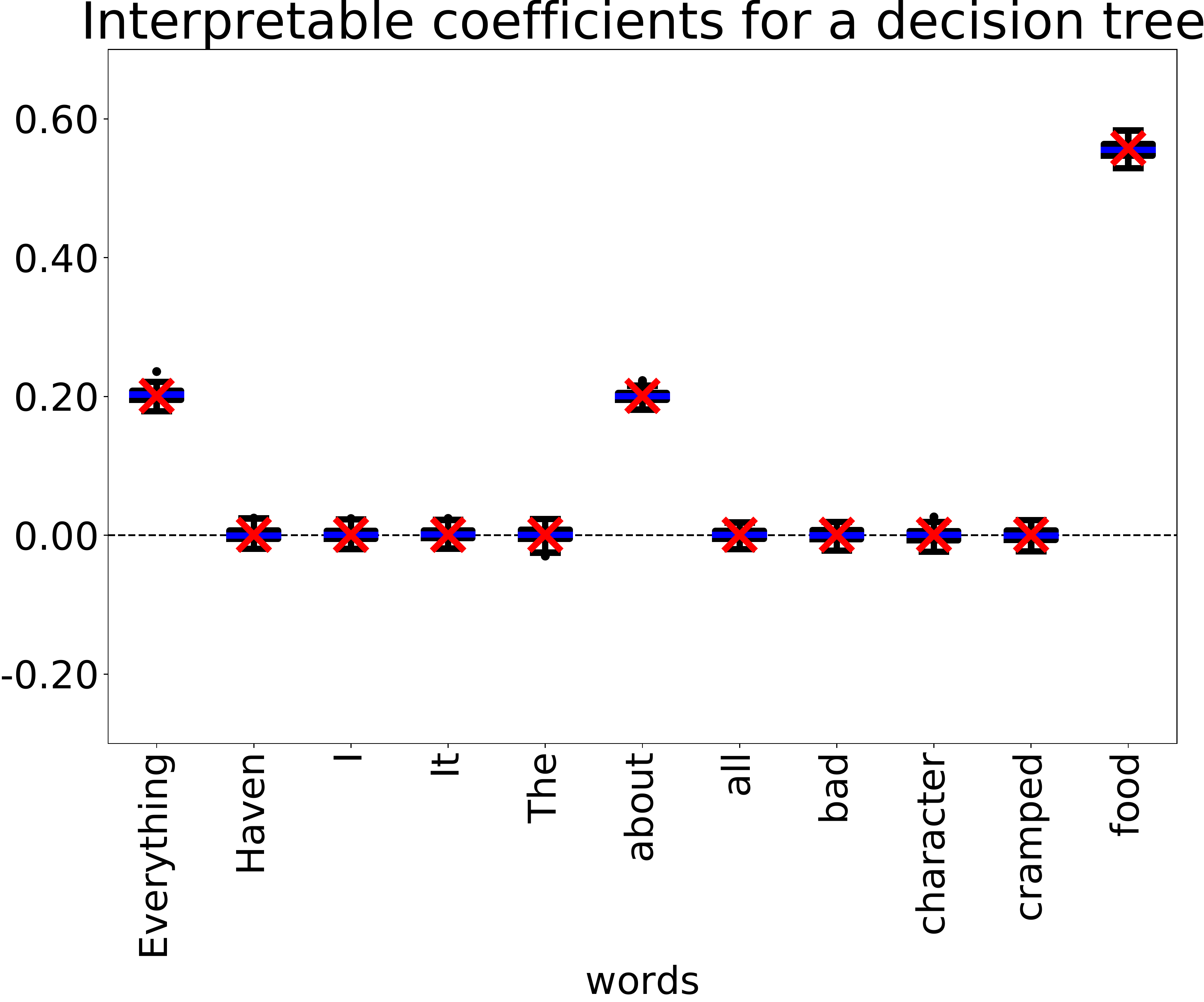}
    \vspace{-0.1in}
    \caption{\label{fig:decision-tree-result}Theory \emph{vs} practice for the tree defined by Eq.~\eqref{eq:def-g}. The black whisker boxes correspond to $100$ runs of LIME with default settings ($n=5000$ new examples and $\nu=0.25$) whereas the red crosses correspond to the theoretical predictions given by our analysis. The example to explain is a Yelp review with $d=35$ distinct words.}
\end{figure}

\subsection{Linear models}
\label{sec:linear-models}

We now focus on linear models, that is, for any document $x$,
\begin{equation}
\label{eq:def-linear-model-main}
f(\normtfidf{x}) \defeq \sum_{j=1}^d \lambda_j \normtfidf{x}_j
\, ,
\end{equation}
where $\lambda_1,\ldots,\lambda_d$ are arbitrary fixed coefficients. 
We have to resort to approximate computations in this case: from now on, we assume that $\nu = +\infty$. 
We start with the simplest linear function: all coefficients are zero except one, that is, $\lambda_k=1$ if $k=j$ and $0$ otherwise in Eq.~\eqref{eq:def-linear-model-main}, for a fixed index~$j$. 
We need to introduce additional notation before stating or result. 
For any $1\leq j\leq d$, define
\[
\omega_k \defeq \frac{m_j^2v_j^2}{\sum_{\ell=1}^d m_\ell^2v_\ell^2}
\, ,
\]
where the $m_k$s and $v_k$s were defined in Section~\ref{sec:tfidf}. 
For any~$J$ that is a strict subset of $\{1,\ldots,d\}$, define $H_S\defeq \sum_{j\in J}\omega_j$. 
Recall that $S$ denotes the random subset of indices chosen by LIME in the sampling step (see Section~\ref{sec:sampling}). 
Define $E_j= \condexpec{(1-H_S)^{-1/2}}{S\not\ni j}$ and for any $k\neq j$, $E_{j,k} = \condexpec{(1-H_S)^{-1/2}}{S\not\ni j,k}$. 
Then we have the following:

\begin{proposition}[Computation of $\beta^f$, linear case]
\label{prop:beta-computation-linear-main}
Let $1\leq j\leq d$ and assume that $f(\normtfidf{x})=\normtfidf{x}_j$. 
Then, for any $1\leq k\leq d$ such that $k\neq j$, 
\begin{align*}
\left(\betainf^f\right)_k &= \biggl[2 E_{j,1} - \frac{2}{d}\sum_{\ell \neq k,j}E_{j,\ell}\biggr] \normtfidf{\xi}_j + \bigo{\frac{1}{d}}
\, ,
\end{align*}
and
\begin{align*}
\left(\betainf^f\right)_j &= \biggl[3E_j - \frac{2}{d} \sum_{k \neq j}E_{j,k}\biggr] \normtfidf{\xi}_j + \bigo{\frac{1}{d}}
\, .
\end{align*}
\end{proposition}

Proposition~\ref{prop:beta-computation-linear-main} is proved in the supplementary material. 
The main difficulty is to compute the expected value of $\normtfidf{x}_j$: this is the reason for the $E_j$ terms, for which we find an approximate expression as a function of the $\omega_k$s. 
Assuming that the $\omega_k$ are small, we can further this approximation and show that $E_j \approx 1.22$ and $E_{j,k}\approx 1.15$. 
In particular, \textbf{these expressions do not depend on~$j$ and~$k$}. 
Thus we can drastically simplify the statement of Proposition~\ref{prop:beta-computation-linear-main}: for any $k\neq j$, $\left(\beta_\infty^f\right)_k \approx 0$ and $\left(\beta_\infty^f\right)_j \approx 1.36 \normtfidf{\xi}_j$. 
We can now go back to our original goal, Eq.~\eqref{eq:def-linear-model-main}. 
By linearity, we deduce that 
\begin{equation}
\label{eq:simplified-betainf-linear-main}
\forall 1\leq j\leq d, \quad \left(\beta_\infty^f\right)_j \approx 1.36 \cdot \lambda_j \cdot  \normtfidf{\xi}_j
\, .
\end{equation}
In other words, up to a numerical constant and small error terms depending on $d$, \textbf{the explanation for a linear~$f$ is the TF-IDF value of the word multiplied by the coefficient of the linear model. }
We believe that this behavior is desirable for an interpretability method: large coefficients in the linear model should intuitively be associated to large interpretable coefficients. 
But at the same time the TF-IDF of the term is taken into account. 

We observe a very good match between theory and practice (see Figure~\ref{fig:linear}). 
Surprisingly, this is the case even though we assume that~$\nu$ is large in our derivations, whereas~$\nu$ is chosen by default in all our experiments.
We present experiments with other bandwidths in the supplementary.  

\begin{figure}
    \centering
\includegraphics[scale=0.24]{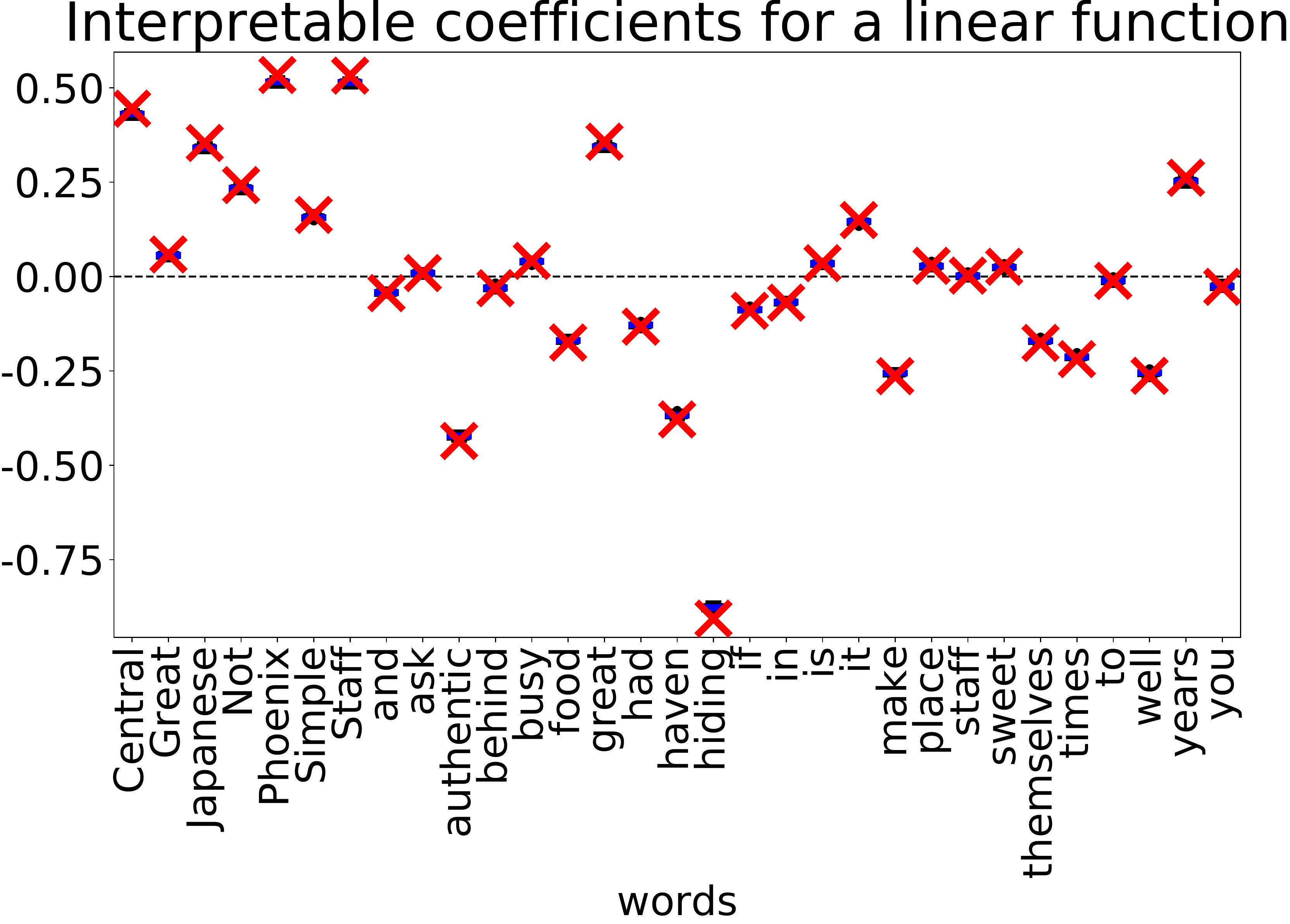}
   % \vspace{-0.2in}
    \caption{\label{fig:linear}Theory \emph{vs} practice for an arbitrary linear model. The black whisker boxes correspond to $100$ runs of LIME with default settings ($n=5000$ and $\nu=0.25$). The red crosses correspond to our theoretical predictions: $\beta_j\approx 1.36\lambda_j\normtfidf{\xi}_j$. Here $d=29$. }
\end{figure}

\section{Conclusion}

In this work we proposed the first theoretical analysis of LIME for text data. 
In particular, we provided a closed-form expression for the interpretable coefficients when the number of perturbed samples is large. 
Leveraging this expression, we exhibited some desirable behavior of LIME such as the linearity with respect to the model. 
In specific cases (simple decision trees and linear models), we derived more precise expression, showing that LIME outputs meaningful explanations in these cases. 

As future work, we want to tackle more complex models. 
More precisely, we think that it is possible to obtained approximate statements in the spirit of Eq.~\eqref{eq:simplified-betainf-linear-main} for models that are not linear. 

\subsubsection*{Acknowledgments}

This work was partly funded by the UCA DEP grant. 
The authors want to thank Andr\'e Galligo for getting them to know eachother. 

%\newpage
\bibliography{biblio}
\bibliographystyle{abbrvnat}

\newpage
% !TeX spellcheck = en_US

\onecolumn
\aistatstitle{Supplementary material for the paper: \\ ``An Analysis of LIME for Text Data''}

\thispagestyle{empty}

\setcounter{section}{0}
\section*{Organization of the supplementary material}

In this supplementary material, we collect the proofs of all our theoretical results and additional experiments. 
We study the covariance matrix in Section~\ref{sec:study-of-sigma} and the responses in Section~\ref{sec:study-of-gamma}. 
The proof of our main results can be found in Section~\ref{sec:study-of-beta}. 
Combinatorial results needed for the approximation formulas obtained in the linear case are collected in Section~\ref{sec:subsets-sums}, while other technical results can be found in Section~\ref{sec:technical}. 
Finally, we present some additional experiments in Section~\ref{sec:experiments}.  

\paragraph{Notation.}
First, let us quickly recall our notation. 
We consider $x,z,\pi$ the generic random variables associated to the sampling of new examples by LIME. 
To put it plainly, the new examples $x_1,\ldots,x_n$ are i.i.d. samples from the random variable $x$. 
Also remember that we denote by $S\subseteq \{1,\ldots,d\}$ the random subset of indices removed by LIME when creating new samples for a text with $d$ distinct words. 
For any finite set $R$, we write $\card{R}$ the cardinality of $R$. 
Recall that we denote by $S$ the random set of indices deleted in the sampling. 
We write $\Expec_s$ the expectation conditionally to $\card{S}=s$. 
Since we consider vectors belonging to $\Reals^{d+1}$ with the zero-th coordinate corresponding to an intercept, we will often start the numbering at $0$ instead of $1$. 
For any matrix $M$, we set $\frobnorm{M}$ the Frobenius norm of $M$ and $\opnorm{M}$ the operator norm of $M$. 

%%%%%%%%%%%%%%%%%%%%%%%%%%%%%%%%%%%%%%%%%%%%%%%%%%%%%%%%%%%%

\section{The study of $\Sigma$}
\label{sec:study-of-sigma}

We begin by the study of the covariance matrix. 
We show in Section~\ref{sec:computation-of-sigma} how to compute $\Sigma$. 
We will see how the $\alpha$ coefficients defined in the main paper appear. 
In Section~\ref{sec:computation-of-sigma-inverse}, we show that it is possible to invert $\Sigma$ in closed-form: it can be written in function of $\dencst_d$ and the $\sigma$ coefficients. 
We show how $\Sigmahat_n$ concentrates around $\Sigma$ in Section~\ref{sec:sigmahat-concentration}. 
Finally, Section~\ref{sec:control-opnorm} is dedicated to the control of $\opnorm{\Sigma^{-1}}$. 

%%%%%%%%%%%%%%%%%%%%%%%%%%%%%%%%%%%%%%%%%%%%%%%%%%%%%%%%%%%%%%%%%%%

\subsection{Computation of $\Sigma$}
\label{sec:computation-of-sigma}

In this section, we derived a closed-form expression for $\Sigma\defeq \smallexpec{\Sigmahat_n}$ as a function of $d$ and $\nu$.  
Recall that we defined $\Sigmahat = \frac{1}{n}Z^\top WZ$. 
By definition of $Z$ and $W$, we have
\[
\Sigmahat =
\begin{pmatrix}
\frac{1}{n}\sum_{i=1}^n \pi_i & \frac{1}{n}\sum_{i=1}^n \pi_i z_{i,1} & \cdots & \frac{1}{n}\sum_{i=1}^n \pi_i z_{i,d} \\ 
\frac{1}{n}\sum_{i=1}^n \pi_i z_{i,1}  & \frac{1}{n}\sum_{i=1}^n \pi_i z_{i,1}  & \cdots &  \frac{1}{n}\sum_{i=1}^n \pi_i z_{i,1}z_{i,d} \\ 
\vdots & \vdots  & \ddots & \vdots \\ 
\frac{1}{n}\sum_{i=1}^n \pi_i z_{i,d} & \frac{1}{n}\sum_{i=1}^n \pi_i z_{i,1}z_{i,d} & \cdots & \frac{1}{n}\sum_{i=1}^n \pi_i z_{i,d}
\end{pmatrix}
\in\Reals^{(d+1)\times (d+1)}
\, .
\]
Taking the expectation in the last display with respect to the sampling of new examples yields
\begin{equation}
\label{eq:def-sigma}
\Sigma =\begin{pmatrix}
\expec{\pi} & \expec{\pi z_1} & \cdots & \expec{\pi z_d} \\ 
 \expec{\pi z_1}  & \expec{\pi  z_1 }  & \cdots &  \expec{\pi z_1 z_d}  \\ 
\vdots & \vdots  & \ddots & \vdots \\ 
\expec{\pi z_d} & \expec{\pi z_1z_d } & \cdots & \expec{\pi z_d}
\end{pmatrix}
\in\Reals^{(d+1)\times (d+1)}
\, .
\end{equation}

% definition of the \alpha_i -> back ref to the main paper
An important remark is that $\expec{\pi z_j}$ does not depend on $j$. 
Indeed, there is no privileged index in the sampling of $S$ (the subset of removed indices). 
Thus we only have to look into  $\expec{\pi z_1}$ (say). 
For the same reason, $\expec{\pi z_jz_k}$ does not depend on the $2$-uple $(j,k)$, and we can limit our investigations to $\expec{\pi z_1z_2}$. 
This is the reason why we defined $\alpha_0 = \expec{\pi}$ and, for any $1\leq p\leq d$, 
\begin{equation}
\label{eq:def-alphas}
\alpha_p = \expec{\pi \cdot z_1 \cdots z_p}
\end{equation}
in the main paper. 
We recognize the definition of the $\alpha_p$s in Eq.~\eqref{eq:def-sigma} and we write
\[
\Sigma_{j,k} = 
\begin{cases}
\alpha_0 &\text{ if } j=k=0, \\
\alpha_1 &\text{ if } j=0 \text{ and } k> 0 \text{ or } j> 0 \text{ and } k=0 \text{ or } j=k> 0, \\
\alpha_2 &\text{ otherwise. }
\end{cases}
\]
%Notice that we will often omit the dependency in $d$ and $\nu$ when this is clear from context. 
As promised, we can be more explicit regarding the $\alpha$ coefficients. 
Recall that we defined the mapping 
\begin{align}
\psi\colon [0,1]& \longrightarrow \Reals \label{eq:def-psi} \\
t &\longmapsto \exp{-(1-\sqrt{1-t})^2/(2\nu^2)} \notag 
\, .
\end{align}
It is a decreasing mapping (see Figure~\ref{fig:psi-t}). 
With this notation in hand, we have the following expression for the $\alpha$ coefficients (this is Proposition~1 in the paper):

\begin{figure}
    \centering
    \includegraphics[scale=0.16]{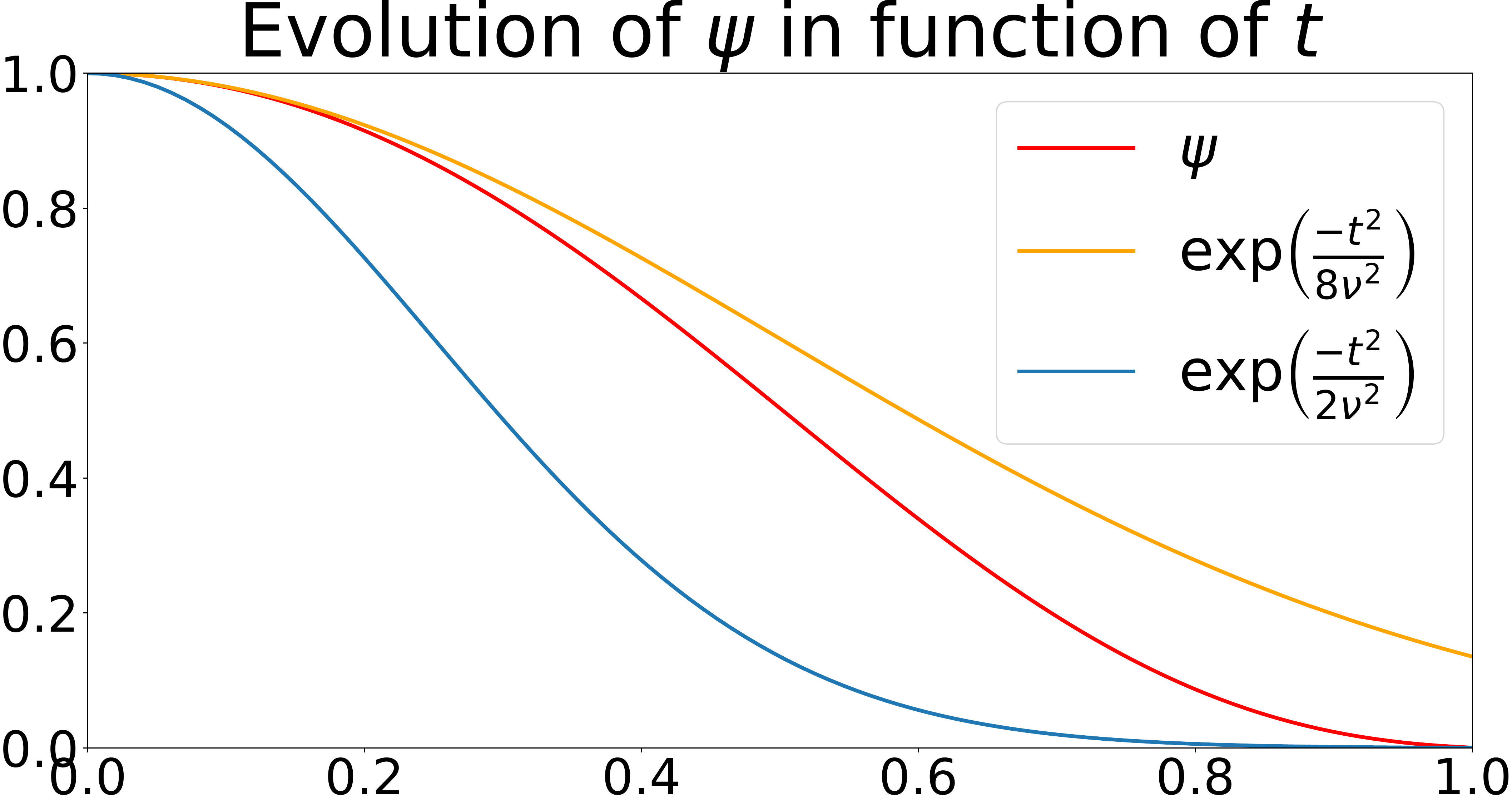}
    \caption{\label{fig:psi-t}The function $\psi$ defined by Eq.~\eqref{eq:def-psi} with bandwidth parameter $\nu=0.25$. In orange (resp. blue), one can see the upper (resp. lower) bound given by Eq.~\eqref{eq:psi-precise-bound}. }
\end{figure}

\begin{proposition}[Computation of the $\alpha$ coefficients]
\label{prop:alphas-computation}
For any $d\geq 1$, $\nu >0$, and $p\geq 0$, it holds that
\[
\alpha_p = \frac{1}{d} \sum_{s=1}^d \prod_{k=0}^{p-1} \frac{d-s-k}{d-k} \psi\left(\frac{s}{d}\right)
\, .
\]
\end{proposition}

In particular, the first three $\alpha$ coefficients can be written
\[
\alpha_0 = \frac{1}{d} \sum_{s=1}^d \psi\left(\frac{s}{d}\right) \, ,
\quad 
\alpha_1 = \frac{1}{d} \sum_{s=1}^d \left(1-\frac{s}{d}\right)\psi\left(\frac{s}{d}\right)
\, ,
\quad \text{ and } \quad
\alpha_2 = \frac{1}{d} \sum_{s=1}^d \left(1-\frac{s}{d}\right)\left(1-\frac{s}{d-1}\right)\psi\left(\frac{s}{d}\right)
\, .
\]

\begin{proof}
The idea of the proof is to use the law of total expectation with respect to the collection of events $\{\card{S}=s\}$ for $s\in\{1,\ldots,d\}$. 
Since $\proba{\card{S}=s}=\frac{1}{d}$ for any $1\leq s\leq d$, all that is left to compute is the expectation of $\pi z_1\cdots z_p$ conditionally to $\card{S}=s$. 
According to the remark in Section~2.3 of the main paper, $\pi = \psi(s/d)$ conditionally to $\{\card{S}=s\}$.
We can conclude since, according to Lemma~\ref{lemma:proba-containing-cond}, 
\[
\probaunder{\word_1\in x,\ldots,\word_p\in x}{s} = \frac{(d-s)(d-s-1)\cdots (d-s-p+1)}{d(d-1)\cdots (d-p+1)}
\, .
\]
\end{proof}

\begin{figure}
\centering
\includegraphics[scale=0.11]{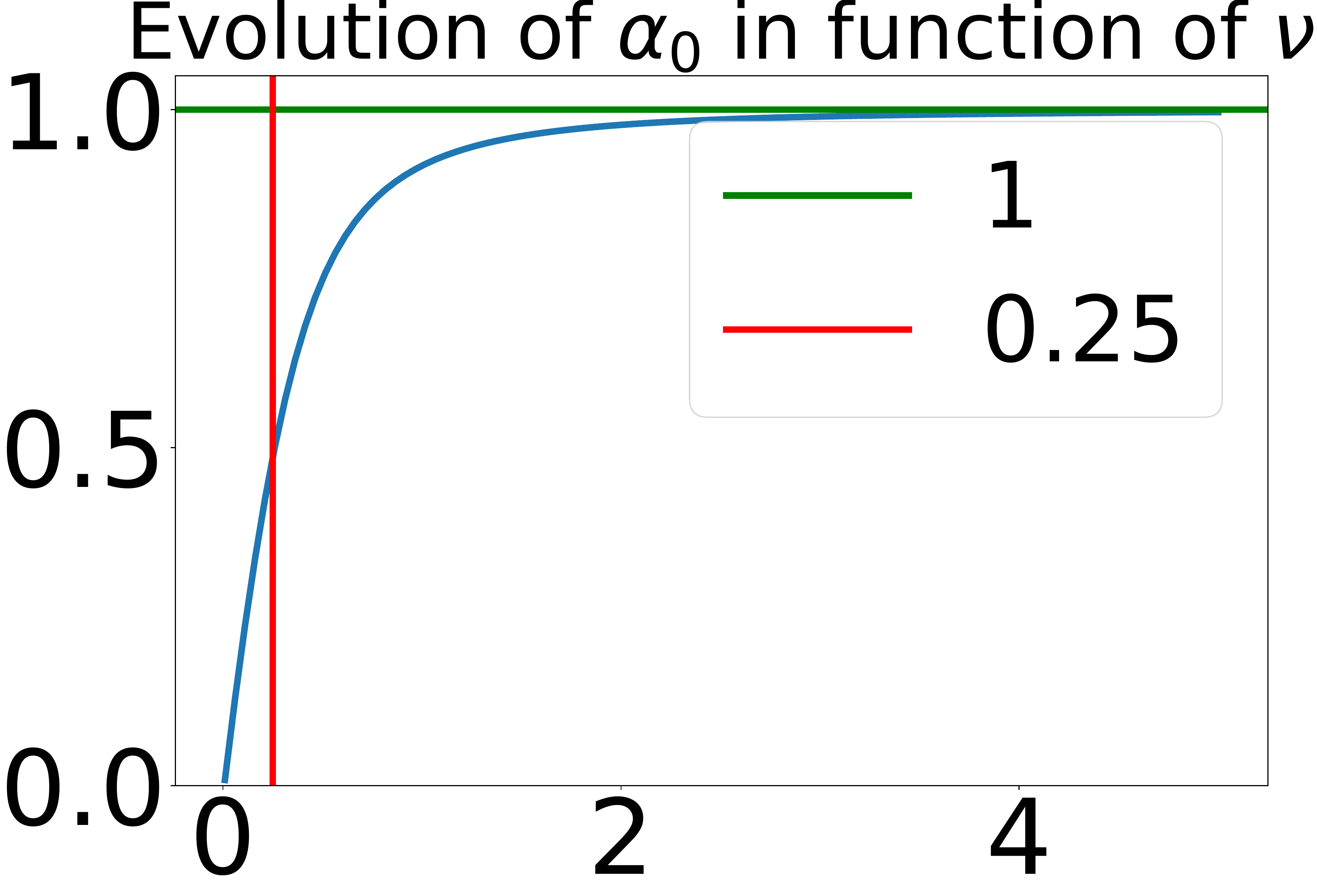} 
\includegraphics[scale=0.11]{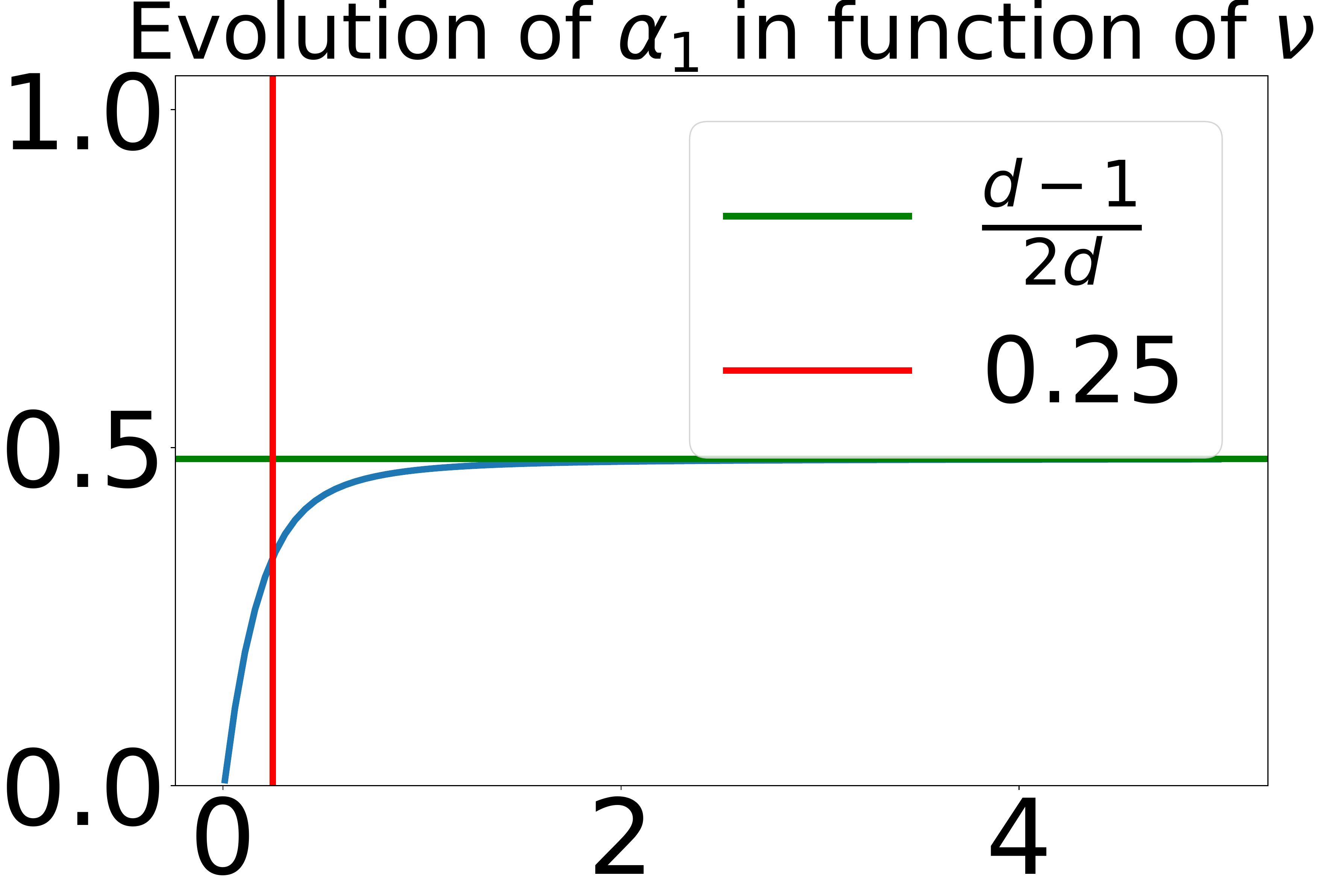}
\includegraphics[scale=0.11]{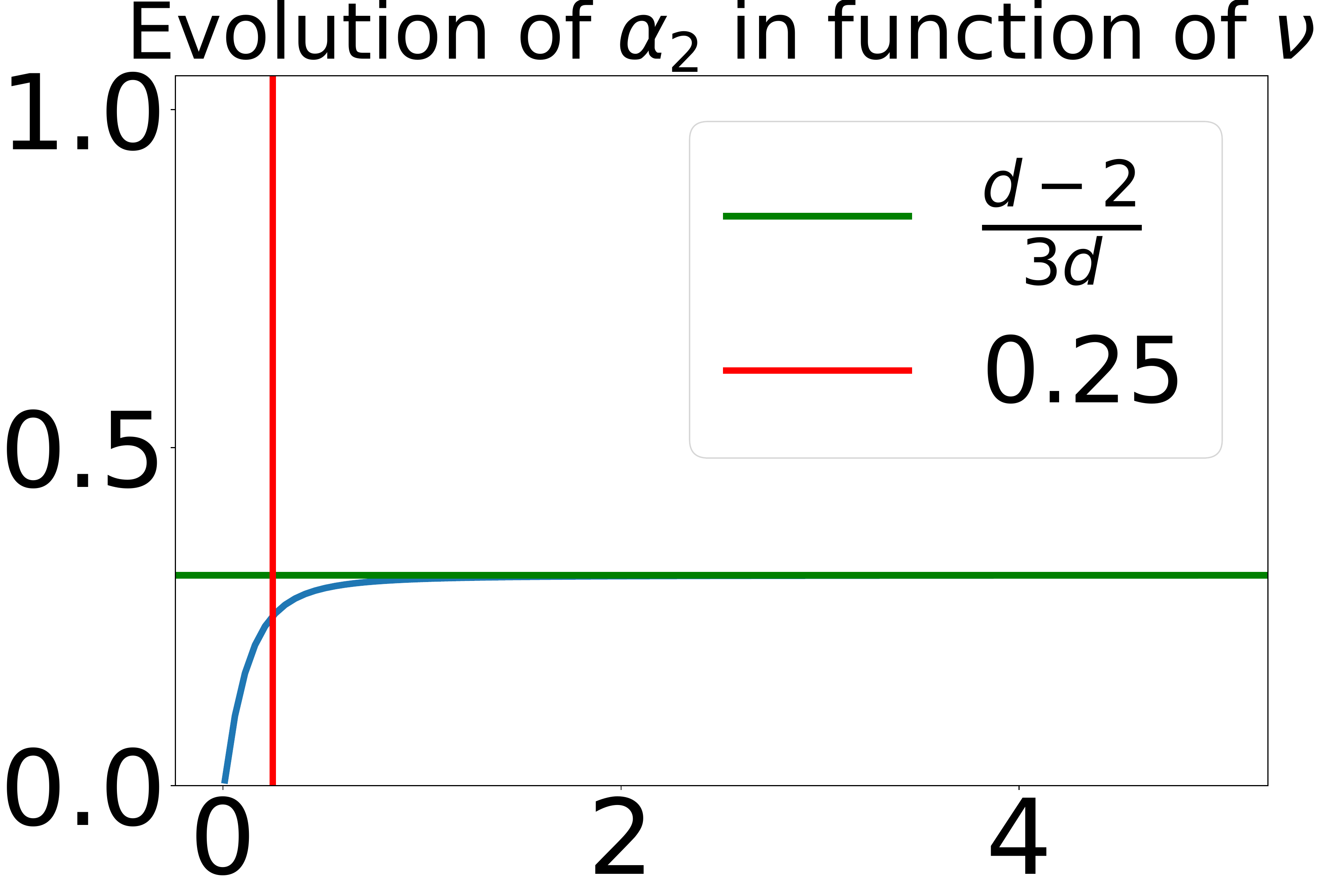}
\includegraphics[scale=0.11]{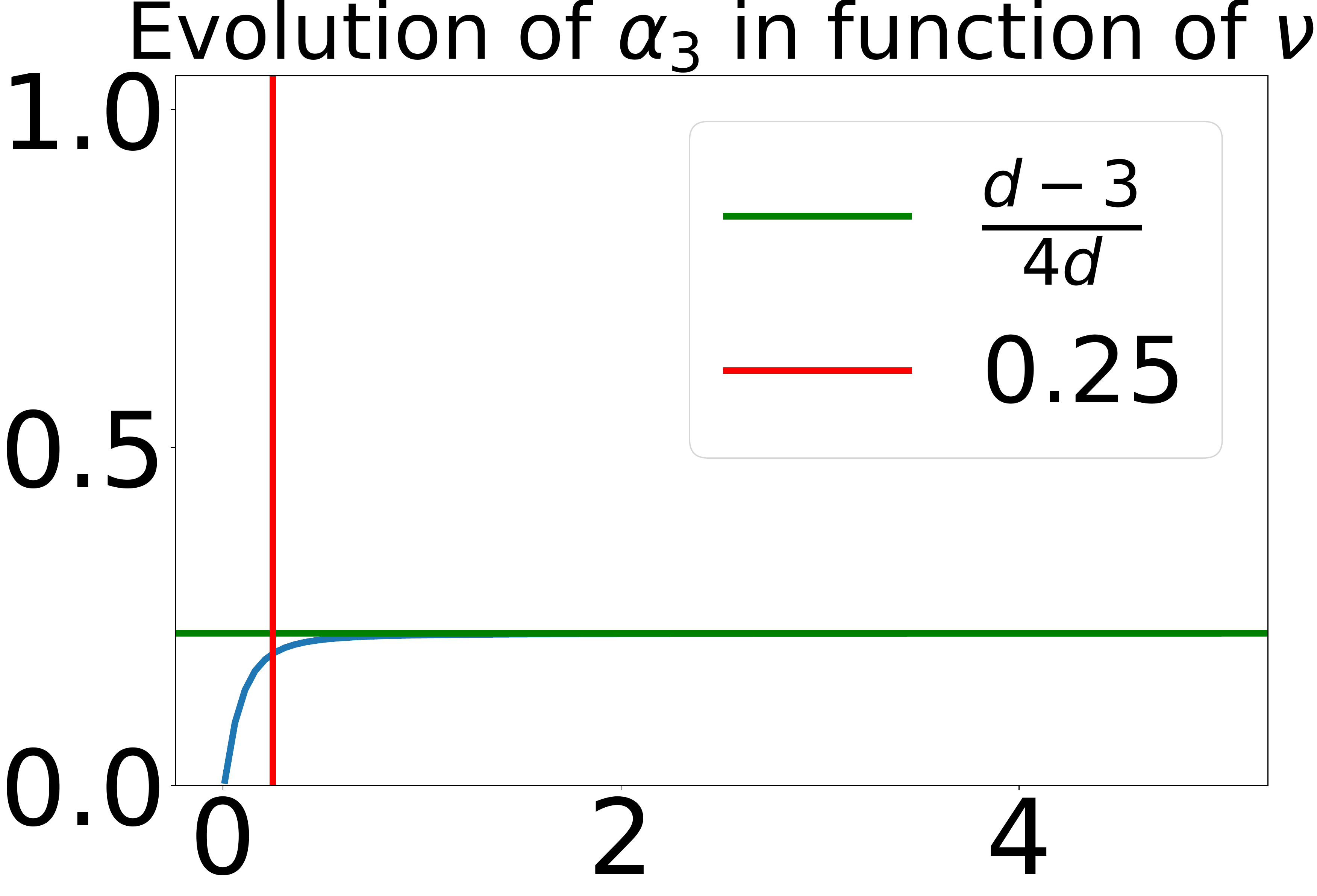}
\caption{\label{fig:alphas-bandwidth-dependency}Behavior of the first $\alpha$ coefficients with respect to the bandwidth parameter $\nu$. The red vertical lines mark the default bandwidth choice ($\nu=0.25$). The green horizontal line denotes the limits for large $d$ given by Corollary~\ref{cor:alphas-approx}.}
\end{figure}

It is important to notice that, when $\nu \to +\infty$, $\psi(t)\to 0$ for any $t\in (0,1]$. 
As a consequence, in the large bandwidth regime, the $\psi(s/d)$ weights are arbitrarily close to one. 
We demonstrate this effect in Figure~\ref{fig:alphas-bandwidth-dependency}. 
In this situation, the $\alpha$ coefficients take a simpler form. 

\begin{corollary}[Large bandwidth approximation of $\alpha$ coefficients]
\label{cor:alphas-approx}
For any $0\leq p\leq d$, it holds that
\[
\lim_{\nu\to +\infty} \alpha_p = \frac{d-p}{(p+1)d}
\, .
\]
\end{corollary}

We report these approximate values in Figure~\ref{fig:alphas-bandwidth-dependency}. 
In particular, when both $\nu$ and $d$ are large, we can see that $\alpha_p\approx 1/(p+1)$. 
Thus $\alpha_0\approx 1$, $\alpha_1\approx \frac{1}{2}$, and $\alpha_2\approx \frac{1}{3}$. 

\begin{proof}
When $\nu\to +\infty$, we have $\psi(s/d)\to 1$ and we can conclude directly by using Lemma~\ref{lemma:proba-containing}. 
\end{proof}

Notice that we can be slightly more precise than Corollary~\ref{cor:alphas-approx}. 
Indeed, $\psi$ is decreasing on $[0,1]$, thus for any $t\in [0,1]$, $\exp{-1/(2\nu^2)}\leq \psi(t)\leq 1$. 
Therefore we can present some efficient bounds for the $\alpha$ coefficients when $\nu$ is large. 

\begin{corollary}[Bounds on the $\alpha$ coefficients]
\label{cor:alphas-bounds}
For any $0\leq p\leq d$, it holds that 
\[
\frac{d-p}{(p+1)d} \exps{\frac{-1}{2\nu^2}} \leq \alpha_p \leq \frac{d-p}{(p+1)d}
\, .
\]
\end{corollary}

One can further show that, for any $0\leq t\leq 1$, 
\begin{equation}
\label{eq:psi-precise-bound}
\exp{\frac{-t^2}{2\nu^2}} \leq \psi(t) \leq \exp{\frac{-t^2}{8\nu^2}}
\, .
\end{equation}
Using Eq.~\eqref{eq:psi-precise-bound} together with the series-integral comparison theorem would yield very accurate bounds for the $\alpha$ coefficients and related quantities, but we will not follow that road.

%%%%%%%%%%%%%%%%%%%%%%%%%%%%%%%%%%%%%%%%%%%%%%%%%%%%%%%%%%%%%%%%%%%%%%%%%%%%%%%%%%%

\subsection{Computation of $\Sigma^{-1}$}
\label{sec:computation-of-sigma-inverse}

In this section, we present a closed-form formula for the matrix inverse of $\Sigma$ as a function of $d$ and $\nu$.

\begin{minipage}{0.7\textwidth}
\begin{proposition}[Computation of $\Sigma^{-1}$]
\label{prop:sigma-inverse-computation}
For any $d\geq 1$ and $\nu >0$, recall that we defined
\[
\dencst_d = (d-1)\alpha_0\alpha_2 -d\alpha_1^2 + \alpha_0\alpha_1
\, .
\]
Assume that $\dencst_d\neq 0$ and $\alpha_1\neq \alpha_2$. 
Define $\sigma_0 \defeq  (d-1)\alpha_2 + \alpha_1$ and recall that we set 
\[
\begin{cases}
\sigma_1 &= -\alpha_1
\, , \\
\sigma_2 &= \frac{(d-2)\alpha_0 \alpha_2 - (d-1)\alpha_1^2 + \alpha_0\alpha_1}{\alpha_1-\alpha_2}\, , \\
\sigma_3 &= \frac{\alpha_1^2-\alpha_0\alpha_2}{\alpha_1-\alpha_2 }
\, .
\end{cases}
\]
Then it holds that
\begin{equation}
\label{eq:sigma-inverse-computation}
 \Sigma^{-1} = 
 \frac{1}{\dencst_d}
\begin{pmatrix}
   \sigma_0 &  \sigma_1 & \sigma_1  &\cdots  & \sigma_1  \\ 
   
   \sigma_1 & \sigma_2 & \sigma_3 & \cdots  &  \sigma_3 \\
   
   \sigma_1 &  \sigma_3 & \sigma_2   &  \ddots  &  \vdots  \\ 
   
   \vdots  &  \vdots &  \ddots &  \ddots  & \sigma_3 \\
   
   \sigma_1 & \sigma_3 & \cdots & \sigma_3 &  \sigma_2  \\
\end{pmatrix}
\in\Reals^{(d+1)\times (d+1)}
\, .
\end{equation}
\end{proposition}
\end{minipage}
\begin{minipage}{0.25\textwidth}
    \begin{center}
    \includegraphics[scale=0.11]{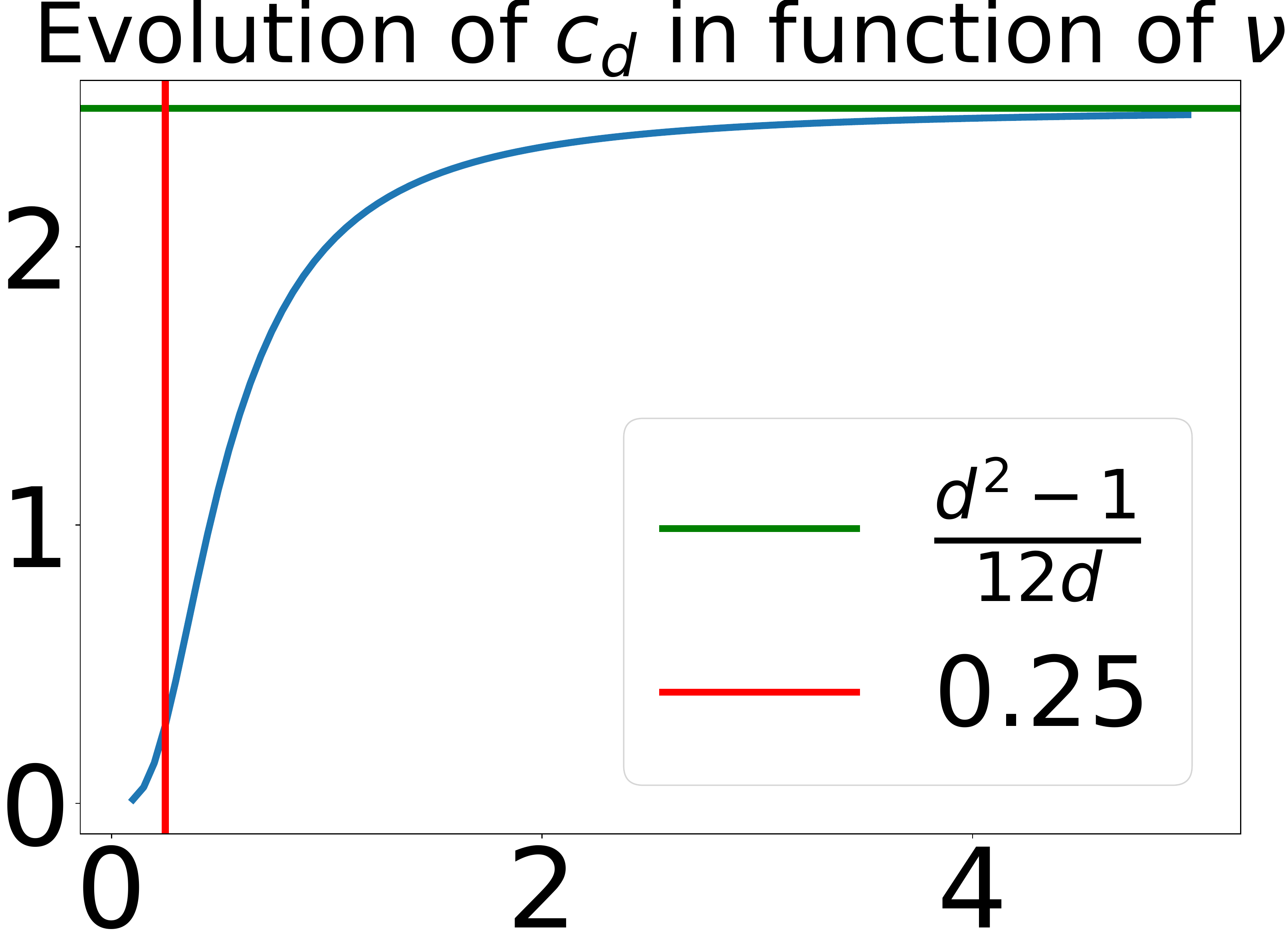}
    \end{center}
    \captionof{figure}{\label{fig:dencst}Evolution of the normalization constant $\dencst_d$ as a function of the bandwidth for $d=30$. In red, the default bandwidth $\nu=0.25$, in green the limit for large bandwidth given by Corollary~\ref{cor:approximate-sigma-inverse}.}
\end{minipage}

We display the evolution of the  $\sigma_i/\dencst_d$  coefficients with respect to $\nu$ in Figure~\ref{fig:sigmas}. 

\begin{figure}
    \centering
    \includegraphics[scale=0.11]{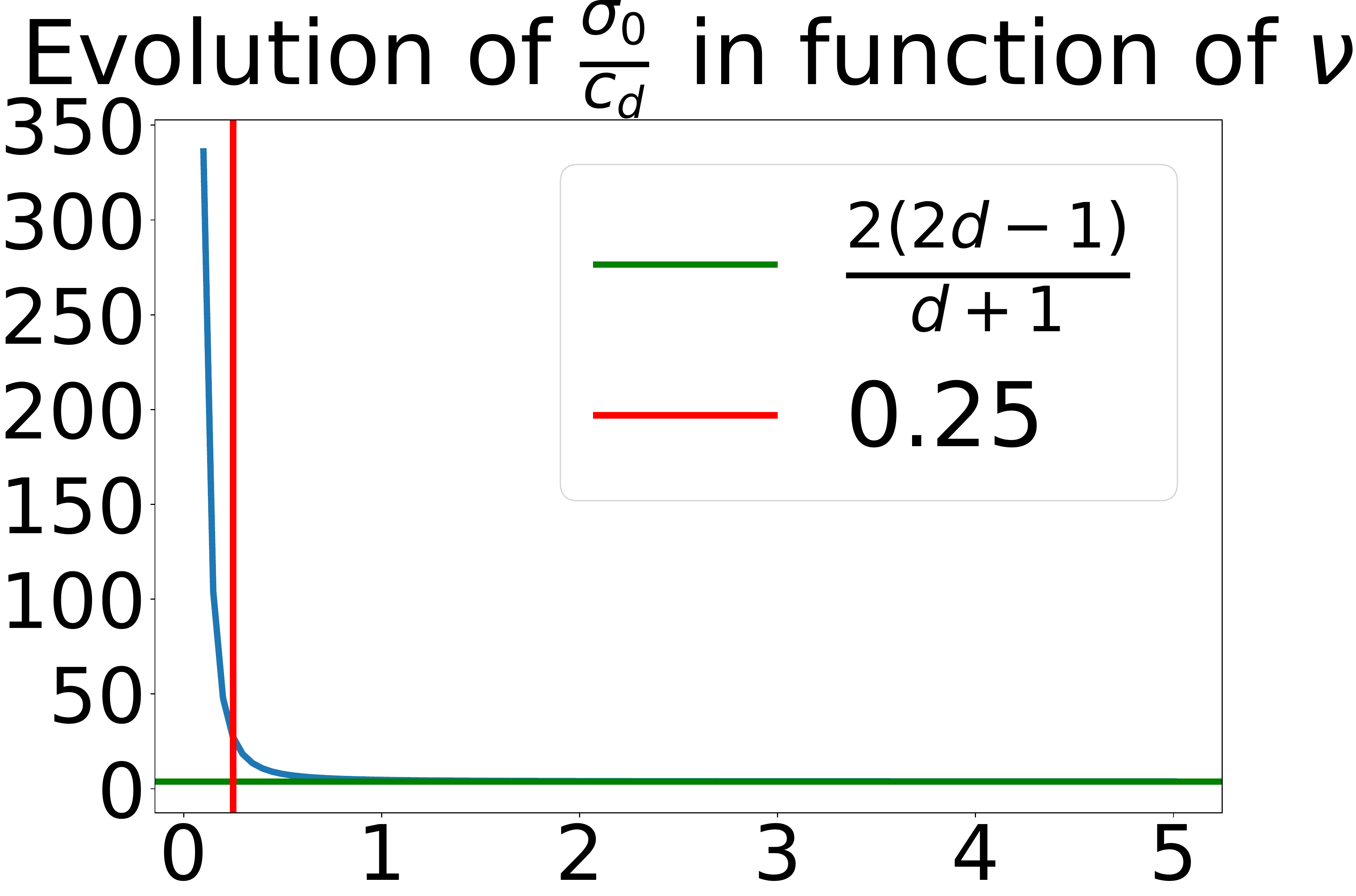} 
    \includegraphics[scale=0.11]{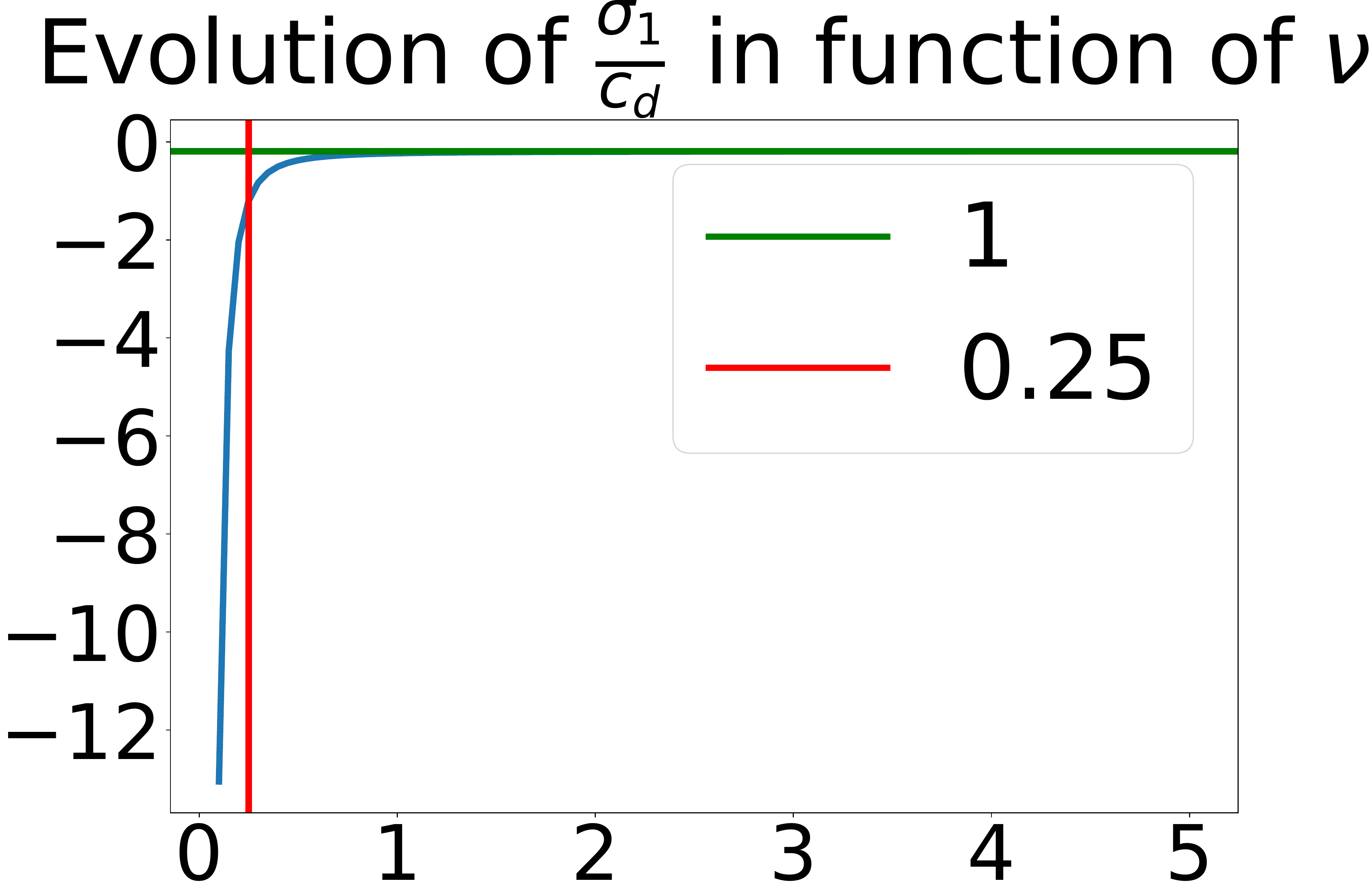}
    \includegraphics[scale=0.11]{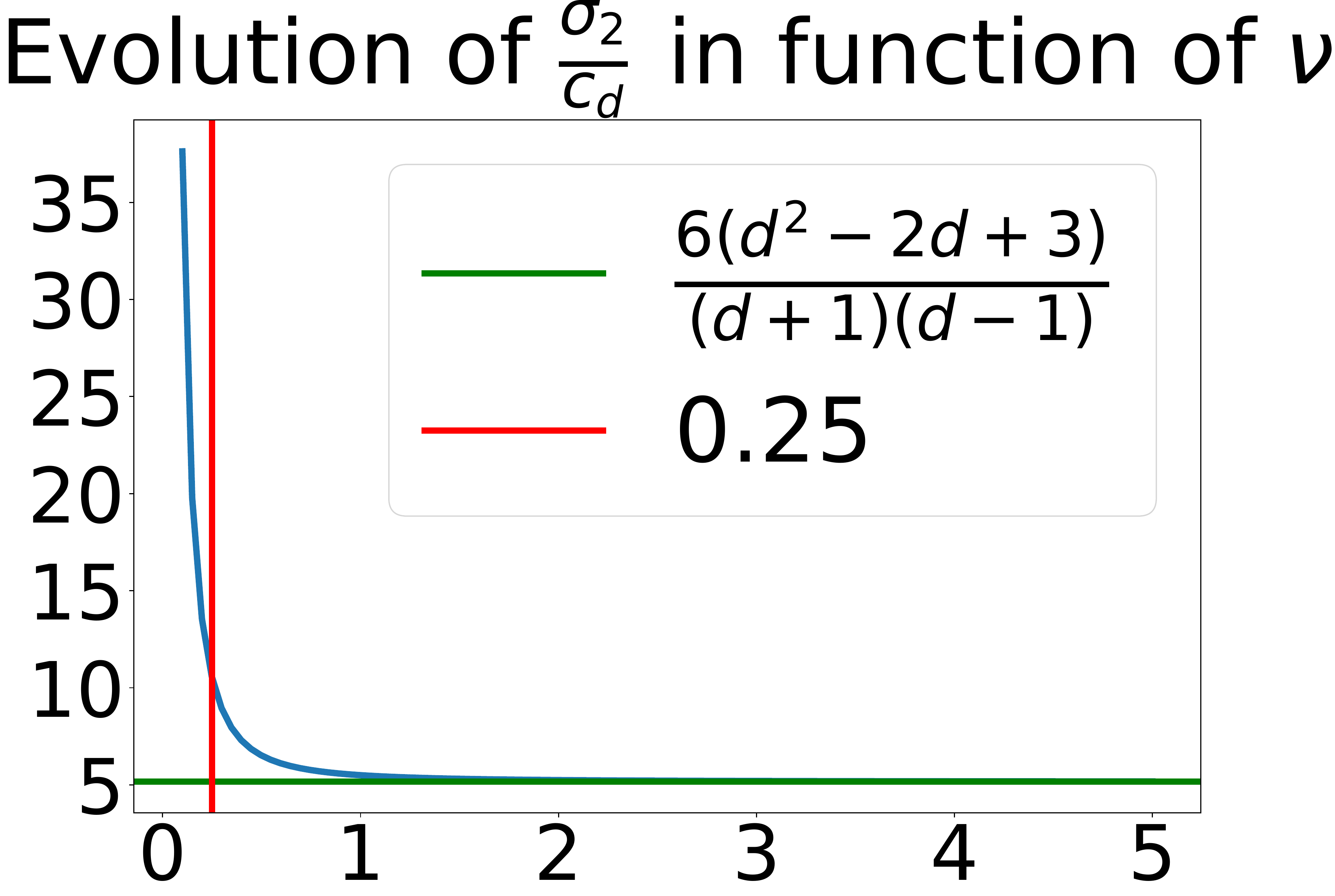}
    \includegraphics[scale=0.11]{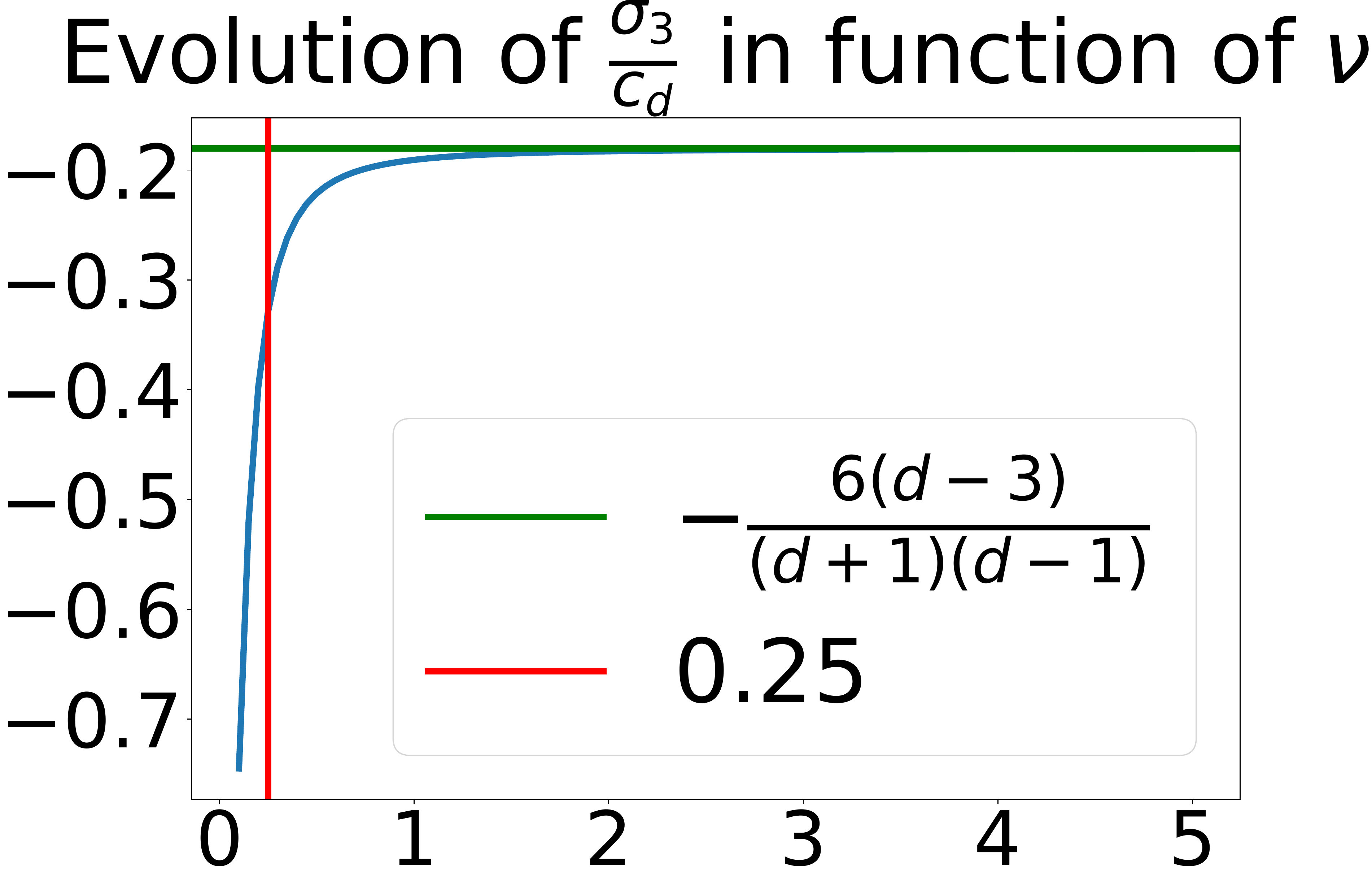}
    \caption{\label{fig:sigmas}Evolution of $\sigma_i/\dencst_d$ as a function of $\nu$ for $1\leq i\leq 4$ for $d=30$. In red the default value of the bandwidth. In green the limits given by Corollary~\ref{cor:approximate-sigma-inverse}. We can see that the $\sigma$ coefficients are close to these limit values for the default bandwidth. }
\end{figure}

\begin{proof}
From Eq.~\eqref{eq:def-sigma}, we can see that $\Sigma$ is a block matrix. 
The result follows from the block matrix inversion formula and one can check directly that $\Sigma\cdot \Sigma^{-1} = \Identity_{d+1}$. 
\end{proof}

Our next result shows that the assumptions of Proposition~\ref{prop:sigma-inverse-computation} are satisfied: $\alpha_1-\alpha_2$ and $\dencst_d$ are positive quantities. 
In fact, we prove a slightly stronger statement which will be necessary to control the operator norm of $\Sigma^{-1}$. 

\begin{proposition}[$\Sigma$ is invertible]
\label{prop:large-nu-makes-everything-ok}
For any $d\geq 2$, 
\[
\alpha_1 - \alpha_2 \geq \frac{\exps{\frac{-1}{2\nu^2}}}{6} > 0
\, ,
\quad \text{ and } \quad 
%\dencst_d \geq \frac{\exp{\frac{-3}{\nu^2}}}{5184 d} > 0
\dencst_d \geq \frac{\exps{\frac{-2}{\nu^2}}}{40} > 0
\, .
\]
\end{proposition}

\begin{proof}
By definition of the $\alpha$ coefficients (Eq.~\eqref{eq:def-alphas}), we have
\[
\alpha_1 - \alpha_2 = \frac{1}{d}\sum_{s=1}^d \left(1-\frac{s}{d}\right) \frac{s}{d-1} \psi\left(\frac{s}{d}\right)
\, .
\]
Since $\exps{\frac{-1}{2\nu^2}} \leq \psi(t) \leq 1$ for any $t\in [0,1]$, we have
\begin{equation}
\label{eq:large-nu-aux-1}
\exps{\frac{-1}{2\nu^2}} \cdot \frac{1}{d}\sum_{s=1}^d \left(1-\frac{s}{d}\right) \frac{s}{d-1} = \frac{d+1}{6d}\cdot \exps{\frac{-1}{2\nu^2}}
\leq \alpha_1-\alpha_2 \leq \frac{d+1}{6d}
\, .
\end{equation}
The right-hand side of Eq.~\eqref{eq:large-nu-aux-1} yields the promised bound. 
Note that the same reasoning gives
\begin{equation}
\label{eq:large-nu-aux-2}
\frac{d+1}{2d} \cdot \exps{\frac{-1}{2\nu^2}} \leq \alpha_0 - \alpha_1 \leq \frac{d+1}{2d}
\, .
\end{equation}

Let us now find a lower bound for $\dencst_d$. 
We first start by noticing that 
\begin{align}
\dencst_d &= d\alpha_1(\alpha_0-\alpha_1) - (d-1)\alpha_0(\alpha_1-\alpha_2) \label{eq:dencst-alt-writing} \\
&= \sum_{s=1}^d \left(1-\frac{s}{d}\right) \psi\left(\frac{s}{d}\right) \cdot \frac{1}{d}\sum_{s=1}^d \frac{s}{d}\psi\left(\frac{s}{d}\right) - \sum_{s=1}^d \psi\left(\frac{s}{d}\right) \cdot \frac{1}{d} \sum_{s=1}^d \left(1-\frac{s}{d}\right) \psi\left(\frac{s}{d}\right) \notag \\
\dencst_d &= \frac{1}{d}\left[ \sum_{s=1}^d \psi\left(\frac{s}{d}\right) \cdot \sum_{s=1}^d \frac{s^2}{d^2}\psi\left(\frac{s}{d}\right) - \left(\sum_{s=1}^d \frac{s}{d}\psi\left(\frac{s}{d}\right)\right)^2\right] \notag 
\, .
\end{align}
Therefore, by Cauchy-Schwarz inequality, $\dencst_d\geq 0$. 
In fact, $\dencst_d>0$ since the equality case in Cauchy-Schwarz is attained for proportional summands, which is not the case here. 

However, we need to improve this result if we want to control $\opnorm{\Sigma^{-1}}$ more precisely.  
To this extent, we use a refinement of Cauchy-Schwarz inequality obtained by \citet{filipovski_2019}. 
Let us set, for any $1\leq s\leq d$, 
\[
a_s \defeq \sqrt{\psi\left(\frac{s}{d}\right)}\, ,
\quad b_s \defeq \frac{s}{d}\sqrt{\psi\left(\frac{s}{d}\right)}\, ,
\quad A \defeq \sqrt{\sum_{s=1}^d a_s^2} \, ,
\quad \text{ and } B\defeq \sqrt{\sum_{s=1}^d b_s^2}
\, .
\]
With these notation, 
\[
\dencst_d = \frac{1}{d}\left[A^2B^2-\left(\sum_{s=1}^d a_sb_s\right)^2\right]
\, ,
\]
and Cauchy-Schwarz yields $A^2B^2\geq \left(\sum_{s=1}^d a_sb_s\right)^2$. 
Theorem~2.1 in \citet{filipovski_2019} is a stronger result, namely
\begin{equation}
\label{eq:improved-cauchy-schwarz}
AB\geq \sum_{s=1}^da_sb_s + \frac{1}{4}\sum_{s=1}^d \frac{(a_s^2 B^2 - b_s^2 A^2)^2}{a_s^4B^4 + b_s^4 A^4} a_sb_s
\, .
\end{equation}
Let us focus on this last term. 
Since all the terms are non-negative, we can lower bound by the term of order $d$, that is,
\begin{equation}
\label{eq:large-nu-aux-3}
\frac{1}{4}\sum_{s=1}^d \frac{(a_s^2 B^2 - b_s^2 A^2)^2}{a_s^4B^4 + b_s^4 A^4} a_sb_s \geq \frac{1}{4} \frac{(b_d^2A^2-a_d^2B^2)^2}{b_d^4A^4+a_d^4B^4}a_db_d = \frac{1}{4}\frac{(A^2-B^2)^2}{A^4+B^4}\psi(1)
\, ,
\end{equation}
since $a_d=b_d = \sqrt{\psi(1)}$. 
On one side, we notice that
\begin{align}
A^2-B^2 &= \sum_{s=1}^d \left(1-\frac{s^2}{d^2}\right)\psi\left(\frac{s}{d}\right) \notag \\
&\geq \exp{\frac{-1}{2\nu^2}} \cdot \sum_{s=1}^d \left(1-\frac{s^2}{d^2}\right) \notag \tag{for any $t\in[0,1]$, $\psi(t)\geq \exps{-1/(2\nu^2)}$} \\
&= \exp{\frac{-1}{2\nu^2}} \cdot \frac{1}{6}\left(4d-\frac{1}{d}-3\right) \notag \\
A^2-B^2 &\geq \frac{3d\cdot \exp{\frac{-1}{2\nu^2}}}{8} \notag 
\, ,
\end{align}
where we used $d\geq 2$ in the last display. 
We deduce that $(A^2-B^2)^2 \geq 9d^2\exps{\frac{-1}{2\nu^2}}/64$. 
On the other side, it is clear that $A^2\leq d$, and 
\[
B^2 \leq \sum_{s=1}^d \frac{s^2}{d^2} = \frac{(d+1)(2d+1)}{6d}
\, .
\]
For any $d\geq 2$, we have $B^2\leq 5d/8$, and we deduce that $A^4+B^4\leq \frac{89}{64}d^2$. 
Therefore,
\[
\frac{(A^2-B^2)^2}{A^4+B^4} \geq \frac{9\exps{\frac{-1}{\nu^2}}}{89}
\, .
\]
Coming back to Eq.~\eqref{eq:large-nu-aux-3}, we proved that 
\[
\frac{1}{4}\sum_{s=1}^d \frac{(a_s^2 B^2 - b_s^2 A^2)^2}{a_s^4B^4 + b_s^4 A^4} a_sb_s \geq \frac{9\exps{\frac{-3}{2\nu^2}}}{356}
\, .
\]
Plugging into Eq.~\eqref{eq:improved-cauchy-schwarz} and taking the square, we deduce that 
\[
A^2B^2\geq \left(\sum_{s=1}^d a_sb_s\right)^2 + 2\cdot \sum_{s=1}^d a_sb_s\cdot \frac{9\exps{\frac{-3}{2\nu^2}}}{356} + \frac{81\exps{\frac{-3}{\nu^2}}}{126736}
\, .
\]
But $\sum a_sb_s \geq d\exps{\frac{-1}{2\nu^2}}/2$, therefore, ignoring the last term, we have
\[
A^2B^2 -\left(\sum_{s=1}^d a_sb_s\right)^2 \geq \frac{9d\exps{\frac{-2}{\nu^2}}}{356}
\, .
\]
We conclude by noticing that $356/9\leq 40$. 
\end{proof}

\begin{remark}
We suspect that the correct lower bound for $\dencst_d$ is actually of order $d$, but we did not manage to prove it. 
Careful inspection of the proof shows that this $d$ factor is lost when considering only the last term of the summation in Eq.~\eqref{eq:improved-cauchy-schwarz}. 
It is however challenging to control the remaining terms, since $B^2$ is roughly half of $A^2$ and $\frac{s^2}{d^2}B^2-A^2$ is close to $0$ for some values of $s$. 
\end{remark}

We conclude this section by giving an approximation of $\Sigma^{-1}$ for large bandwidth. 
This approximation will be particularly useful in Section~\ref{sec:beta-computation}. 

\begin{corollary}[Large bandwidth approximation of $\Sigma^{-1}$]
\label{cor:approximate-sigma-inverse}
For any $d\geq 2$, when $\nu \to +\infty$, we have 
\[
\dencst_d \longrightarrow \frac{d^2-1}{12d} 
\, ,
\]
and, as a consequence,
\begin{equation}
\label{eq:def-sigma-infty}
\begin{cases}
\frac{\sigma_0}{\dencst_d} &\to \frac{2(2d-1)}{d+1} = 4 - \frac{6}{d} + \bigo{\frac{1}{d^2}} \\
\frac{\sigma_1}{\dencst_d} &\to \frac{-6}{d+1} = - \frac{6}{d} + \bigo{\frac{1}{d^2}} \\
\frac{\sigma_2}{\dencst_d} &\to \frac{6(d^2-2d+3)}{(d+1)(d-1)} = 6 - \frac{12}{d} + \bigo{\frac{1}{d^2}} \\
\frac{\sigma_3}{\dencst_d} &\to \frac{-6(d-3)}{(d+1)(d-1)} = -\frac{6}{d} + \bigo{\frac{1}{d^2}}\, .
\end{cases}
\end{equation}
\end{corollary}

\begin{proof}
The proof is straightforward from the definition of $\dencst_d$ and the $\sigma$ coefficients, and Corollary~\ref{cor:alphas-approx}. 
\end{proof}

\subsection{Concentration of $\Sigmahat_n$}
\label{sec:sigmahat-concentration}

We now turn to the concentration of $\Sigmahat_n$ around $\Sigma$. 
More precisely, we show that $\Sigmahat_n$ is close to $\Sigma$ in operator norm, with high probability. 
Since the definition of $\Sigmahat_n$ is identical to the one in the Tabular LIME case, we can use the proof machinery of \citet{garreau_luxburg_2020_arxiv}. 

\begin{proposition}[Concentration of $\Sigmahat_n$]
\label{prop:sigmahat-concentration}
For any $t\geq 0$, 
\[
\proba{\opnorm{\Sigmahat_n - \Sigma} \geq t} \leq 4d\cdot \exp{\frac{-nt^2}{32d^2}}
\, .
\]
\end{proposition}

\begin{proof}
We can write $\Sigmahat=\frac{1}{n}\sum_i \pi_i Z_iZ_i^\top$. 
The summands are bounded i.i.d. random variables, thus we can apply the matrix version of Hoeffding inequality. 
More precisely, the entries of $\Sigmahat_n$ belong to $[0,1]$ by construction, and Corollary~\ref{cor:alphas-bounds} guarantees that the entries of $\Sigma$ also belong to $[0,1]$. 
Therefore, if we set $M_i\defeq \frac{1}{n}\pi_i Z_iZ_i^\top -\Sigma$, then the $M_i$ satisfy the assumptions of Theorem~21 in \citet{garreau_luxburg_2020_arxiv} and we can conclude since $\frac{1}{n}\sum_i M_i = \Sigmahat_n-\Sigma$. 
\end{proof}

\subsection{Control of $\opnorm{\Sigma^{-1}}$}
\label{sec:control-opnorm}

% control of the \sigma_i
We now turn to the control of $\opnorm{\Sigma^{-1}}$. 
Essentially, our strategy is to bound the entries of $\Sigma^{-1}$, and then to derive an upper bound for $\opnorm{\Sigma^{-1}}$ by noticing that $\opnorm{\Sigma^{-1}}\leq \frobnorm{\Sigma^{-1}}$. 
Thus let us start by controlling the $\sigma$ coefficients in absolute value. 

\begin{lemma}[Control of the $\sigma$ coefficients]
\label{lemma:sigma-elements-control}
Let $d\geq 2$ and $\nu \geq 1.66$. 
Then it holds that
\[
\abs{\sigma_0} \leq \frac{d}{3} \, ,
\quad \abs{\sigma_1} \leq 1\, ,
\quad \abs{\sigma_2} \leq \frac{3d}{2}\exps{\frac{1}{2\nu^2}}\, ,
\quad \text{ and }\quad \abs{\sigma_3} \leq \frac{3}{2}\exps{\frac{1}{2\nu^2}}
\, .
\]
\end{lemma}

\begin{proof}
By its definition, we know that $\sigma_0$ is positive. 
Moreover, from Corollary~\ref{cor:alphas-bounds}, we see that
\begin{align*}
\sigma_0 &= (d-1)\alpha_2 + \alpha_1 \\
&\leq \frac{(d-1)(d-2)}{3d} + \frac{d-1}{2d} \\
&= \frac{2d^2-3d+3}{6d}
\, .
\end{align*}
One can check that for any $d\geq 2$, we have $2d^2-3d+3\leq 2d^2$, which concludes the proof of the first claim. 

Since $\abs{\sigma_1}=\alpha_1$, the second claim is straightforward from Corollary~\ref{cor:alphas-bounds}. 

Regarding $\sigma_2$, we notice that
\[
\sigma_2 = \frac{\dencst_d + \alpha_1^2 - \alpha_0\alpha_2}{\alpha_1-\alpha_2}
\, .
\]
Since $\alpha_0\geq \alpha_1\geq \alpha_2$, we have
\[
-\alpha_1(\alpha_0-\alpha_1) \leq \alpha_1^2-\alpha_0\alpha_2 \leq \alpha_0(\alpha_1-\alpha_2)
\, .
\]
Using Eqs.~\eqref{eq:large-nu-aux-1} and~\eqref{eq:large-nu-aux-2} in conjunction with Corollary~\ref{cor:alphas-bounds}, we find that $\abs{\alpha_1^2-\alpha_0\alpha_2} \leq 1/4$. 
Moreover, from Eq.~\eqref{eq:dencst-alt-writing}, we see that $\dencst_d\leq d/4$. 
We deduce that 
\[
\abs{\sigma_2} \leq \left(\frac{d}{4} + \frac{1}{4}\right)\cdot 6\exps{\frac{1}{2\nu^2}}
\, ,
\]
where we used the first statement of Proposition~\ref{prop:large-nu-makes-everything-ok} to lower bound $\alpha_1\alpha_2$. 
The results follows, since $d\geq 2$.  

Finally, we write
\begin{align*}
\abs{\sigma_3} &= \frac{\abs{\alpha_1^2-\alpha_0\alpha_2}}{\alpha_1-\alpha_2} \\
&\leq \frac{1/4}{\frac{d+1}{6d}\cdot \exps{\frac{-1}{2\nu^2}}}
\end{align*}
according to Proposition~\ref{prop:large-nu-makes-everything-ok}. 
\end{proof}

We now proceed to bound the operator norm of $\Sigma^{-1}$. 

\begin{proposition}[Control of $\opnorm{\Sigma^{-1}}$]
\label{prop:opnorm-control}
For any $d\geq 2$ and any $\nu > 0$, it holds that 
\[
\opnorm{\Sigma^{-1}} \leq 70 d^{3/2} \exps{\frac{5}{2\nu^2}}
\, .
\]
\end{proposition}

\begin{remark}
\label{remark:influence-of-d}
We notice that the control obtained worsens as $d\to +\infty$ and $\nu\to 0$. 
We conjecture that the dependency in $d$ is not tight. 
For instance, showing that $\dencst_d=\Omega(d)$ (that is, improving Proposition~\ref{prop:large-nu-makes-everything-ok}) would yield an upper bound of order $d$ instead of $d^{3/2}$. 
The discussion after Proposition~\ref{prop:large-nu-makes-everything-ok} indicates that such an improvement may be possible. 
Moreover, we see in experiments that the concentration of $\betahat_n$ does not degrade that much for large $d$ (see, in particular, Figure~\ref{fig:linear-large-d} in Section~\ref{sec:add-exp-linear}), another sign that Proposition~\ref{prop:opnorm-control} could be improved. 
\end{remark}

\begin{proof}
We will use the fact that $\opnorm{\Sigma^{-1}}\leq \frobnorm{\Sigma^{-1}}$. 
We first write
\[
\frobnorm{\Sigma^{-1}}^2 = \frac{1}{\dencst_d^2}\left(\sigma_0^2 + 2d\sigma_1^2 + d\sigma_2^2 + (d^2-d)\sigma_3^2\right)
\, ,
\]
by definition of the $\sigma$ coefficients. 
On one hand, using Lemma~\ref{lemma:sigma-elements-control}, we write
\begin{align}
\sigma_0^2 + 2d\sigma_1^2 + d\sigma_2^2 + (d^2-d)\sigma_3^2 &\leq \frac{d^2}{9} + 2d + d\cdot (3d/2)^2 \exps{\frac{1}{\nu^2}} + (d^2-d)\cdot \frac{9}{4}\exps{\frac{1}{\nu^2}} \notag \\
&\leq 3d^3\exps{\frac{1}{\nu^2}} \label{eq:opnorm-control-aux-1}
% NOTE: 3 was 14
\, ,
\end{align}
where we used $\dencst_d\leq d$ and $d\geq 2$ in the last display. 
On the other hand, a direct consequence of Proposition~\ref{prop:large-nu-makes-everything-ok} is that
\begin{equation}
\label{eq:opnorm-control-aux-2}
%\frac{1}{\dencst_d^2} \leq 5184^2 d^2 \exps{\frac{6}{\nu^2}}
\frac{1}{\dencst_d^2} \leq 1600\exps{\frac{4}{\nu^2}}
\, .
\end{equation}
Putting together Eqs.~\eqref{eq:opnorm-control-aux-1} and~\eqref{eq:opnorm-control-aux-2}, we obtain the claimed result, since $\sqrt{3\cdot 1600}\leq 70$. 
\end{proof}

\section{The study of $\Gamma^f$}
\label{sec:study-of-gamma}

We now turn to the study of the (weighted) responses.  
In Section~\ref{sec:gamma-computation}, we obtain an explicit expression for the average responses. 
We show how to obtain closed-form expressions in the case of indicator functions in Section~\ref{sec:gamma-computation-indicator}. 
In the case of a linear model, we have to resort to approximations that are detailed in Section~\ref{sec:gamma-computation-linear}. 
Section~\ref{sec:concentration-gammahat} contains the concentration result for $\Gammahat_n$.

\subsection{Computation of $\Gamma^f$}
\label{sec:gamma-computation}

We start our study by giving an expression for $\Gamma^f$ for any $f$ under mild assumptions. 
Recall that we defined $\Gammahat_n=\frac{1}{n}Z^\top W y$, where $y\in\Reals^{d+1}$ is the random vector defined coordinate-wise by $y_i=f(x_i)$. 
From the definition of $\Gammahat_n$, it is straightforward that
\[ 
\Gammahat_n  =
\begin{pmatrix}
\frac{1}{n}\sum_{i=1}^n \pi_{i}f(\normtfidf{x_i}) \\ 
\frac{1}{n}\sum_{i=1}^{n} \pi_{i}{z_{i,1}}f(\normtfidf{x_i})\\ 
\vdots\\ 
\frac{1}{n}\sum_{i=1}^{n} \pi_{i}{z_{i,d}}f(\normtfidf{x_i}) \\ 
\end{pmatrix}
\in\Reals^{d+1}
\, .
\]
As a consequence, since we defined $\Gamma^f=\smallexpec{\Gammahat_n}$, it holds that
\begin{equation}
\label{eq:def-gamma}
\Gamma^f = 
\begin{pmatrix}
\expec{ \pi f(\normtfidf{x}) } \\ 
\expec{\pi z_1 f(\normtfidf{x})}\\ 
\vdots\\ 
\expec{\pi z_d f(\normtfidf{x})} \\ 
\end{pmatrix}
\, .
\end{equation}
Of course, Eq.~\eqref{eq:def-gamma} depends on the model $f$. 
These computations can be challenging. 
Nevertheless, it is possible to obtain exact results in simple situations. 

\paragraph{Constant model. }
As a warm up, let us show how to compute $\Gamma^f$ when $f$ is constant. 
Perhaps the simplest model of all: $f$ always returns the same value, whatever the value of $\normtfidf{x}$ may be. 
By linearity of $\Gamma^f$ (see Section~3.2 of the main paper), it is sufficient to consider the case $f=1$. 
From Eq.~\eqref{eq:def-gamma}, we see that 
\[
\Gamma^f_j =
\begin{cases}
\expec{\pi} &\text{ if } j = 0, \\
\expec{\pi z_j} &\text{ otherwise. }
\end{cases}
\]
We recognize the definitions of the $\alpha$ coefficients, and, more precisely, $\Gamma^f_0=\alpha_0$ and $\Gamma^f_j=\alpha_1$ if $j\geq 1$.

\subsection{Indicator functions}
\label{sec:gamma-computation-indicator}

Let us turn to a slightly more complicated class of models: indicator functions, or rather products of indicator functions. 
As explained in the paper, these functions fall into our framework.
We have the following result:  

\begin{proposition}[Computation of $\Gamma^f$, product of indicator functions]
\label{prop:gamma-indicator-general}
Set $J\subseteq \{1,\ldots,d\}$ a set of $p$ distinct indices. 
Define 
\[
f(\normtfidf{x}) \defeq \prod_{j\in J} \indic{\normtfidf{x}_j > 0}
\, .
\]
Then it holds that
\[
\Gamma_\ell^f = 
\begin{cases}
\alpha_p & \text{ if } \ell \in \{0\}\cup J \\
\alpha_{p+1} & \text{ otherwise.}
\end{cases}
\]
\end{proposition}

\begin{proof}
As noticed in the paper, $f$ can be written as a product of $z_j$s. 
Therefore, we only have to compute
\[
\Expec\biggl[\pi \prod_{j\in J}z_j\biggr] \quad \text{ and } \quad \Expec\biggl[\pi z_k \prod_{j\in J}z_j\biggr]
\, ,
\]
for any $1\leq k\leq d$. 
The first term is $\alpha_p$ by definition. 
For the second term, we notice that if $\ell \in \{0\}\cup J$, then two terms are identical in the product of binary features, and we recognize the definition of $\alpha_p$. 
In all other cases, there are no cancellation and we recover the definition of $\alpha_{p+1}$. 
\end{proof}

\subsection{Linear model}
\label{sec:gamma-computation-linear}

We now consider a linear model, that is, 
\begin{equation}
\label{eq:def-linear-model}
f(\normtfidf{x}) \defeq \sum_{j=1}^d \lambda_j \normtfidf{x}_j
\, ,
\end{equation}
where $\lambda_1,\ldots,\lambda_d$ are arbitrary fixed coefficients. 
In order to simplify the computations, we will consider that $\nu \to +\infty$ in this section. 
In that case, $\pi \cvas 1$. 
It is clear that $f$ is bounded on $\sphere{D-1}$, thus, by dominated convergence, 
\begin{equation}
\label{eq:def-gamma-infty}
\Gamma^f \longrightarrow \Gammainf \defeq 
\begin{pmatrix}
\expec{f(\normtfidf{x}} \\
\expec{z_1 f(\normtfidf{x}} \\
\vdots \\
\expec{z_d f(\normtfidf{x}}
\end{pmatrix}
\in\Reals^{d+1}
\, .
\end{equation}

By linearity of $f\mapsto \Gamma^f_{\infty}$, it is sufficient to compute $\expec{\normtfidf{x}_j}$ and $\expec{z_k \normtfidf{x}_j}$ for any $1\leq j,k\leq d$.

For any $1\leq j\leq d$, recall that we defined 
\[
\omega_k = \frac{m_j^2v_j^2}{\sum_{k=1}^d m_k^2v_k^2}
\, ,
\]
and $H_S \defeq \sum_{k \in S}\omega_k$, where $S$ is the random subset of indices chosen by LIME. 
The motivation for the definition of the random variable $H_S$ is the following proposition: it is possible to write the expected TF-IDF as an expression depending on $H_S$. 

\begin{proposition}[Expected normalized TF-IDF]
\label{prop:expected-tfidf}
Let $\word_j$ be a fixed word of $\xi$. 
Then, it holds that 
\begin{equation}
\label{eq:expected-tfidf}
\expec{\normtfidf{x}_j} = \expec{z_j\normtfidf{x}_j} = \frac{d-1}{2d}\cdot \normtfidf{\xi}_j \cdot \condexpec{\frac{1}{\sqrt{1 - H_S}}}{S\not\ni j}
\, ,
\end{equation}
and, for any $k\neq j$, 
\begin{equation}
\label{eq:expected-tfidf-z}
\expec{z_k\normtfidf{x}_j} = \frac{d-2}{3d} \cdot \normtfidf{\xi}_j \cdot \condexpec{\frac{1}{\sqrt{1-H_S}}}{S\not\ni j,k}
\, .
\end{equation}
\end{proposition}

\begin{proof}
We start by proving Eq~\eqref{eq:expected-tfidf}. 
Let us split the expectation depending on $\word_j\in x$. 
Since the term frequency is $0$ if $\word_j\notin x$, we have
\begin{equation}
\label{eq:tfidf-aux-1}
\expec{\normtfidf{x}_j} = \condexpec{\normtfidf{x}_j}{w_j\in x} \proba{w_j\in x}
\, .
\end{equation}
Lemma~\ref{lemma:proba-containing} gives us the value of $\proba{w_j\in x}$. 
Let us focus on the TF-IDF term in Eq.~\eqref{eq:tfidf-aux-1}. 
By definition, it is the product of the term frequency and the inverse document frequency, normalized. 
Since the latter does not change when words are removed from $\xi$, only the norm changes: we have to remove all terms indexed by $S$. 
For any $1\leq j\leq d$, let us set $m_j$ (resp. $v_j$) the term frequency (resp. the inverse term frequency) of $\word_j$ 
Conditionally to $\{\word_j\in x\}$,
\[
\normtfidf{x}_j = \frac{m_j v_j}{\sqrt{\sum_{k \notin S} m_k^2 v_k^2}}
\, .
\]
Let us factor out $\normtfidf{\xi}_j$ in the previous display. 
By definition of $H_S$, we have
\[
\normtfidf{x}_j = \normtfidf{\xi}_j \cdot \frac{1}{\sqrt{1 - \sum_{k\in S} \frac{m_k^2v_k^2}{\norm{\tfidf{\xi}}^2}}} = \normtfidf{\xi}_j \cdot \frac{1}{\sqrt{1-H_S}}
\, .
\]
Since $\{w_j\in x\}$ is equivalent to $\{j\notin S\}$ by construction, we can conclude. 
The proof of the second statement is similar; one just has to condition with respect to $\{w_j,w_k\in x\}$ instead, which is equivalent to $\{S\not\ni j,k\}$. 
\end{proof}

As a direct consequence of Proposition~\ref{prop:expected-tfidf}, we can derive $\Gamma_{\infty}^f=\lim_{\nu\to +\infty}\Gamma^f$ when $f:x\mapsto x_j$. 
Recall that we set $E_j = \condexpec{(1-H_S)^{-1/2}}{S\not\ni j}$ and $E_{j,k} = \condexpec{(1-H_S)^{-1/2}}{S\not\ni j,k}$. 
Then
\begin{equation}
\label{eq:gamma-computation-linear}
\left(\Gamma_{\infty}^f\right)_k = 
\begin{cases}
\left(\frac{1}{2}-\frac{1}{2d}\right) \cdot E_j \cdot \normtfidf{\xi}_j &\text{ if } k=0 \text{ or } k=j, \\
\left(\frac{1}{3}-\frac{2}{3d}\right) \cdot E_{j,k} \cdot \normtfidf{\xi}_j  &\text{ otherwise.}
\end{cases}
\end{equation}

In practice, the expectation computations required to evaluate $E_j$ and $E_{j,k}$ are not tractable as soon as $d$ is large. 
Indeed, in that case, the law of $H_S$ is unknown and approximating the expectation by Monte-Carlo methods requires is hard since one has to sum over all subsets and there are $\bigo{2^d}$ subsets $S$ such that $S\subseteq \{1,\ldots,d\}$. 
Therefore we resort to approximate expressions for these expected values computations. 

We start by writing
\begin{equation}
\label{eq:main-approx-expec}
\expec{\frac{1}{\sqrt{1-X}}} \approx \frac{1}{\sqrt{1-\expec{X}}}
\, .
\end{equation}
All that is left to compute will be $\condexpec{H_S}{S\not\ni j}$ and $\condexpec{H_S}{S\not\ni j,k}$. 
We see in Section~\ref{sec:subsets-sums} that after some combinatoric considerations, it is possible to obtain these expected values as a function of $\omega_j$ and $\omega_k$. 
More precisely, Lemma~\ref{lemma:expectation-computation} states that
\begin{equation}
\label{eq:approx-expec-hs}
\condexpec{H_S}{S\not\ni j} = \frac{1-\omega_j}{3} + \bigo{\frac{1}{d}}
\quad \text{ and }\quad 
\condexpec{H_S}{S\not\ni j,k} = \frac{1-\omega_j-\omega_k}{4} + \bigo{\frac{1}{d}}
\, .
\end{equation}
When $d$ is large and the $\omega_k$s are small, using Eq.~\eqref{eq:main-approx-expec}, we obtain the following approximations:
\begin{equation}
\label{eq:approx-norm-tfidf}
\expec{\normtfidf{x}_j} \approx \frac{1}{2}\cdot \sqrt{\frac{1}{1-\frac{1}{3}}} \cdot \normtfidf{\xi}_j \approx 0.61 \cdot \normtfidf{\xi}_j
\, ,
\end{equation}
and, for any $k\neq j$, 
\begin{equation}
\label{eq:approx-norm-tfidf-2}
\expec{z_k\normtfidf{x}_j} \approx \frac{1}{3}\cdot \sqrt{\frac{1}{1-\frac{1}{4}}}\cdot \normtfidf{\xi}_j \approx 0.38 \cdot \normtfidf{\xi}_j
\, .
\end{equation}
For all practical purposes, we will use Eq.~\eqref{eq:approx-norm-tfidf} and~\eqref{eq:approx-norm-tfidf-2}. 

\begin{remark}
One could obtain better approximations than above in two ways. 
First, it is possible to take into account the dependency in $\omega_j$ and $\omega_k$ in the expectation of $H_S$. 
That is, plugging Eq.~\eqref{eq:approx-expec-hs} into Eq.~\eqref{eq:main-approx-expec} instead of the numerical values $1/3$ and $1/4$. 
This yields more accurate, but more complicated formulas. 
Without being so precise, it is also possible to consider an arbitrary distribution for the $\omega_k$s (for instance, assuming that the term frequencies follow the Zipf's law \citep{powers_1998}). 
Second, since the mapping $\theta : x\mapsto \frac{1}{\sqrt{1-x}}$ is convex, by Jensen's inequality, we are always \emph{underestimating} by considering $\theta(\expec{X})$ instead of $\expec{\theta(X)}$. 
Going further in the Taylor expansion of $\theta$ is a way to fix this problem, namely using
\[
\expec{\frac{1}{\sqrt{1-X}}} \approx \frac{1}{\sqrt{1-\expec{X}}} + \frac{3\var{X}}{8\sqrt{1-\expec{X}}}
\, ,
\]
instead of Eq.~\eqref{eq:main-approx-expec}.
We found that \textbf{it was not useful to do so from an experimental point of view:} our theoretical predictions match the experimental results while remaining simple enough. 
\end{remark}

\subsection{Concentration of $\Gammahat_n$}
\label{sec:concentration-gammahat}

We now show that $\Gammahat_n$ is concentrated around $\Gamma^f$. 
Since the expression of $\Gammahat_n$ is the same than in the tabular case, and since $f$ is bounded on the unit sphere $\sphere{D-1}$, the same reasoning as in the proof of Proposition~24 in \citet{garreau_luxburg_2020_arxiv} can be applied. 

\begin{proposition}[Concentration of $\Gammahat_n$]
\label{prop:concentration-gammahat}
Assume that $f$ is bounded by $M>0$ on $\sphere{D-1}$. 
Then, for any $t>0$, it holds that 
\[
\proba{\smallnorm{\Gammahat_n - \Gamma^f} \geq t} \leq 4d \exp{\frac{-nt^2}{32Md^2}}
\, .
\]
\end{proposition}

\begin{proof}
Recall that $\norm{\normtfidf{x}}=1$ almost surely. 
Since $f$ is bounded by $M$ on $\sphere{D-1}$, it holds that $\abs{f(\normtfidf{x})}\leq M$ almost surely. 
We can then proceed as in the proof of Proposition~24 in \citet{garreau_luxburg_2020_arxiv}. 
\end{proof}

\section{The study of $\beta^f$}
\label{sec:study-of-beta}

In this section, we study the interpretable coefficients. 
We start with the computation of $\beta^f$ in Section~\ref{sec:beta-computation}. 
In Section~\ref{sec:betahat-concentration}, we show how $\betahat_n$ concentrates around $\beta^f$. 

\subsection{Computation of $\beta^f$}
\label{sec:beta-computation}

Recall that, for any model $f$, we have defined $\beta^f = \Sigma^{-1}\Gamma^f$. 
Directly multiplying the expressions found for $\Sigma^{-1}$ (Eq.~\eqref{eq:sigma-inverse-computation}) and $\Gamma^f$ (Eq.~\eqref{eq:def-gamma}) obtained in the previous sections, we obtain the expression of $\beta^f$ in the general case (this is Proposition~2 in the paper). 

\begin{proposition}[Computation of $\beta^f$, general case]
\label{prop:beta-computation-general}
Assume that $f$ is bounded on the unit sphere. 
Then 
\begin{equation}
\label{eq:beta-computation-intercept}
\beta^f_0 = \dencst^{-1}_d\biggl\{\sigma_0\expec{\pi f(\normtfidf{x})} + \sigma_1\sum_{k=1}^d \expec{\pi z_k f(\normtfidf{x})}\biggr\}
\, ,
\end{equation}
and, for any $1\leq j\leq d$, 
\begin{equation}
\label{eq:beta-computation-general}
\beta^f_j = 
\dencst^{-1}_d\biggl\{\sigma_1 \expec{\pi f(\normtfidf{x})} + \sigma_2 \expec{\pi z_j f(\normtfidf{x})} + \sigma_3 \sum_{\substack{k=1 \\ k\neq j}}^d \expec{\pi z_k f(\normtfidf{x})}\biggr\}
\, .
\end{equation}
\end{proposition}

This is Proposition~2 in the paper, with the additional expression of the intercept $\beta_0^f$. 
Let us see how to obtain an approximate, simple expression when both the bandwidth parameter and the size of the local dictionary are large.  
When $\nu\to +\infty$, using Corollary~\ref{cor:approximate-sigma-inverse}, we find that 
\[
\beta_0^f \longrightarrow \left(\betainf^f\right)_0\defeq \frac{4d-2}{d+1}\expec{\pi f(\normtfidf{x})} - \frac{6}{d+1}\sum_{k=1}^d \expec{\pi z_k f(\normtfidf{x})}
\, ,
\]
and, for any $1\leq j\leq d$,
\[
\beta_j^f \longrightarrow \left(\betainf^f\right)_j \defeq \frac{-6}{d+1}\expec{\pi f(\normtfidf{x})} + \frac{6(d^2-2d+3)}{d^2-1}\expec{\pi z_j f(\normtfidf{x})} - \frac{6(d-3)}{d^2-1}\sum_{k\neq j} \expec{\pi z_k f(\normtfidf{x})}
\, .
\]
For large $d$, since $f$ is bounded on $\sphere{D-1}$, we find that
\[
\left(\betainf^f\right)_0 = 4\expec{\pi f(\normtfidf{x})} - \frac{6}{d}\sum_{k=1}^d \expec{\pi z_k f(\normtfidf{x})} + \bigo{\frac{1}{d}}
\, ,
\]
and, for any $1\leq j\leq d$,
\[
\left(\betainf^f\right)_j =  6\expec{\pi z_j f(\normtfidf{x})} - \frac{6}{d}\sum_{k\neq j} \expec{\pi z_k f(\normtfidf{x})} + \bigo{\frac{1}{d}}
\, .
\]
Now, by definition of the interpretable features, for any $1\leq j\leq d$, 
\begin{align*}
\expec{\pi z_j f(\normtfidf{x})} &= \condexpec{\pi z_j f(\normtfidf{x})}{\word_j \in x} \cdot \proba{\word_j \in x} + \condexpec{\pi z_j f(\normtfidf{x})}{\word_j \notin x} \cdot \proba{\word_j \notin x} \\
&= \condexpec{\pi f(\normtfidf{x})}{\word_j \in x}\cdot \frac{d-1}{2d} + 0
\, ,
\end{align*}
where we used Lemma~\ref{lemma:proba-containing} in the last display. 
Therefore, we have the following approximations of the interpretable coefficients: 
\begin{equation}
\label{eq:betainf-simplified-intercept}
\left(\betainf^f\right)_0 = 2\expec{\pi f(\normtfidf{x})} - \frac{3}{d}\sum_k \condexpec{\pi f(\normtfidf{x})}{\word_k\in x} + \bigo{\frac{1}{d}}
\, ,
\end{equation}
and, for any $1\leq j\leq d$,
\begin{equation}
\label{eq:betainf-simplified}
\left(\betainf^f\right)_j = 3\condexpec{\pi f(\normtfidf{x})}{\word_j\in x} - \frac{3}{d}\sum_k \condexpec{\pi f(\normtfidf{x})}{\word_k\in x} + \bigo{\frac{1}{d}}
\, .
\end{equation}
The last display is the approximation of Proposition~\ref{prop:beta-computation-general} presented in the paper. 

\begin{remark}
In \citet{garreau_luxburg_2020_arxiv}, it is noted that LIME for tabular data provably ignores unused coordinates.
In other words, if the model $f$ does not depend on coordinate $j$, then the explanation $\beta^f_j$ is $0$. 
We could not prove such a statement in the case of text data, even for simplified expressions such as Eq.~\eqref{eq:betainf-simplified}. 
\end{remark}

We now show how to compute~$\beta^f$ in specific cases, thus returning to generic $\nu$ and $d$.

\paragraph{Constant model. }
As a warm up exercise, let us assume that $f$ is a constant, which we set to $1$ without loss of generality (by linearity). 
Recall that, in that case, $\Gamma^f_0=\alpha_0$ and $\Gamma^f_j=\alpha_1$ for any $1\leq j\leq d$. 
From the definition of $\dencst_d$ and the $\sigma$ coefficients (Proposition~\ref{prop:sigma-inverse-computation}), we find that 
\[
\begin{cases}
\sigma_0 \alpha_0 + d\sigma_1\alpha_1 &= \dencst_d \, ,\\
\sigma_1\alpha_0 + \sigma_2\alpha_1 + (d-1)\sigma_3\alpha_1 &= 0 
\, .
\end{cases}
\]
We deduce from Proposition~\ref{prop:beta-computation-general} that $\beta^f_0=1$ and $\beta^f_j=0$ for any $1\leq j\leq d$. 
This is conform to our intuition: if the model is constant, then no word should receive nonzero weight in the explanation provided by Text LIME.  

\paragraph{Indicator functions. }
We now turn to indicator functions, more precisely \emph{products} of indicator functions. 
We will prove the following (Proposition~3 in the paper):

\begin{proposition}[Computation of $\beta^f$, product of indicator functions]
\label{prop:beta-computation-indicator-product-general}
Let $j\subseteq \{1,\ldots,d\}$ be a set of $p$ distinct indices and set $f(x) = \prod_{j\in J}\indic{x_j>0}$. 
Then 
\[
\begin{cases}
\beta_0^f &= \dencst_d^{-1}\left(\sigma_0\alpha_p+p\sigma_1\alpha_p+(d-p)\sigma_1\alpha_{p+1}\right)\, , \\
\beta_j^f &= \dencst_d^{-1}\left(\sigma_1\alpha_p + \sigma_2\alpha_p + (d-p)\sigma_3\alpha_{p+1} + (p-1)\sigma_3\alpha_p\right) \text{ if }j \in J\, ,\\
\beta_j^f &= \dencst_d^{-1}\left(\sigma_1\alpha_p+\sigma_2\alpha_{p+1}+(d-p-1)\sigma_3\alpha_{p+1} + p\sigma_3\alpha_p \right) \text{ otherwise}
\, .
\end{cases}
\]
\end{proposition}

\begin{proof}
The proof is straightforward from  Proposition~\ref{prop:gamma-indicator-general} and Proposition~\ref{prop:beta-computation-general}.  
\end{proof}

% linear case
\paragraph{Linear model. }
In this last paragraph, we treat the linear case. 
As noted in Section~\ref{sec:gamma-computation-linear}, we have to resort to approximate computations: in this paragraph, we assume that $\nu = +\infty$. 
We start with the simplest linear function: all coefficients are zero except one (this is Proposition~4 in the paper). 

\begin{proposition}[Computation of $\beta^f$, linear case]
\label{prop:beta-computation-linear}
Let $1\leq j\leq d$ and assume that $f(\normtfidf{x})=\normtfidf{x}_j$. 
Recall that we set $E_j= \condexpec{(1-H_S)^{-1/2}}{S\not\ni j}$ and for any $k\neq j$, $E_{j,k} = \condexpec{(1-H_S)^{-1/2}}{S\not\ni j,k}$. 
Then 
\begin{equation*}
\left(\beta_\infty^f\right)_0 =
\left\{5 E_j -\frac{2}{d} \sum_{k \neq j}E_{j,k}\right\} \normtfidf{\xi}_j + \bigo{\frac{1}{d}}
\end{equation*}
for any $k\neq j$, 
\begin{equation*}
\left(\beta_\infty^f\right)_k = 
\left\{2E_{j,1} -\frac{2}{d}\sum_{\ell \neq k,j}E_{j,\ell} \right\}\normtfidf{\xi}_j + \bigo{\frac{1}{d}}
\, ,
\end{equation*}
and
\begin{equation*}
\left(\beta_\infty^f\right)_j =
\left\{3E_j -\frac{2}{d} \sum_{k \neq j}E_{j,k}\right\} \normtfidf{\xi}_j + \bigo{\frac{1}{d}}
\, .
\end{equation*}
\end{proposition}

\begin{proof}
Straightforward from Eqs.~\eqref{eq:def-sigma-infty} and~\eqref{eq:gamma-computation-linear}. 
\end{proof}

Assuming that the $\omega_k$ are small, we deduce from Eqs.~\eqref{eq:approx-norm-tfidf} and~\eqref{eq:approx-norm-tfidf-2} that $E_j \approx 1.22$ and $E_{j,k}\approx 1.15$. 
In particular, \emph{they do not depend on $j$ and $k$.} 
Thus we can drastically simplify the statement of Proposition~\ref{prop:beta-computation-linear}:
\begin{equation}
\label{eq:simplified-betainf-linear-1}
\forall k\neq j,\quad \left(\beta_\infty^f\right)_k \approx 0
\quad \text{ and } \left(\beta_\infty^f\right)_j \approx 1.36 \normtfidf{\xi}_j
\, .
\end{equation}
We can now go back to our original goal: $f(x) = \sum_{j=1}^{d}\lambda_j x_j$. 
By linearity, we deduce from Eq.~\eqref{eq:simplified-betainf-linear-1} that 
\begin{equation}
\label{eq:simplified-betainf-linear}
\forall 1\leq j\leq d, \quad \left(\beta_\infty^f\right)_j \approx 1.36 \cdot \lambda_j \cdot  \normtfidf{\xi}_j
\, .
\end{equation}
In other words, as noted in the paper, \textbf{the explanation for a linear~$f$ is the TF-IDF of the word multiplied by the coefficient of the linear model,} up to a numerical constant and small error terms depending on~$d$.

%%%%%%%%%%%%%%%%%%%%%%%%%%%%%%%%%%%%%%%%%%%%%%%%%%%%%%%%%%%%%%%%%%%%%%%%%%%%%%%%%%

\subsection{Concentration of $\betahat$}
\label{sec:betahat-concentration}

In this section, we state and prove our main result: the concentration of $\betahat_n$ around $\beta^f$ with high probability (this is Theorem~1 in the paper). 

\begin{theorem}[Concentration of $\betahat_n$]
\label{th:betahat-concentration}
Suppose that $f$ is bounded by $M>0$ on $\sphere{D-1}$. 
Let $\epsilon >0$ be a small constant, at least smaller than $M$. 
Let $\eta\in (0,1)$.  
Then, for every 
\[
n\geq \max \left\{2^9\cdot 70^4 M^2d^{9} \exps{\frac{10}{\nu^2}}, 2^9\cdot 70^2 Md^5\exps{\frac{5}{\nu^2}}\right\} \frac{\log \frac{8d}{\eta}}{\epsilon^2}
\, ,
\]
we have $\proba{\smallnorm{\betahat_n - \beta^f} \geq \epsilon} \leq \eta$. 
\end{theorem}

\begin{proof}
We follow the proof scheme of Theorem~28 in \citet{garreau_luxburg_2020_arxiv}. 
The key point is to notice that 
\begin{equation}
\label{eq:binding-lemma}
\smallnorm{\betahat_n-\beta^f} \leq 2\opnorm{\Sigma^{-1}} \smallnorm{\Gammahat_n-\Gamma^f} + 2\opnorm{\Sigma^{-1}}^2 \norm{\Gamma^f}\smallopnorm{\Sigmahat_n-\Sigma} 
\, ,
\end{equation}
provided that $\smallopnorm{\Sigma^{-1}(\Sigmahat_n-\Sigma)}\leq 0.32$ (this is Lemma~27 in \citet{garreau_luxburg_2020_arxiv}. 
Therefore, in order to show that $\smallnorm{\betahat_n-\beta^f}\leq \epsilon$, it suffices to show that each term in Eq.~\eqref{eq:binding-lemma} is smaller than $\epsilon/4$ and that $\smallopnorm{\Sigma^{-1}(\Sigmahat-\Sigma)}\leq 0.32$. 
The concentration results obtained in Section~\ref{sec:study-of-sigma} and ~\ref{sec:study-of-gamma} guarantee that both $\smallopnorm{\Sigmahat-\Sigma}$ and $\smallnorm{\Gammahat-\Gamma^f}$ are small if $n$ is large enough, with high probability. 
This, combined with the upper bound on $\smallopnorm{\Sigma^{-1}}$ given by Proposition~\ref{prop:opnorm-control}, concludes the proof. 

Let us give a bit more details. 
We start with the control of $\smallopnorm{\Sigma^{-1}(\Sigmahat_n-\Sigma)}$. 
Set $t_1\defeq (220 d^{3/2}\exps{\frac{5}{2\nu^2}})^{-1}$ and $n_1\defeq 32d^2\log \frac{8d}{\eta} / t_1^2$. 
Then, according to Proposition~\ref{prop:sigmahat-concentration}, for any $n\geq n_1$, 
\[
\proba{\smallopnorm{\Sigmahat_n-\Sigma} \geq t_1} \leq 4d\exp{\frac{-nt_1^2}{32d^2}} \leq \frac{\eta}{2}
\, .
\]
Since $\smallopnorm{\Sigma^{-1}}\leq 70 d^{3/2}\exps{\frac{5}{2\nu^2}}$ (according to Proposition~\ref{prop:opnorm-control}), by sub-multiplicativity of the operator norm, it holds that
\begin{equation}
\label{eq:proof-main-aux-1}
\smallopnorm{\Sigma^{-1}(\Sigmahat-\Sigma)} \leq \smallopnorm{\Sigma^{-1}} \smallopnorm{\Sigmahat-\Sigma} \leq 70/220 < 0.32
\, ,
\end{equation}
with probability greater than $1-\eta/2$. 

Now let us set $t_2\defeq (4\cdot 70^2 M d^{7/2} \exps{\frac{5}{\nu^2}})^{-1}\epsilon$ and $n_2 \defeq 32d^2 \log \frac{8d}{\eta} / t_2^2$. 
According to Proposition~\ref{prop:sigmahat-concentration}, for any $n\geq n_2$, it holds that 
\[
\smallopnorm{\Sigmahat_n-\Sigma} \leq \frac{\epsilon}{4Md^{1/2}} \cdot (70^2 d^3 \exps{5/\nu^2})^{-1}
\, ,
\]
with probability greater than $\eta/2$. 
Since $\smallnorm{\Gamma^f}\leq M\cdot d^{1/2}$ and $\smallopnorm{\Sigma^{-1}}^2\leq 70^2d^3\exps{5/\nu^2}$, 
\[
\opnorm{\Sigma^{-1}} \smallnorm{\Gammahat-\Gamma^f} \leq \frac{\epsilon}{4}
\]
with probability grater than $1-\eta/2$. 
Notice that, since we assumed $\epsilon < M$, $t_2< t_1$, and thus Eq.~\eqref{eq:proof-main-aux-1} also holds. 

Finally, let us set $t_3\defeq \epsilon / (4\cdot 70 d^{3/2}\exps{\frac{5}{2\nu^2}})$ and $n_3\defeq 32Md^2\log \frac{8d}{\eta}/t_3^2$. 
According to Proposition~\ref{prop:concentration-gammahat}, for any $n\geq n_3$, 
\[
\proba{\smallnorm{\Gammahat_n-\Gamma^f} \geq t_3} \leq 4d\exp{\frac{-nt_3^2}{32Md^2}} \leq \frac{\eta}{2}
\, .
\]
Since $\smallopnorm{\Sigma^{-1}}\leq 70d^{3/2}\exps{\frac{5}{2\nu^2}}$, we deduce that 
\[
\opnorm{\Sigma^{-1}}^2 \norm{\Gamma^f}\smallopnorm{\Sigmahat_n-\Sigma}  \leq \frac{\epsilon}{2}
\, ,
\]
with probability greater than $1-\eta/2$. 
We conclude by a union bound argument. 
\end{proof}

\section{Sums over subsets}
\label{sec:subsets-sums}

In this section, independent from the rest, we collect technical facts about sums over subsets. 
More particularly, we now consider arbitrary, fixed positive real numbers $\omega_1,\ldots,\omega_d$ such that $\sum_k \omega_k = 1$. 
We are interested in subsets $S$ of $\{1,\ldots,d\}$. 
For any such $S$, we define $H_S\defeq \sum_{k\in S}\omega_k$ the sum of the $\omega_k$ coefficients over $S$. 
Our main goal in this section is to compute the expectation of $H_S$ conditionally to $S$ not containing a given index (or two given indices), which is the key quantity appearing in Proposition~\ref{prop:beta-computation-linear}.

\begin{lemma}[First order subset sums]
\label{lemma:first-order-subset-sums}
Let $1\leq s\leq d$ and $1\leq j,k\leq d$ with $j\neq k$. 
Then 
\[
\sum_{\substack{\card{S} = s \\ S\not\ni j}} H_S = \binom{d-2}{s-1}(1-\omega_j)
\, ,
\]
and 
\[
\sum_{\substack{\card{S} = s \\ S\not\ni j,k}} H_S = \binom{d-3}{s-1}(1-\omega_j-\omega_k)
\, .
\]
\end{lemma}

\begin{proof}
The main idea of the proof is to rearrange the sum, summing over all indices and then counting how many subsets satisfy the condition. 
That is, 
\begin{align*}
\sum_{\substack{\card{S} = s \\ S \ni j}} H_S &= \sum_{k=1}^d \omega_k \cdot \cardset{S \text{ s.t. } j,k \in S} \\
&= \sum_{k\neq j} \omega_k \cdot \binom{d-2}{s-2} + \omega_j \cdot \binom{d-1}{s-1} \\
&= \binom{d-2}{s-2} + \left[ \binom{d-1}{s-1} - \binom{d-2}{s-2} \right]\omega_j
\, . 
\end{align*}
We conclude by using the binomial identity
\[
\binom{d-1}{s-1} - \binom{d-2}{s-2} = \binom{d-2}{s-1}
\, .
\]
Notice that, in the previous derivation, we had to split the sum to account for the case $j=k$. 
The proof of the second formula is similar. 
\end{proof}

Let us turn to expectation computation that are important to derive approximation in Section~\ref{sec:gamma-computation-linear}.
We now see $S$ and $H_S$ as random variables. 
We will denote by $\expecunder{\cdot}{s}$ the expectation conditionally to the event $\{\card{S}=s\}$. 

\begin{lemma}[Expectation computation]
\label{lemma:expectation-computation}
Let $j,k$ be distinct elements of $\{1,\ldots,d\}$. 
Then
\begin{equation}
\label{eq:subset-sum-expectation-computation}
\condexpec{H_S}{S\not\ni j} = \frac{(1-\omega_j)(d+1)}{3(d-1)} = \frac{1-\omega_j}{3} + \bigo{\frac{1}{d}}
\, ,
\end{equation}
and
\begin{equation}
\label{eq:subset-sum-expectation-2}
\condexpec{H_S}{S\not\ni j,k} = \frac{(1-\omega_j-\omega_k)(d+1)}{4(d-2)} = \frac{1-\omega_j-\omega_k}{4} + \bigo{\frac{1}{d}}
\end{equation}
\end{lemma}

\begin{proof}
By the law of total expectation, we know that 
\[
\condexpec{H_S}{S\not\ni j} = \sum_{s=1}^{d} \condexpecunder{H_S}{S\not\ni j}{s} \cdot \condproba{\card{S}=s}{S\not\ni j}
\, .
\]
We first notice that, for any $s< d$, 
\begin{align*}
\condproba{\card{S}=s}{S\not\ni j} &= \frac{\condproba{S\not\ni j}{\card{S} = s} \proba{\card{S}=s}}{\proba{j\notin S}} \\
&= \frac{\binom{d-1}{s} / \binom{d}{s}\cdot \frac{1}{d}}{\frac{d-1}{2d}}\\
\condproba{\card{S}=s}{S\not\ni j} &= \frac{2(d-s)}{d(d-1)}
\, .
\end{align*}
According to Lemma~\ref{lemma:first-order-subset-sums}, for any $1\leq s < d$, 
\[
\sum_{\substack{\card{S} = s \\ S\not\ni j}} H_S = \binom{d-2}{s-1}(1-\omega_j)
\, .
\]
Moreover, there are $\binom{d-1}{s}$ such subsets. 
Since $\binom{d-1}{s-1}^{-1}\binom{d-2}{s}=\frac{s}{d-1}$, we deduce that
\[
\condexpecunder{H_S}{S\not\ni j}{s} = \frac{s}{d-1}(1-\omega_j)
\, .
\]
Finally, we write
\begin{align*}
\condexpec{H_S}{S\not\ni j} &= \sum_{s=1}^{d-1} \frac{s}{d-1}(1-\omega_j) \cdot \frac{2(d-s)}{d(d-1)} \\
&= (1-\omega_j) \cdot \frac{2}{d(d-1)^2} \sum_{s=1}^{d-1} s(d-s) \\
\condexpec{H_S}{S\not\ni j} &= \frac{(d+1)(1-\omega_j)}{3(d-1)}
\, .
\end{align*}
The second case is similar. 
One just has to note that
\begin{align*}
\condproba{\card{S}=s}{S\not\ni j,k} &=  \frac{\condproba{S\not\ni j,k}{\card{S}=s}}{\proba{j,k \notin S}} \\
&= \frac{3(d-s)(d-s-1)}{d(d-1)(d-2)} \tag{Lemma~\ref{lemma:proba-containing}}
\, .
\end{align*}
Then we can conclude since 
\[
\sum_{s=1}^{d-2} s(d-s)(d-s-1) = \frac{(d-2)(d-1)d(d+1)}{12}
\, .
\]
\end{proof}

\section{Technical results}
\label{sec:technical}

In this section, we collect small probability computations that are ubiquitous in our derivations. 
We start with the probability for a given word to be present in the new sample $x$, conditionally to $\card{S} =s$. 

\begin{lemma}[Conditional probability to contain given words]
\label{lemma:proba-containing-cond}
Let $\word_1,\ldots,\word_p$ be $p$ distinct words of $\dl$. 
Then, for any $1\leq s\leq d$, 
\[
\probaunder{\word_1\in x,\ldots,\word_p\in x}{s} = \frac{(d-s)(d-s-1)\cdots (d-s-p+1)}{d(d-1)\cdots (d-p+1)} = \frac{(d-s)!}{(d-s-p)!}\cdot \frac{(d-p)!}{d!}
\, .
\]
\end{lemma}

In the proofs, we use extensively Lemma~\ref{lemma:proba-containing-cond} for $p=1$ and $p=2$, that is,
\[
\probaunder{\word_j\in x}{s} = \frac{d-s}{d} 
\quad \text{ and } \quad 
\probaunder{\word_j \in x, \word_k \in x}{s} = \frac{(d-s)(d-s-1)}{d(d-1)}
\, ,
\]
for any $1\leq j,k\leq d$ with $j\neq k$. 

\begin{proof}
We prove the more general statement. 
Conditionally to $\card{S} =s$, the choice of $S$ is uniform among all subsets of $\{1,\ldots,d\}$ of cardinality $s$. 
There are $\binom{d}{s}$ such subsets, and only $\binom{d-p}{s}$ of them do not contain the indices corresponding to $\word_1,\ldots,\word_p$.
\end{proof}

We have the following result, without conditioning on the cardinality of $S$: 

\begin{lemma}[Probability to contain given words]
	\label{lemma:proba-containing}
Let $\word_1,\ldots,\word_p$ be $p$ distinct words of $\dl$. 
	Then
	\[
	\proba{\word_1,\ldots,\word_p\in x} = \frac{d-p}{(p+1)d}
	\, .
	\]
\end{lemma}

\begin{proof}
By the law of total expectation,
\begin{align*}
\proba{\word_1,\ldots,\word_p\in x} &= \frac{1}{d}\sum_{s=1}^d \condproba{\word_1,\ldots,\word_p\in x}{s} \\
&= \frac{1}{d}\sum_{s=1}^d \frac{(d-s)!}{(d-s-p)!}\cdot \frac{(d-p)!}{d!}
\, ,
\end{align*}
where we used Lemma~\ref{lemma:proba-containing-cond} in the last display. 
By the hockey-stick identity~\citep{ross_1997}, we have
\[
\sum_{s=1}^d \binom{d-s}{p} = \sum_{s=p}^{d-1} \binom{s}{p} = \binom{d}{p+1}
\, .
\]
We deduce that 
\begin{equation}
\label{eq:aux-limit-1}
\sum_{s=1}^d \frac{(d-s)!}{(d-s-p)!} = \frac{d!}{(p+1)\cdot (d-p-1)!}
\, .
\end{equation}
We deduce that 
\begin{align*}
\proba{\word_1,\ldots,\word_p\in x}
&= \frac{1}{d} \frac{(d-p)!}{d!} \sum_{s=1}^d \frac{(d-s)!}{(d-s-p)!} \\
&= \frac{1}{d} \frac{(d-p)!}{d!} \frac{d!}{(p+1)\cdot (d-p-1)!} \tag{by Eq.~\eqref{eq:aux-limit-1}} \\
\proba{\word_1,\ldots,\word_p\in x} &= \frac{d-p}{(p+1)d}
\, .
\end{align*}
\end{proof}

\section{Additional experiments}
\label{sec:experiments}

In this section, we present additional experiments. 
We collect the experiments related to decision trees in Section~\ref{sec:add-exp-trees} and those related to linear models in Section~\ref{sec:add-exp-linear}. 

\paragraph{Setting.}
All the experiments presented here and in the paper are done on Yelp reviews (the data are publicly available at \url{https://www.kaggle.com/omkarsabnis/yelp-reviews-dataset}). 
For a given model $f$, the general mechanism of our experiments is the following. 
For a given document $\xi$ containing $d$ distinct words, we set a bandwidth parameter $\nu$ and a number of new samples $n$. 
Then we run LIME $n_\text{exp}$ times on $\xi$, with no feature selection procedure (that is, all words belonging to the local dictionary receive an explanation). 
We want to emphasize again that this is the only difference with the default implementation. 
Unless otherwise specified, the parameters of LIME are chosen by default, that is, $\nu=0.25$ and $n=5000$.
The number of experiments $n_\text{exp}$ is set to $100$. 
The whisker boxes are obtained by collecting the empirical values of the $n_\text{exp}$ runs of LIME: they give an indication as to the variability in explanations due to the sampling of new examples. 
Generally, we report a subset of the interpretable coefficients, the other having near zero values. 

Let us explain briefly how to read these whisker boxes: to each word corresponds a whisker box containing all the $n_\text{exp}$ values of interpretable coefficients provided by LIME ($\betahat_j$ in our notation). 
The horizontal dark lines mark the quartiles of these values, and the horizontal blue line is the median. 
On top of these experimental results, we report with red crosses the values predicted by our analysis ($\beta_j^f$ in our notation).

The Python code for all experiments is available at \url{https://github.com/dmardaoui/lime_text_theory}.
We encourage the reader to try and run the experiments on other examples of the dataset and with other parameters.   

\subsection{Decision trees}
\label{sec:add-exp-trees}

In this section, we present additional experiments for small decision trees. 
We begin by investigating the influence of $\nu$ and $n$ on the quality of our theoretical predictions. 

% one figure with influence of the bandwidth (n default) 
% n_\text{features}=100
\begin{figure}
\centering
\includegraphics[scale=0.21]{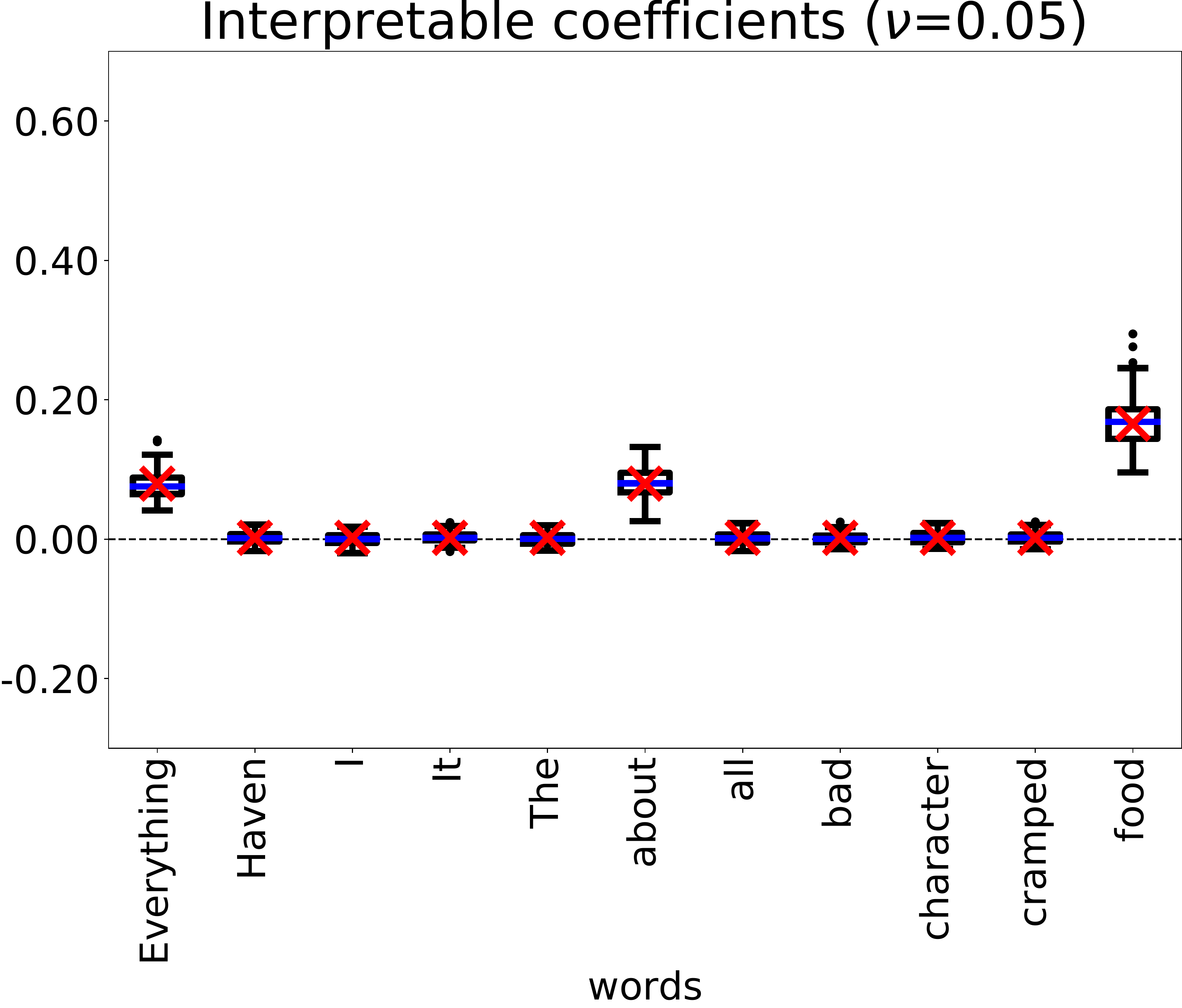} 
\hspace{0.5cm}
\includegraphics[scale=0.21]{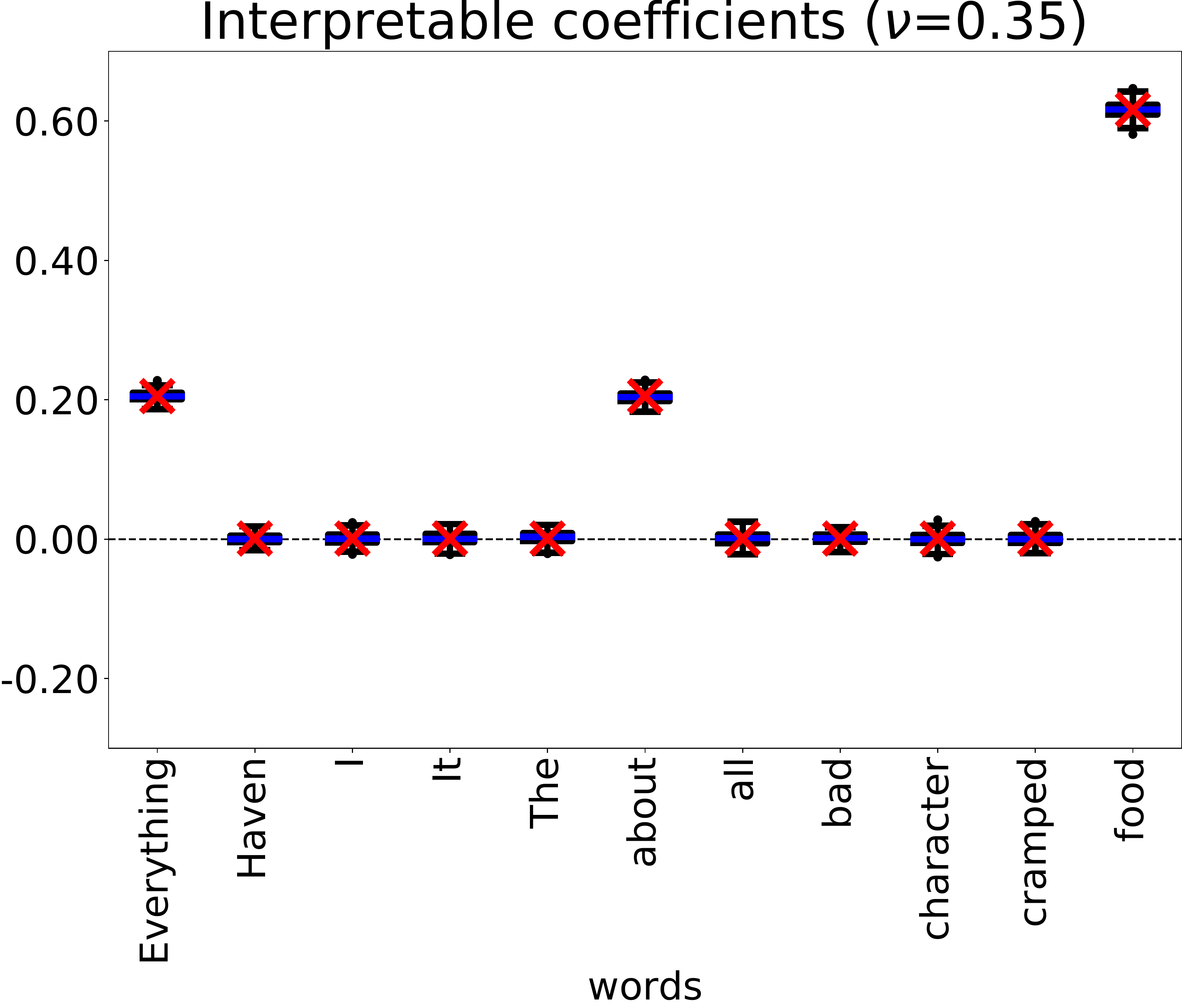}
\caption{\label{fig:tree-bandwidth}Influence of the bandwidth on the explanation  given for a small decision tree on a Yelp review ($n=5000,n_\text{exp}=100$, $d=29$). \emph{Left panel:} $\nu=0.05$, \emph{right panel:} $\nu=0.35$. Our theoretical predictions remain accurate for non-default bandwidths.}
\end{figure}

\paragraph{Influence of the bandwidth.}
Let us consider the same example $\xi$ and decision tree as in the paper. 
In particular, the model $f$ is written as 
\[
\indic{\text{``food''}} + (1-\indic{\text{``food''}}) \cdot \indic{\text{``about''}} \cdot \indic{\text{``Everything''}}
\, .
\]
We now consider non-default bandwidths, that is, bandwidths different than $0.25$. 
We present in Figure~\ref{fig:tree-bandwidth} the results of these experiments. 
In the left panel, we took a smaller bandwidth ($\nu=0.05$) and in the right panel a larger bandwidth ($\nu=0.35$). 
We see that while the numerical value of the coefficients changes slightly, their relative order is preserved. 
Moreover, our theoretical predictions remain accurate in that case, which is to be expected since we did not resort to any approximation in this case. 
Interestingly, the empirical results for small $\nu$ seem more spread out, as hinted by Theorem~\ref{th:betahat-concentration}.

\begin{figure}
\centering
\includegraphics[scale=0.21]{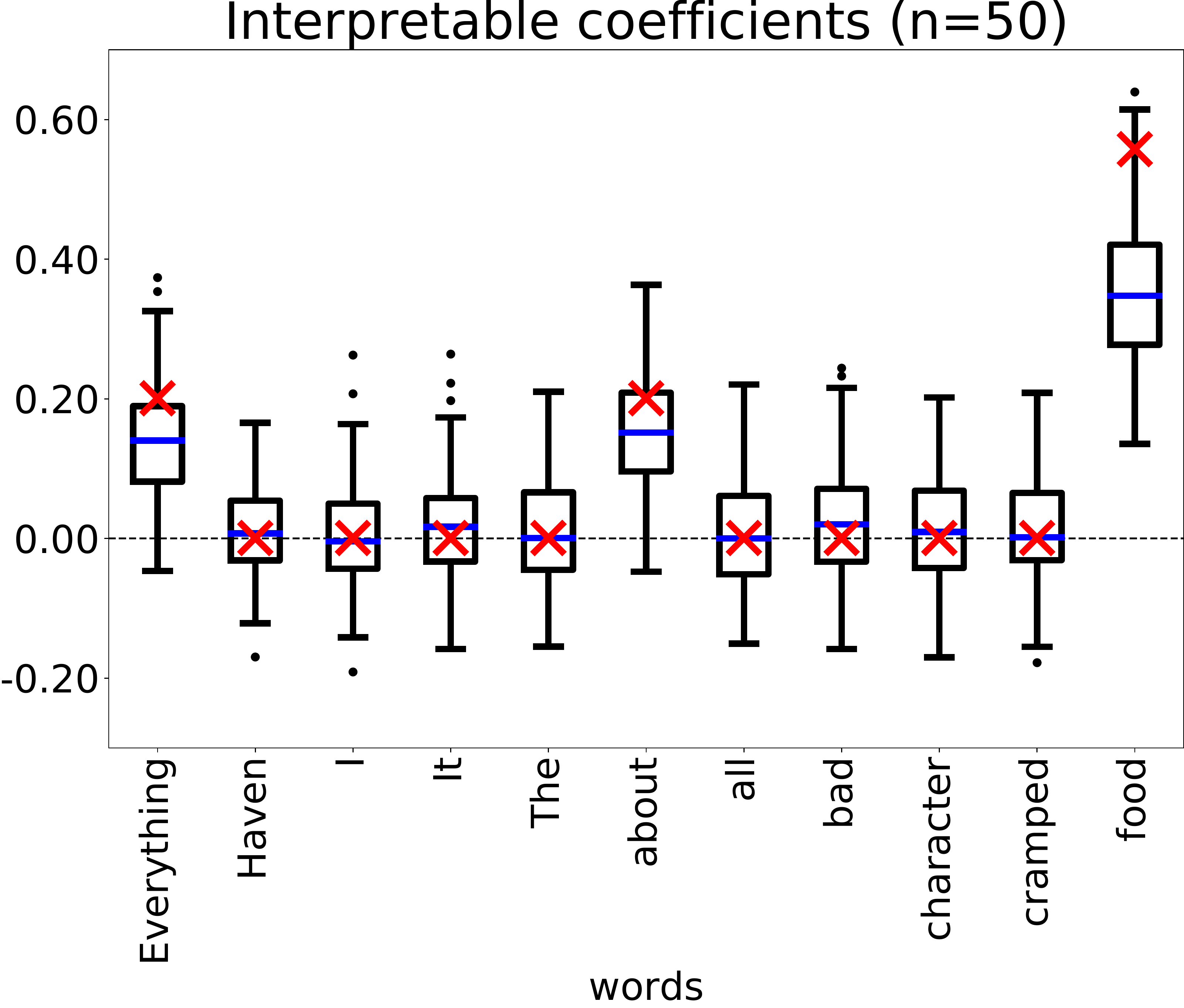}  
\hspace{0.5cm}
\includegraphics[scale=0.21]{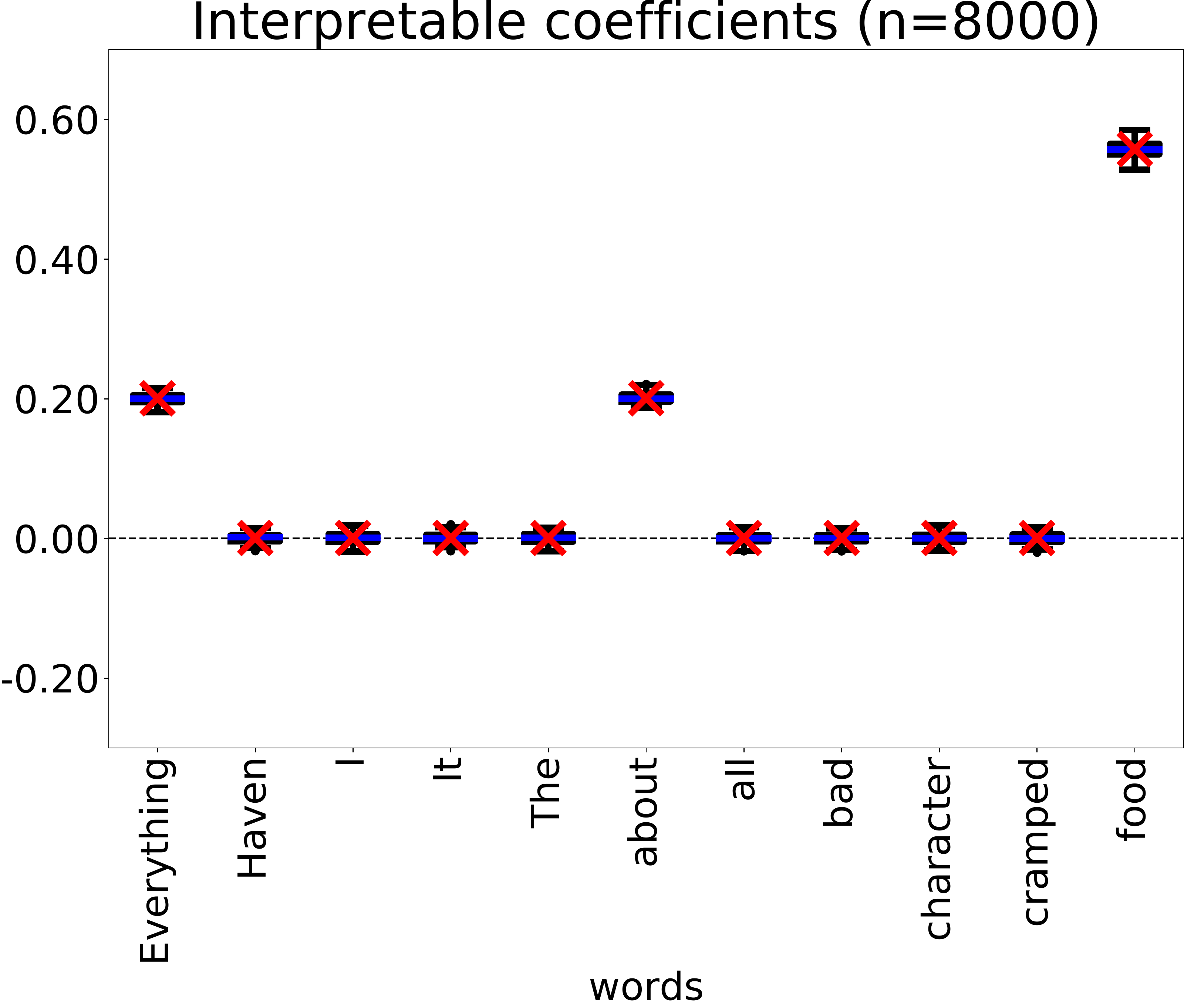}
\caption{\label{fig:tree-nsample}Influence of the number of perturbed samples on the explanation  given for a small decision tree on a Yelp review ($\nu=0.25,n_\text{exp}=100,d=29$). \emph{Left panel:} $n=50$, \emph{right panel:} $n=8000$. Empirical values are less likely to be close to the theoretical predictions for small $n$.}
\end{figure}

\paragraph{Influence of the number of samples.}
Keeping the same model and example to explain as above, we looked into non-default number of samples $n$. 
We present in Figure~\ref{fig:tree-nsample} the results of these experiments. 
We took a very small $n$ in the left panel ($n=50$ is two orders of magnitude smaller than the default $n=5000$) and a larger $n$ in the right panel. 
As expected, when $n$ is larger, the concentration around our theoretical predictions is even better. 
To the opposite, for small $n$, we see that the explanations vary wildly. 
This is materialized by much wider whisker boxes. 
Nevertheless, to our surprise, it seems that our theoretical predictions still contain some relevant information in that case. 

\paragraph{Influence of depth.}
Finally, we looked into more complex decision trees. 
The decision rule used in Figure~\ref{fig:tree-complex} is given by 
\[
\indic{\text{``food''}} + (1-\indic{\text{``food''}})\indic{\text{``about''}}\indic{\text{``Everything''}} +\indic{\text{``bad''}}+ \indic{\text{``bad''}}\indic{\text{``character''}}
\, .
\]
We see that increasing the depth of the tree is not a problem from a theoretical point of view. 
It is interesting to see that words used in several nodes for the decision receive more weight (\emph{e.g.}, ``bad'' in this example). 

\begin{figure}
\centering
\includegraphics[scale=0.25]{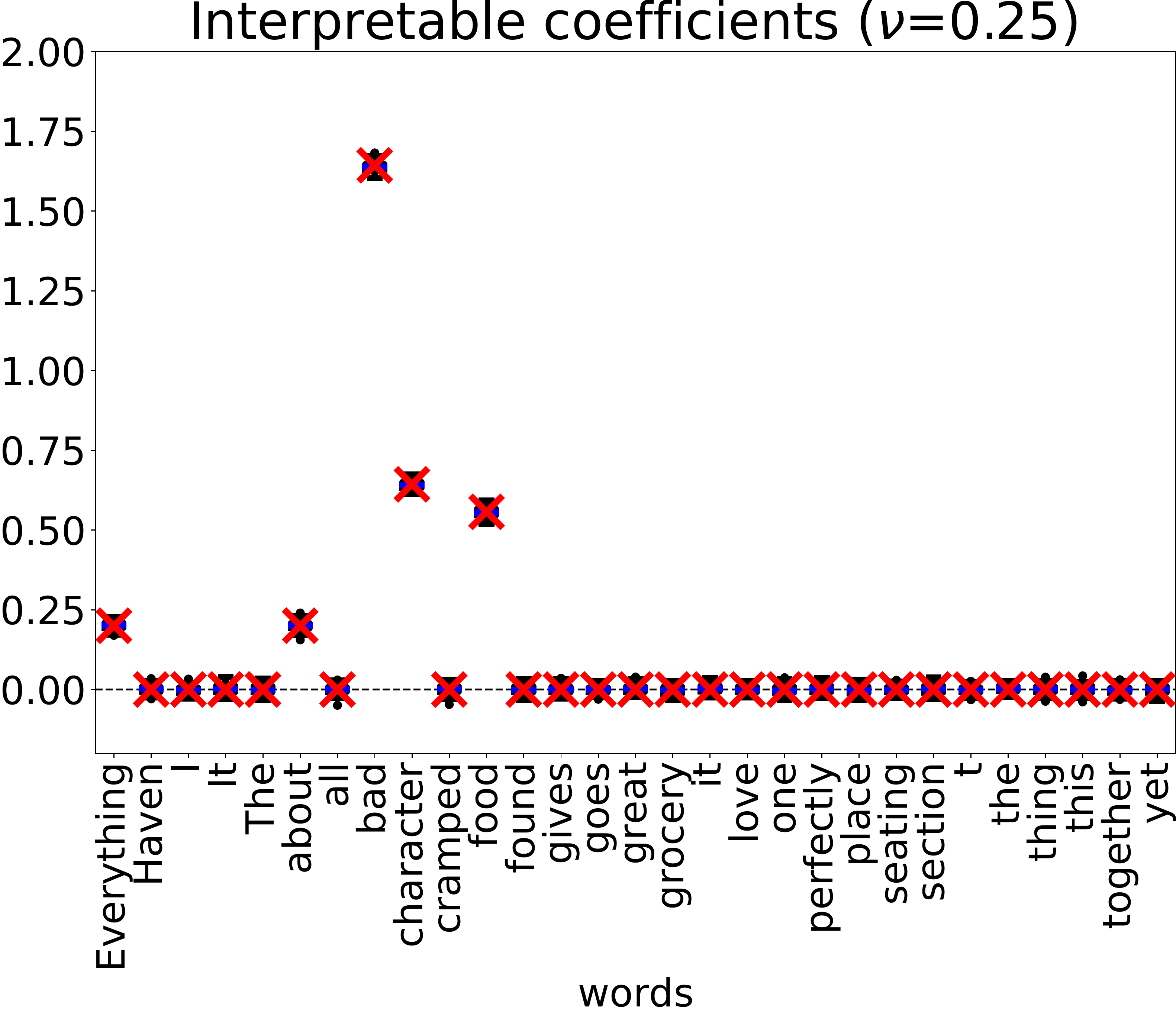}  
\caption{\label{fig:tree-complex}Theory meets practice for a more complex decision tree  ($\nu=0.25,n_\text{exp}=100,n=5000,d=29$). Here we report all coefficients. The theory still holds for more complex trees.}
\end{figure}

\subsection{Linear models}
\label{sec:add-exp-linear}

Let us conclude this section with additional experiments for linear models. 
As in the paper, we consider an arbitrary linear model 
\[
f(\normtfidf{x}) = \sum_{j=1}^d \lambda_j \normtfidf{x}_j
\, .
\]
In practice, the coefficients $\lambda_j$ are drawn i.i.d. according to a Gaussian distribution. 

\paragraph{Influence of the bandwidth.}
As in the previous section, we start by investigating the role of the bandwidth in the accuracy of our theoretical predictions. 
We see in the right panel of Figure~\ref{fig:linear-bandwidth} that taking a larger bandwidth does not change much neither the explanations nor the fit between our theoretical predictions and the empirical results. 
This is expected, since our approximation (Eq.~\eqref{eq:simplified-betainf-linear}) is based on the large bandwidth approximation. 
However, the left panel of Figure~\ref{fig:linear-bandwidth} shows how this approximation becomes dubious when the bandwidth is small.  
It is interesting to note that in that case, the theory seems to always \emph{overestimate} the empirical results, in absolute value. 
The large bandwidth approximation is definitely a culprit here, but it could also be the regularization coming into play. 
Indeed, the discussion at the end of Section~2.4 in the paper that lead us to ignore the regularization is no longer valid for a small $\nu$. 
In that case, the $\pi_i$s can be quite small and the first term in Eq.~(5) of the paper is of order $\exps{-1/(2\nu^2)}n$ instead of $n$. 

\begin{figure}
\centering
\includegraphics[scale=0.25]{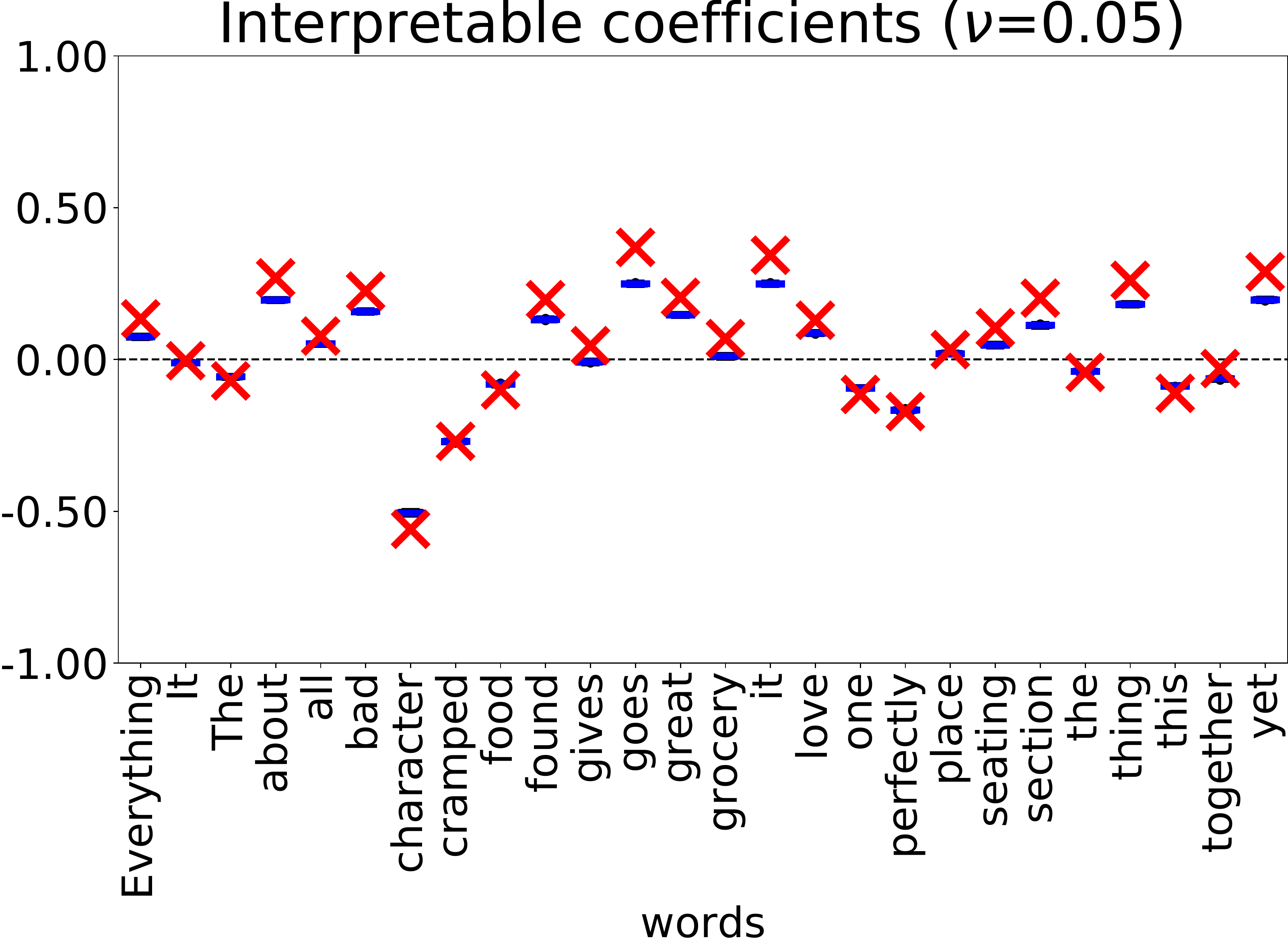} 
\includegraphics[scale=0.25]{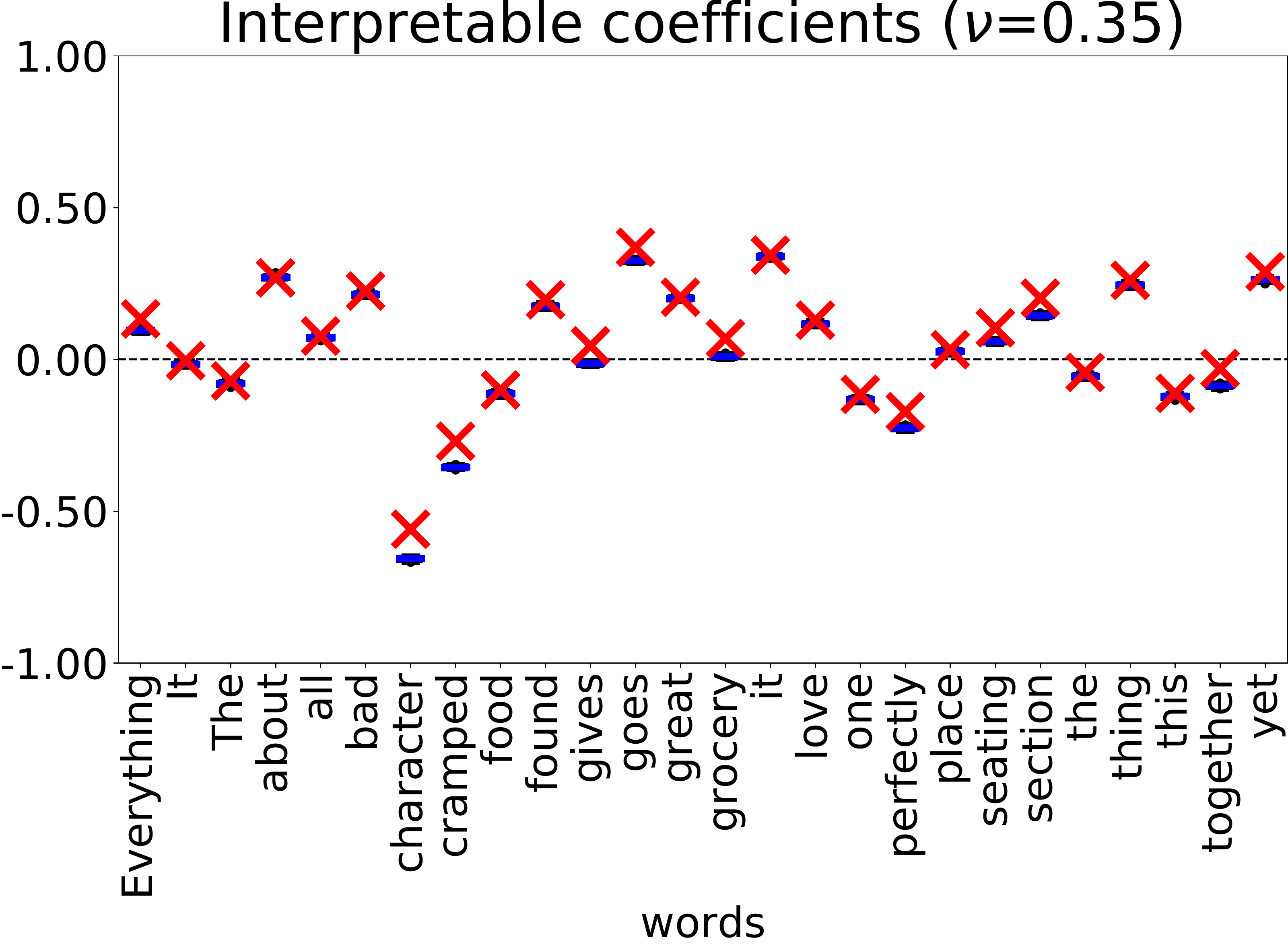}  
\caption{\label{fig:linear-bandwidth}Influence of the bandwidth on the explanation for a linear model on a Yelp review ($n_\text{exp}=100,n=5000, d=29 $). \emph{Left panel:} $\nu=0.05$, \emph{right panel:} $\nu=0.35$. The approximate theoretical values are less accurate for smaller bandwidths.}
\end{figure}

\paragraph{Influence of the number of samples.}
Now let us look at the influence of the number of perturbed samples. 
As in the previous section, we look into very small values of~$n$, \emph{e.g.}, $n=50$.  
We see in the left panel of Figure~\ref{fig:linear-n} that, as expected, the variability of the explanations increases drastically. 
The theoretical predictions seem to overestimate the empirical results in absolute value, which could again be due to the regularization beginning to play a role for small $n$, since the discussion in Section~2.4 of the paper is only valid for large~$n$. 

\begin{figure}
\centering
\includegraphics[scale=0.25]{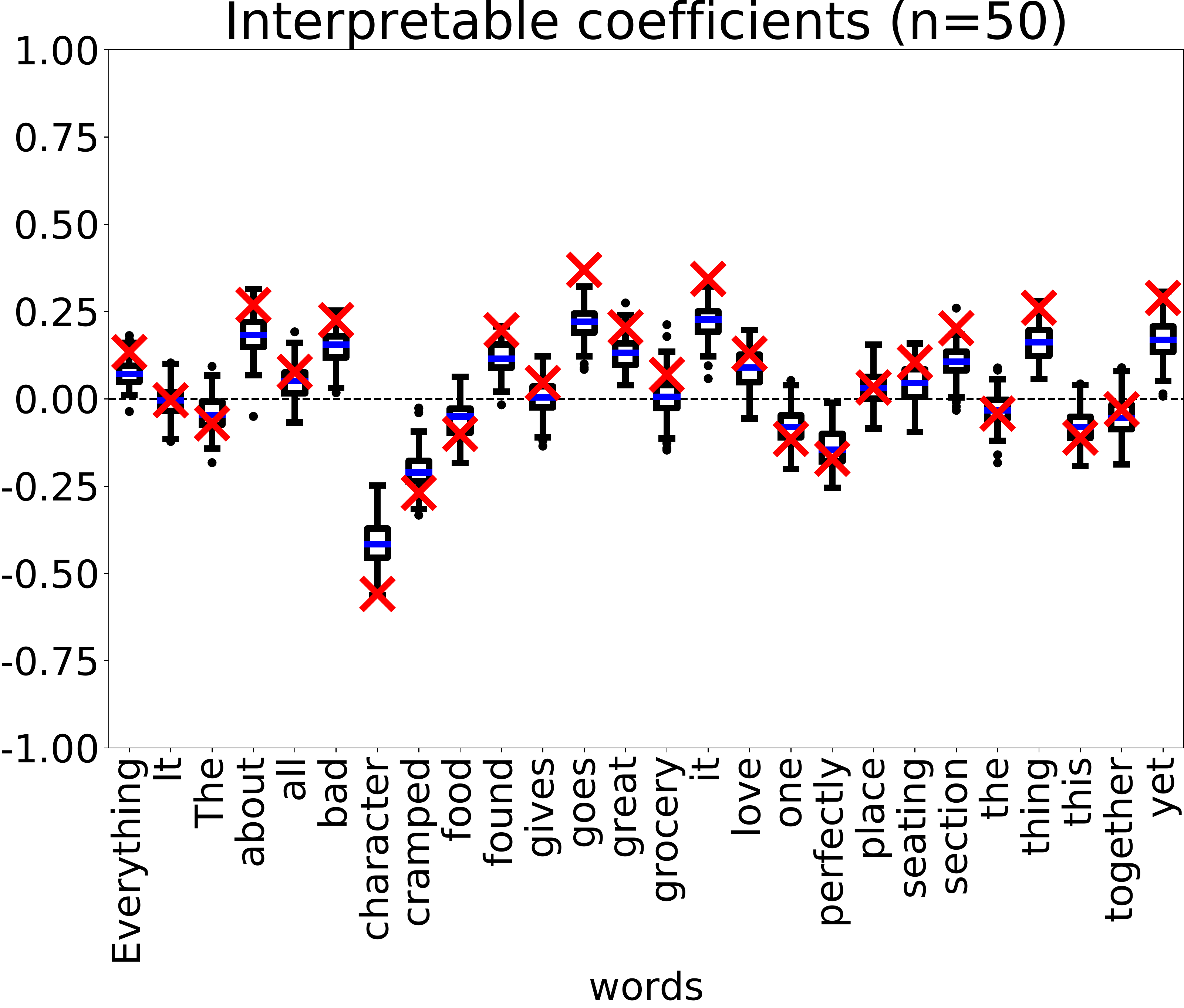} 
\includegraphics[scale=0.25]{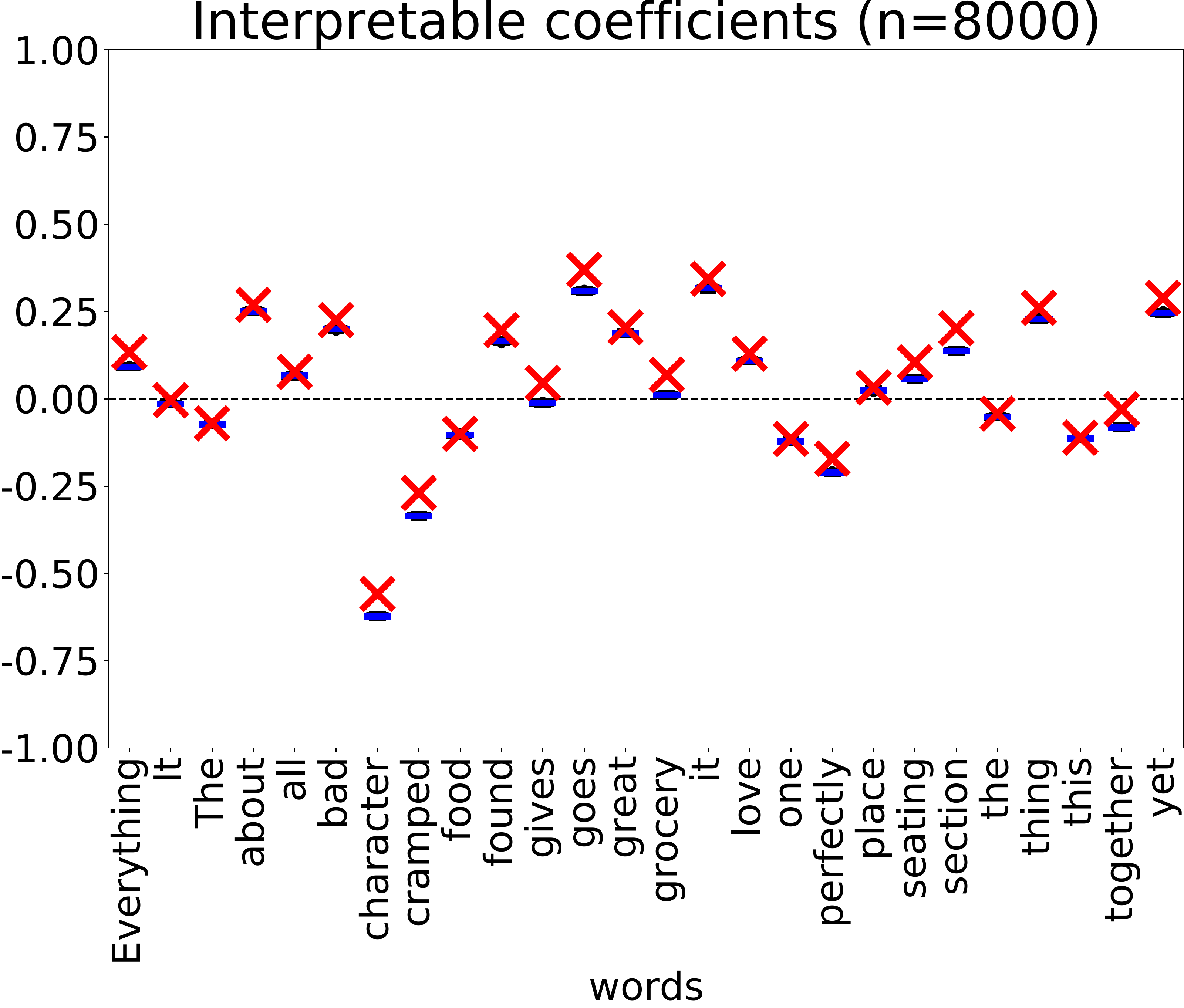}  
\caption{\label{fig:linear-n}Influence of the number of perturbed samples on the explanation  for a linear model on a Yelp review ($ \nu=0.25,n_\text{exp}=100, d=29 $). \emph{Left panel:} $n=50$, \emph{right panel:} $n=8000$. The empirical explanations are more spread out for small values of $n$.}
\end{figure}

\paragraph{Influence of $d$.}
To conclude this section, let us note that $d$ does not seem to be a limiting factor in our analysis. 
While Theorem~\ref{th:betahat-concentration} hints that the concentration phenomenon may worsen for large $d$, as noted before in Remark~\ref{remark:influence-of-d}, we have reason to suspect that it is not the case. 
All experiments presented on this section so far consider an example whose local dictionary has size $d=29$. 
In Figure~\ref{fig:linear-large-d} we present an experiment on an example that has a local dictionary of size $d=52$.
We observed no visible change in the accuracy of our predictions. 

\begin{figure}

\centering
\includegraphics[scale=0.25]{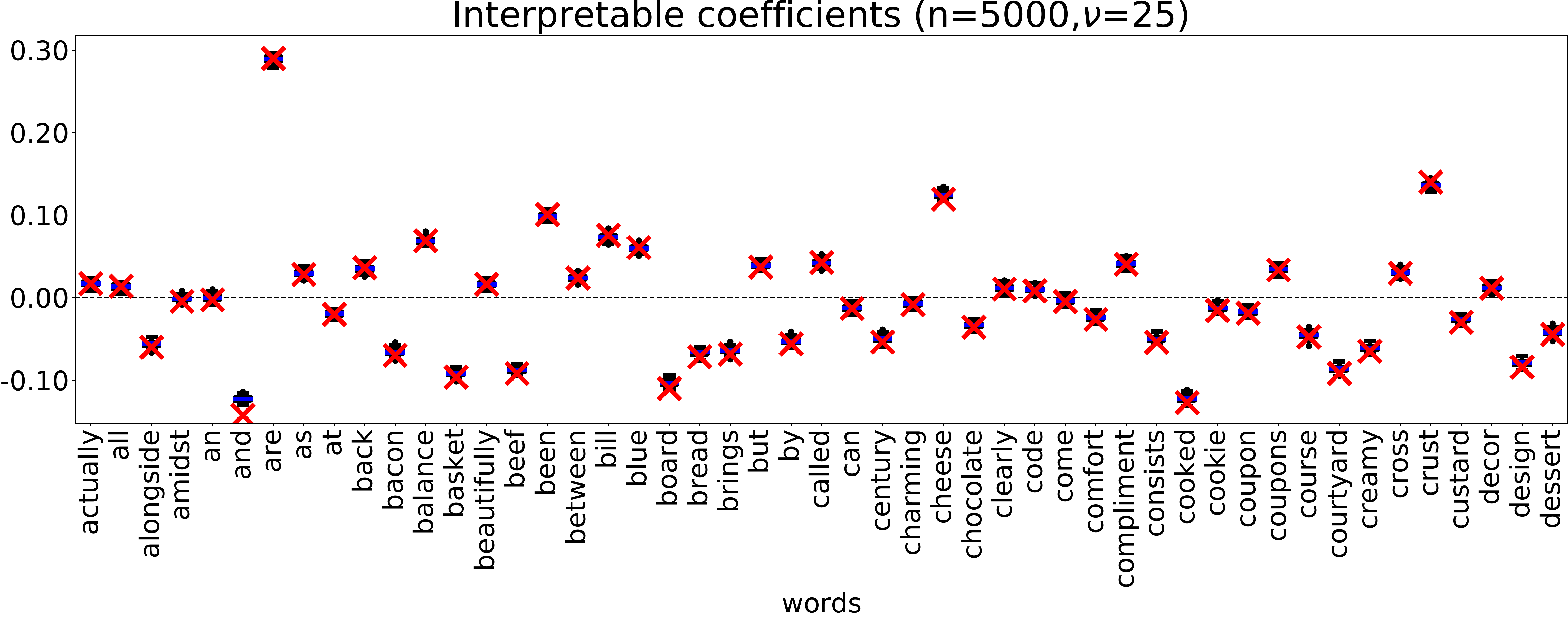}  
\caption{\label{fig:linear-large-d}Theory meets practice for an example with a larger vocabulary ($\nu=0.25,n_\text{exp}=100,n=5000,d=537$). Here we report only $50$ interpretable coefficients. Our theoretical predictions seem to hold for larger local dictionaries.}
\end{figure}

\end{document}